\documentclass[10pt]{article} 

\usepackage[accepted]{rlj} 

\usepackage{amssymb}            
\usepackage{mathtools}
\usepackage{mathrsfs}
\usepackage{subcaption}
\usepackage[space]{grffile}
\usepackage{url}
\usepackage{lipsum}

\usepackage{graphicx}
\usepackage{amsmath}
\usepackage{amsthm}
\usepackage{bbm}
\usepackage{algorithm}
\usepackage{algorithmic}
\usepackage{dsfont}
\usepackage{listings}

\newtheorem{definition}{Definition}[section]
\newtheorem{assumption}{Assumption}[section]

\newtheorem{theorem}{Theorem}[section]
\newtheorem{lemma}{Lemma}[section]
\newtheorem{corollary}{Corollary}[section]
\newtheorem{remark}{Remark}[section]

\title{Burning RED: Unlocking Subtask-Driven\\Reinforcement Learning and Risk-Awareness\\in Average-Reward Markov Decision Processes}

\setrunningtitle{Burning RED: Unlocking Subtask-Driven RL and Risk-Awareness in Average-Reward MDPs}

\author{Juan Sebastian Rojas\textsuperscript{1}, Chi-Guhn Lee\textsuperscript{1}}

\emails{juan.rojas@mail.utoronto.ca, cglee@mie.utoronto.ca}

\affiliations{
$^{1}$\textbf{University of Toronto, Canada}
}

\contribution{
    We provide a general-purpose framework and a corresponding set of prediction/control algorithms for solving an arbitrary number of learning objectives, or \emph{subtasks}, simultaneously in the average-reward setting with only a TD error-based update, including proven-convergent algorithms for the tabular case.
    }
    {
    Our work builds on (and can be viewed as a generalization of) \citet{Wan2021-re}, which proposed proven-convergent average-reward RL algorithms that are able to learn and/or optimize the value function and average-reward simultaneously using only the TD error. In particular, the focus in \citet{Wan2021-re} was on proving the convergence of such algorithms, without exploring the underlying structural properties of the average-reward MDP that made such a process possible to begin with. In this work, we formalize these underlying properties, and utilize them to show that if one modifies, or \emph{extends}, the reward from the MDP with various learning objectives, then these objectives, or \emph{subtasks}, can be solved simultaneously using a modified, or \emph{reward-extended}, version of the TD error.
    }

\contribution{
    We utilize the framework described in Contribution 1 to derive the first family of RL algorithms that can optimize the well-known conditional value-at-risk (CVaR) risk measure in a fully-online manner \emph{without} the use of an explicit bi-level optimization scheme or an augmented state-space. We perform an empirical evaluation on two toy experiments, thereby illustrating the properties and effectiveness of the algorithms, while also noting that a more comprehensive empirical study is needed to fully gauge their practical implications.
    }
    {
    Several prior works have investigated CVaR optimization in the discounted setting (e.g. \citet{Bauerle2011-yi} and \citet{Chow2015-vw}). However, no prior work has developed an algorithm for CVaR optimization that does not require either an augmented state-space or an explicit bi-level optimization, which can, for example, involve solving multiple MDPs. In the average-reward setting, \citet{Xia2023-cq} proposed a set of algorithms for optimizing the CVaR risk measure, however their methods require the use of an augmented state-space and a sensitivity-based bi-level optimization. By contrast, our work, to the best of our knowledge, is the first to optimize CVaR in an MDP-based setting without the use of an explicit bi-level optimization scheme or an augmented state-space. 
    }

\keywords{Average-Reward Reinforcement Learning, Risk-Sensitive Decision-Making, CVaR} 

\summary{
Average-reward Markov decision processes (MDPs) provide a foundational framework for sequential decision-making under uncertainty. However, average-reward MDPs have remained largely unexplored in reinforcement learning (RL) settings, with the majority of RL-based efforts having been allocated to discounted MDPs. In this work, we study a unique structural property of average-reward MDPs and utilize it to introduce \emph{Reward-Extended Differential} (or \emph{RED}) reinforcement learning: a novel RL framework that can be used to effectively and efficiently solve various learning objectives, or \emph{subtasks}, simultaneously in the average-reward setting. We introduce a family of RED learning algorithms for prediction and control, including proven-convergent algorithms for the tabular case. We then showcase the power of these algorithms by demonstrating how they can be used to learn a policy that optimizes, for the first time, the well-known conditional value-at-risk (CVaR) risk measure in a fully-online manner, \emph{without} the use of an explicit bi-level optimization scheme or an augmented state-space.
}

\begin{document}

\maketitle  

\begin{abstract}
Average-reward Markov decision processes (MDPs) provide a foundational framework for sequential decision-making under uncertainty. However, average-reward MDPs have remained largely unexplored in reinforcement learning (RL) settings, with the majority of RL-based efforts having been allocated to discounted MDPs. In this work, we study a unique structural property of average-reward MDPs and utilize it to introduce \emph{Reward-Extended Differential} (or \emph{RED}) reinforcement learning: a novel RL framework that can be used to effectively and efficiently solve various learning objectives, or \emph{subtasks}, simultaneously in the average-reward setting. We introduce a family of RED learning algorithms for prediction and control, including proven-convergent algorithms for the tabular case. We then showcase the power of these algorithms by demonstrating how they can be used to learn a policy that optimizes, for the first time, the well-known conditional value-at-risk (CVaR) risk measure in a fully-online manner, \emph{without} the use of an explicit bi-level optimization scheme or an augmented state-space.
\end{abstract}

\section{Introduction}

Markov decision processes (MDPs) \citep{Puterman1994-dq} are a long-established framework for sequential decision-making under uncertainty. Discounted MDPs, which aim to optimize a potentially-discounted sum of rewards over time, have enjoyed success in recent years when utilizing reinforcement learning (RL) solution methods \citep{Sutton2018-eh} to tackle certain problems of interest in various domains. Despite this success, however, these MDP-based methods have yet to be fully embraced in real-world applications due to the various intricacies and implications of real-world operation that often outweigh the capabilities of current state-of-the-art methods \citep{Dulac-Arnold2021-sc}. We therefore turn to the less-explored average-reward MDP, which aims to optimize the reward received per time-step, to see how its unique structural properties can be leveraged to tackle challenging problems that have eluded its discounted counterpart. 

In particular, we present results that show how the average-reward MDP's unique structural properties can be leveraged to enable a more \emph{subtask-driven} approach to reinforcement learning, where various learning objectives, or \emph{subtasks}, are solved simultaneously (and in a fully-online manner) to help solve a larger, central learning objective. Importantly, we find a compelling case-study in the realm of risk-aware decision-making that illustrates how this subtask-driven approach can alleviate some of the computational challenges and complexities that can arise in the discounted setting.

More formally, we introduce \emph{Reward-Extended Differential} (or \emph{RED}) reinforcement learning: a first-of-its-kind RL framework that makes it possible to solve various subtasks simultaneously in the average-reward setting. At the heart of this framework is the novel concept of the reward-extended temporal-difference (TD) error, an extension of the celebrated TD error \citep{Sutton1988-vs}, which we derive by leveraging a unique structural property of average-reward MDPs, and utilize to solve various subtasks simultaneously. We first present the RED RL framework in a generalized way, then adopt it to successfully tackle a problem that has exceeded the capabilities of current state-of-the-art methods in risk-aware decision-making: learning a policy that optimizes the well-known conditional value-at-risk (CVaR) risk measure \citep{Rockafellar2000-xu} in a fully-online manner \emph{without} the use of an explicit bi-level optimization scheme or an augmented state-space.

Our work is organized as follows: In Section \ref{related_work}, we provide a brief overview of related work. In Section \ref{preliminaries}, we give an overview of the fundamental concepts related to average-reward RL and CVaR. In Section \ref{subtask_driven}, we motivate the need and opportunity for a subtask-driven approach to RL through the lens of CVaR optimization. In Section \ref{red}, we introduce the RED RL framework, including the concept of the reward-extended TD error. We also introduce a family of RED RL algorithms for prediction and control, and highlight their convergence properties (with full convergence proofs in Appendix \ref{appendix_proofs}). In Section \ref{risk_red}, we use the RED RL framework to derive a subtask-driven approach for CVaR optimization, and provide empirical results which show that this approach can be used to successfully learn a policy that optimizes the CVaR risk measure. Finally, in Section \ref{discussion}, we emphasize our framework's potential usefulness towards tackling other challenging problems outside the realm of risk-awareness, highlight some of its limitations, and suggest some directions for future research.

\section{Related Work}
\label{related_work}

\textbf{Average-Reward Reinforcement Learning:} Average-reward (or average-cost) MDPs, despite being one of the most well-studied frameworks for sequential decision-making under uncertainty \citep{Puterman1994-dq}, have remained relatively unexplored in reinforcement learning (RL) settings. To date, notable works on the subject (in the context of RL) include \citet{Schwartz1993-jb}, \citet{Tsitsiklis1999-tu}, \citet{Abounadi2001-xu}, 
\citet{Gosavi2004-qf}, \citet{Bhatnagar2009-lt}, and \citet{Wan2021-re}. Most relevant to our work is \citet{Wan2021-re}, which provided a rigorous theoretical treatment of average-reward MDPs in the context of RL, and proposed the proven-convergent ‘Differential TD-learning’ and ‘Differential Q-learning’ algorithms. Our work builds on the methods from \citet{Wan2021-re} to develop a theoretical framework for solving various learning objectives simultaneously. 

We note that these learning objectives, or \emph{subtasks}, as explored in our work, are different from that of hierarchical RL (e.g. \citet{Sutton1999-kf}). In particular, in hierarchical RL, the focus is on using temporally-abstracted actions, known as ‘options’ (or ‘skills’), such that the agent learns a policy for each option, as well as an inter-option policy. By contrast, in our work, we learn a single policy, and the subtasks are not part of the action-space. Similarly, the notion of solving multiple objectives in parallel has been widely-explored in the discounted setting (e.g. \citet{McLeod2021-qy}). However, much of this work focuses on learning multiple state representations (or ‘features’), options, policies, and/or value functions. By contrast, in our work, we learn a single policy and value function, and the subtasks are not part of the state or action-spaces. To the best of our knowledge, our work is the first to explore the solving of subtasks simultaneously in the average-reward setting.

\textbf{Risk-Aware Learning and Optimization in MDPs:} The notion of risk-aware learning and optimization in MDP-based settings has been long-studied, from the well-established expected utility framework \citep{Howard1972-zv}, to the more contemporary framework of coherent risk measures \citep{Artzner1999-bs}. To date, these risk-based efforts have focused almost exclusively on the discounted setting. Importantly, optimizing the CVaR risk measure in this setting typically requires augmenting the state-space and/or having to utilize an explicit bi-level optimization scheme, which can, for example, involve solving multiple MDPs. Seminal works that have investigated CVaR optimization in the standard discounted setting include \citet{Bauerle2011-yi} and \citet{Chow2015-vw, Hau2023-li}. In the distributional setting, works such as \citet{Dabney2018-yi} and \citet{Keramati2020-ht} have proposed CVaR optimization approaches that do not require an augmented state-space or an explicit bi-level optimization; however, it was later shown in \citet{Lim2022-lr} that such approaches converge to neither the optimal dynamic-CVaR nor the optimal static-CVaR policies (\citet{Lim2022-lr} then proposed a valid approach that utilizes an augmented state-space). Some works have investigated optimizing a time-consistent \citep{Ruszczynski2010-sm} interpretation of CVaR in the discounted setting; however, such approaches only approximate CVaR, as CVaR is generally not a time-consistent risk measure in the discounted setting \citep{Boda2006-tw}. Other works have investigated optimizing similar objectives to CVaR that are more computationally tractable, such as the entropic value-at-risk (e.g. \citet{Hau2023-ey}). 

Most similar to our work (in non-average-reward settings) is \citet{Stanko2019-ca}, which proposed a similar CVaR update scheme to the one derived in our work. However, all of the methods proposed in \citet{Stanko2019-ca} require either an augmented state-space or an explicit bi-level optimization. In the average-reward setting, \citet{Xia2023-cq} proposed a set of algorithms for optimizing the CVaR risk measure; however, their methods require the use of an augmented state-space and a sensitivity-based bi-level optimization. By contrast, our work, to the best of our knowledge, is the first to optimize the CVaR risk measure in an MDP-based setting without the use of an explicit bi-level optimization scheme or an augmented state-space. We note that other works have investigated optimizing other risk measures in the average-reward setting, such as the exponential cost (e.g. \citet{Murthy2023-sh}), and variance (e.g. \citet{Prashanth2016-mr}).

\section{Preliminaries}
\label{preliminaries}

\subsection{Average-Reward Reinforcement Learning}
\label{ar_prelim}

A finite average-reward MDP is the tuple \(\mathcal{M} \doteq \langle \mathcal{S}, \mathcal{A}, \mathcal{R}, p \rangle\), where \(\mathcal{S}\) is a finite set of states, \(\mathcal{A}\) is a finite set of actions, \(\mathcal{R} \subset \mathbb{R}\) is a bounded set of rewards, and \(p: \mathcal{S}\, \times\, \mathcal{A}\, \times\, \mathcal{R}\, \times\,  \mathcal{S} \rightarrow{} [0, 1]\) is a probabilistic transition function that describes the dynamics of the environment. At each discrete time step, \(t = 0, 1, 2, \ldots\), an agent chooses an action, \(A_t \in \mathcal{A}\), based on its current state, \(S_t \in \mathcal{S}\), and receives a reward, \(R_{t+1} \in \mathcal{R}\), while transitioning to a (potentially) new state, \(S_{t+1}\), such that \(p(s', r \mid s, a) = \mathbb{P}(S_{t+1} = s', R_{t+1} = r \mid S_t = s, A_t = a)\). In an average-reward MDP, an agent aims to find a policy, \(\pi: \mathcal{S} \rightarrow{} \mathcal{A}\), that optimizes the long-run (or limiting) average-reward, \(\bar{r}\), which is defined as follows for a given policy, \(\pi\):
\begin{equation}
\label{eq_avg_reward}
\bar{r}_{\pi}(s) \doteq  \lim_{n \rightarrow{} \infty} \frac{1}{n} \sum_{t=1}^{n} \mathbb{E}[R_t \mid S_0=s, A_{0:t-1} \sim \pi].
\end{equation}
In this work, we limit our discussion to \emph{stationary Markov} policies, which are time-independent policies that satisfy the Markov property.

When working with average-reward MDPs, it is common to simplify Equation \eqref{eq_avg_reward} into a more workable form by making certain assumptions about the Markov chain induced by following policy \(\pi\). To this end, a \emph{unichain} assumption is typically used when doing prediction (learning) because it ensures the existence of a unique limiting distribution of states, \(\mu_{\pi}(s) \doteq \lim_{t \rightarrow{} \infty} \mathbb{P}(S_t = s \mid A_{0:t-1} \sim \pi)\), that is independent of the initial state, thereby simplifying Equation \eqref{eq_avg_reward} to the following:
\begin{equation}
\label{eq_avg_reward_2}
\bar{r}_{\pi} = \sum_{s \in \mathcal{S}} \mu_{\pi}(s) \sum_{a \in \mathcal{A}} \pi(a \mid s) \sum_{s' \in \mathcal{S}} \sum_{r \in  \mathcal{R}}p(s', r \mid s, a)r.
\end{equation}
Similarly, a \emph{communicating} assumption is typically used for control (optimization) because it ensures the existence of a unique optimal average-reward, \(\bar{r}*\), that is independent of the initial state.

To solve an average-reward MDP, solution methods such as dynamic programming or RL can be used in conjunction with the following \emph{Bellman} (or \emph{Poisson}) equations:
\begin{equation}
\label{eq_avg_reward_4}
v_{\pi}(s) = \sum_{a}\pi(a \mid s) \sum_{s'}\sum_{r}p(s', r \mid s, a)[r - \bar{r}_{\pi} + v_{\pi}(s')],
\end{equation}

\begin{equation}
\label{eq_avg_reward_5}
q_{\pi}(s, a) = \sum_{s'}\sum_{r}p(s', r \mid s, a)[r - \bar{r}_{\pi} + \max_{a'}q_{\pi}(s', a')],
\end{equation}

where \(v_{\pi}(s)\) is the state-value function and \(q_{\pi}(s, a)\) is the state-action value function for a given policy, \(\pi\). Solution methods for average-reward MDPs are typically referred to as \emph{differential} methods because of the reward difference (i.e., \(r - \bar{r}_{\pi}\)) operation that occurs in Equations \eqref{eq_avg_reward_4} and \eqref{eq_avg_reward_5}. We note that solution methods typically find the solutions to Equations \eqref{eq_avg_reward_4} and \eqref{eq_avg_reward_5} up to a constant. This is typically not a concern, given that the relative ordering of policies is usually what is of interest. 

In the context of RL, \citet{Wan2021-re} proposed the tabular ‘Differential TD-learning’ and ‘Differential Q-learning’ algorithms, which are able to learn and/or optimize the value function and average-reward simultaneously using only the TD error. The ‘Differential TD-learning’ algorithm is shown below:
\begin{subequations}
\label{eq_avg_reward_6}
\begin{align}
\label{eq_avg_reward_6_1}
& \delta_{t} \doteq R_{t+1} - \bar{R}_t + V_{t}(S_{t+1}) - V_{t}(S_t)\\
\label{eq_avg_reward_6_2}
& V_{t+1}(s) \doteq V_{t}(s), \quad \forall s \neq S_t\\
\label{eq_avg_reward_6_3}
& V_{t+1}(S_t) \doteq V_{t}(S_t) + \alpha_{t}\rho_{t}\delta_{t}\\
\label{eq_avg_reward_6_4}
& \bar{R}_{t+1} \doteq \bar{R}_t + \alpha_{\bar{r},t}\rho_{t}\delta_{t}
\end{align}
\end{subequations}

where \(V_t: \mathcal{S} \rightarrow \mathbb{R}\) is a table of state-value function estimates, \(\bar{R}_t\) is an estimate of the average-reward, \(\bar{r}_{\pi}\), \(\delta_t\) is the TD error, \(\rho_t \doteq \pi(A_t \mid S_t) \, / \, B(A_t \mid S_t)\) is the importance sampling ratio (with behavior policy, \(B\)), \(\alpha_t\) is the value function step size, and \(\alpha_{\bar{r},t}\) is the average-reward step size.

\subsection{Conditional Value-at-Risk (CVaR)}
\label{cvar_prelim}

Consider a random variable \(X\) with a finite mean on a probability space \((\Omega, \mathcal{F}, \mathbb{P})\), and with a cumulative distribution function \(F(x) = \mathbb{P}(X \le x)\). The (left-tail) \emph{value-at-risk (VaR)} of \(X\) with parameter \(\tau \in (0, 1)\) represents the \(\tau\)-quantile of \(X\), such that \(\text{VaR}_{\tau}(X) = \sup\{x \mid F(x) \le \tau\}\). The (left-tail) \emph{conditional value-at-risk (CVaR)} of \(X\) with parameter \(\tau\) is defined as follows:
\begin{equation}
\label{eq_cvar_1}
\text{CVaR}_{\tau}(X) = \frac{1}{\tau} \int_{0}^{\tau}\text{VaR}_{u}(X)du.
\end{equation}

When \(F(X)\) is continuous at \(x = \text{VaR}_{\tau}(X)\), \(\text{CVaR}_{\tau}(X)\) can be interpreted as the expected value of \(X\) conditioned on \(X\) being less than or equal to \(\text{VaR}_{\tau}(X)\), such that \(\text{CVaR}_{\tau}(X) = \mathbb{E}[X \mid X \le \text{VaR}_{\tau}(X)]\). 

Importantly, \(\text{CVaR}_{\tau}(X)\) can be formulated as the solution to the following optimization problem \citep{Rockafellar2000-xu}:
\begin{equation}
\label{eq_cvar_2}
\text{CVaR}_{\tau}(X) = \sup_{y \in \mathbb{R}} \mathbb{E}[y - \frac{1}{\tau}(y - X)^{+}] = \mathbb{E}[\text{VaR}_{\tau}(X) - \frac{1}{\tau}(\text{VaR}_{\tau}(X) - X)^{+}], 
\end{equation}
where \((u)^{+} = \max(u, 0)\). Existing MDP-based methods typically leverage the above formulation when optimizing for CVaR, by augmenting the state-space with a state that corresponds (either directly or indirectly) to an estimate of \(\text{VaR}_{\tau}(X)\) (in this case, \(y\)), and solving the following bi-level optimization:
\begin{equation}
\label{eq_cvar_3}
\sup_{\pi} \text{CVaR}_{\tau}(X) = \sup_{\pi} \sup_{y \in \mathbb{R}} \mathbb{E}[y - \frac{1}{\tau}(y - X)^{+}] = \sup_{y \in \mathbb{R}} (y - \frac{1}{\tau}\sup_{\pi} \mathbb{E}[(y - X)^{+}]),
\end{equation}
where the ‘inner’ optimization problem can be solved using standard MDP solution methods. 

In discounted MDPs, the random variable \(X\) corresponds to a (potentially-discounted) sum of rewards. In average-reward MDPs, \(X\) corresponds to the limiting \emph{per-step reward}. In other words, the natural interpretation of CVaR in the average-reward setting is that of the CVaR associated with the limiting per-step reward distribution, as shown below (for a given policy, \(\pi\)) \citep{Xia2023-cq}:
\begin{equation}
\label{eq_cvar_4}
\text{CVaR}_{\tau, \pi}(s) \doteq  \lim_{n \rightarrow{} \infty} \frac{1}{n} \sum_{t=1}^{n} \text{CVaR}_{\tau}[R_t \mid S_0=s, A_{0:t-1} \sim \pi].
\end{equation}
As with the average-reward (i.e., Equation \eqref{eq_avg_reward}), a unichain assumption (or similar) makes this CVaR objective independent of the initial state. In recent years, CVaR has emerged as a popular risk measure, in part because it is a ‘coherent’ risk measure \citep{Artzner1999-bs}, meaning that it satisfies key mathematical properties which can be meaningful in safety-critical and risk-related applications.

Figure \ref{fig_average_reward} depicts the agent-environment interaction in an average-reward MDP, where following policy \( \pi \) yields a limiting average-reward and reward CVaR.

\begin{figure}[htbp]
\centerline{\includegraphics[scale=0.48]{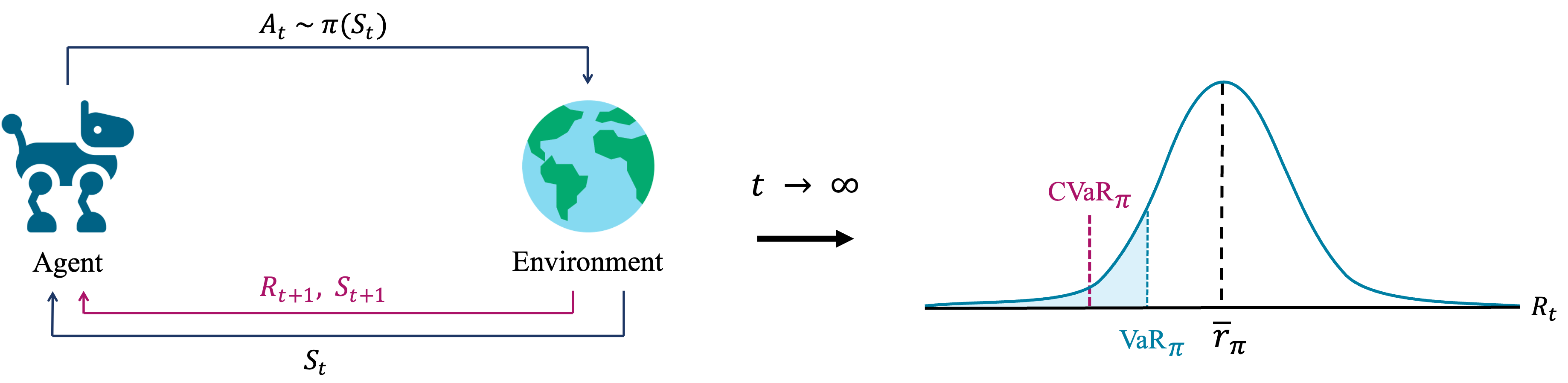}}
\caption{Illustration of the agent-environment interaction in an average-reward MDP. As \(t \to \infty\), following policy \(\pi\) yields a limiting per-step reward distribution with an average-reward, \(\bar{r}_{\pi}\), and a conditional value-at-risk, \(\text{CVaR}_{\pi}\). Standard average-reward RL methods aim to optimize the average-reward, \(\bar{r}_{\pi}\). By contrast, in our work, we aim to optimize \(\text{CVaR}_{\pi}\).}
\label{fig_average_reward}
\end{figure}

\section{A Subtask-Driven Approach}
\label{subtask_driven}

In this section, we motivate the need and opportunity for a subtask-driven approach to RL through the lens of CVaR optimization. Let us begin by considering the standard approach used by existing MDP-based methods for CVaR optimization. This approach, which is described in Equation \eqref{eq_cvar_3}, requires that we pick a wide range of guesses for the optimal value-at-risk, VaR, and that for each guess, \(y\), we solve an MDP. Then, among all the MDP solutions, we pick the best one as our final solution (which corresponds to \(y=\text{VaR}\)). Moreover, to further compound the computational costs associated with solving several MDPs, this approach requires that the state-space be augmented with a state that corresponds (either directly or indirectly) to the VaR guess, \(y\) (e.g. see \citet{Bauerle2011-yi}). Hence, this approach requires the use of both an explicit bi-level optimization scheme, and an augmented state-space. Importantly, however, this complex process would not be needed if we somehow knew what the optimal value for \(y\) (i.e., VaR) was. In fact, in the average-reward setting, if we know this optimal value, VaR, then optimizing for CVaR ultimately amounts to optimizing an average (as per Equation \eqref{eq_cvar_2}), which can be done trivially using the standard average-reward MDP.

As such, it would appear that, to optimize CVaR, we are stuck between two extremes: a significantly computationally-expensive process if we do not know the optimal value-at-risk, VaR, and a trivial process if we do. But what if we could somehow estimate VaR along the way? That is, keep some sort of running estimate of VaR that we optimize simultaneously as we optimize CVaR. Indeed, such an approach has been proposed in the discounted setting (e.g. \citet{Stanko2019-ca}); however, no approach has been able to successfully remove both the augmented state-space and the explicit bi-level optimization requirements. The primary difficulty lies in \emph{how} one updates the estimate of VaR along the way.

Critically, this is where the findings from \citet{Wan2021-re} become relevant. In particular, \citet{Wan2021-re} proposed proven-convergent algorithms for the average-reward setting that can learn and/or optimize the value function and average-reward simultaneously using only the TD error. In other words, these algorithms are able to solve two learning objectives simultaneously using only the TD error. Yet, the focus in \citet{Wan2021-re} was on proving the convergence of such algorithms, without exploring the underlying structural properties of the average-reward MDP that made such a process possible to begin with. In this work, we formalize these underlying properties, and utilize them to show that if one modifies, or \emph{extends}, the reward from the MDP with various learning objectives that satisfy certain key properties, then these objectives, or \emph{subtasks}, can be solved simultaneously using a modified, or \emph{reward-extended}, version of the TD error. Consequently, in terms of CVaR optimization, this allows us to develop appropriate learning updates for the VaR and CVaR estimates based solely on the TD error, such that we can optimize the CVaR risk measure in a fully-online manner without needing to augment the state-space or perform an explicit bi-level optimization.

In Section \ref{red}, we present the theoretical framework that enables the aforementioned subtask-driven approach. Then, in Section \ref{risk_red}, we adapt this general-purpose framework for CVaR optimization.

\section{Reward-Extended Differential (RED) Reinforcement Learning}
\label{red}

In this section, we present our primary contribution: a framework for solving various learning objectives, or \emph{subtasks}, simultaneously in the average-reward setting. We call this framework \emph{reward-extended differential} (or \emph{RED}) reinforcement learning. The ‘differential’ part of the name stems from the use of differential algorithms typically associated with average-reward MDPs. The ‘reward-extended’ part of the name stems from the use of the \emph{reward-extended TD error}, a novel concept that we will introduce shortly. Through this framework, we show how the average-reward MDP's unique structural properties can be leveraged to solve (i.e., learn or optimize) any given subtask by using only a TD error-based update. We first provide a formal definition for a (generic) subtask, then proceed to derive a framework that allows us to solve any subtask that satisfies this definition. In the subsequent section, we utilize this framework to tackle the CVaR optimization problem.

\begin{definition}[Subtask]
\label{definition_1}
A subtask, \(z_i\), is any scalar prediction or control objective belonging to a corresponding bounded set \(\mathcal{Z}_i \subset \mathbb{R}\), such that there exists a linear or piecewise linear subtask function, \(f: \mathcal{R} \, \times\, \mathcal{Z}_1 \, \times\, \mathcal{Z}_2 \, \times\, \cdots \, \times\, \mathcal{Z}_i \, \times\, \cdots \, \times\, \mathcal{Z}_n \rightarrow \tilde{\mathcal{R}}\), where \(\mathcal{R}\) is the bounded set of observed per-step rewards from the MDP \(\mathcal{M}\), \(\mathcal{\tilde{R}} \subset \mathbb{R}\) is a bounded set of ‘extended’ per-step rewards whose long-run average is the primary prediction or control objective of the MDP, \(\mathcal{\tilde{M}} \doteq \langle \mathcal{S}, \mathcal{A}, \mathcal{\tilde{R}}, \tilde{p} \rangle\), and \(\mathcal{Z} = \{z_1 \in \mathcal{Z}_1, z_2 \in \mathcal{Z}_2, \ldots, z_n \in \mathcal{Z}_n\}\) is the set of \(n\) subtasks that we wish to solve, such that:

i) \(f\) is invertible with respect to each input given all other inputs; and

ii) each subtask \(z_i \in \mathcal{Z}\) in \(f\) is independent of the states and actions, and hence independent of the observed per-step reward, \(R_t \in \mathcal{R}\), such that \(\mathbb{P}(S_{t+1} = s', \tilde{R}_{t+1} = f(r, z_1, \ldots, z_n) \mid S_t = s, A_t = a) = \mathbb{P}(S_{t+1} = s', R_{t+1} = r \mid S_t = s, A_t = a)\), and \(\mathbb{E}[f_j(R_t, z_1,  z_2, \ldots,  z_n)] = f_j(\mathbb{E}[R_t], z_1,  z_2, \ldots,  z_n)\), where \(f_j\) denotes the \(j\)th piecewise segment of \(f\), and \(\mathbb{E}\) denotes any expectation taken with respect to the states and actions.
\end{definition}

In essence, the above definition states that a subtask is some constant, \(z_i\), that we wish to learn and/or optimize. From an algorithmic perspective, this means that we will start with some initial estimate (or guess) for the subtask, \(Z_{i,t}\), then update this estimate at every time step, such that \(Z_{i,t} \to z_i\) or \(Z_{i,t} \to z^{*}_i\), depending on whether we are doing prediction or control (where \(z^{*}_i\) denotes the optimal subtask value). But how can we derive an appropriate update rule that accomplishes this? In the following section, we will introduce the \emph{reward-extended TD error}, through which we can derive such an update rule for any subtask that satisfies Definition \ref{definition_1}, such that \(Z_{i,t} \to z_i\) when doing prediction and \(Z_{i,t} \to z^{*}_i\) when doing control.

\subsection{The Reward-Extended TD Error}
\label{reward_extended_td}

In this section, we introduce and derive the \emph{reward-extended TD error}. In particular, we derive a generic, subtask-specific, TD-like error, \(\beta_{i,t}\), through which we can learn and/or optimize any subtask that satisfies Definition \ref{definition_1} via the update rule: \(Z_{i, t+1} = Z_{i,t} + \alpha_{z_i,t}\beta_{i, t}\), where \(Z_{i,t}\) is an estimate of subtask \(z_i\), \(\alpha_{z_i,t}\) is the step size, and \(\beta_{i,t}\) is the reward-extended TD error for subtask \(z_i\). 

Importantly, we will show that the reward-extended TD error satisfies the following property: \(\mathbb{E}_\pi[\beta_{i,t}] \to 0 \; \forall i=1, 2, \ldots, n\) as \(\mathbb{E}_\pi[\delta_t] \to 0\), where \(\delta_t\) is the regular TD error, such that minimizing the regular TD error allows us to solve all subtasks simultaneously. This motivates our naming of the reward-extended TD error, given that it is intrinsically tied to the regular TD error.

Let us begin by considering the common RL update rule of the form: \emph{NewEstimate} \(\leftarrow\) \emph{OldEstimate} \(+\) \emph{StepSize [Target \(-\) OldEstimate]} \citep{Sutton2018-eh, Naik2024-el}. Our aim is to find an appropriate set of subtask-specific ‘targets’, \(\{\phi_{i, t}\}_{i=1}^{n}\), such that \(\mathbb{E}_\pi[\beta_{i,t}] = \mathbb{E}_\pi[\phi_{i, t} - Z_{i,t}] \to 0 \; \forall i=1, 2, \ldots, n\) as \(\mathbb{E}_\pi[\delta_t] \to 0\). To this end, let us consider a generic piecewise linear subtask function with \(m\) piecewise segments:
\begin{equation}
\label{eq_red_1}
\tilde{R}_t =
\begin{cases} 
      b^{1}_rR_t + b^{1}_0 + b^{1}_1z_1 + b^{1}_2z_2 + \ldots + b^{1}_nz_n, \,\ r_0 \leq R_{t} < r_1 \\
      b^{2}_rR_t + b^{2}_0 + b^{2}_1z_1 + b^{2}_2z_2 + \ldots + b^{2}_nz_n, \,\ r_1 \leq R_{t} < r_2 \\
      \vdots \\
      b^{m}_rR_t + b^{m}_0 + b^{m}_1z_1 + b^{m}_2z_2 + \ldots + b^{m}_nz_n, \,\ r_{m-1} \leq R_{t} \leq r_m
   \end{cases},
\end{equation}
where \(r_k \in \mathcal{R} \; \forall \, k = 0, 1, \ldots, m\), and \(r_0 \leq r_1 \leq \ldots \leq r_m\), such that \(r_0, r_m\) represent the lower and upper bounds of the observed per-step reward, \(R_t\), respectively. Moreover,  \(b^{j}_r, b^{j}_0 \in \mathbb{R}\), and \(b^{j}_i \in \mathbb{R}\setminus{\{0\}}\), where \(b^{j}\) denotes a (predefined) constant in the \(j\)th piecewise segment of \(\tilde{R}_t\). 

Now, let us consider the TD error, \(\delta_t\), associated with \eqref{eq_red_1} in the prediction setting. Let \(\tilde{R}_{j, t}\) be shorthand for the \(j\)th segment of \eqref{eq_red_1}, such that the TD error at any time step can be expressed as:
\begin{subequations}
\label{eq_red_2}
\begin{align}
\delta_{j,t} &= \tilde{R}_{j,t+1} - \bar{R}_t + V_t(S_{t+1}) - V_t(S_t)\\
&= b^{j}_rR_{t+1} + b^{j}_0 + b^{j}_1Z_{1, t} + b^{j}_2Z_{2, t} + \ldots + b^{j}_nZ_{n, t} - \bar{R}_t + V_t(S_{t+1}) - V_t(S_t),
\end{align}
\end{subequations}
where \(V_t: \mathcal{S} \rightarrow \mathbb{R}\) denotes a table of state-value function estimates, \(\bar{R}_t\) denotes an estimate of the average-reward, \(\bar{r}_{\pi}\), \(Z_{i, t}\) denotes an estimate of subtask \(z_i \; \forall i=1, 2, \ldots, n\), and \(j\) corresponds to the piecewise condition, \( r_{j-1} \leq R_{t+1} \leq r_j\), that is satisfied by the  observed per-step reward, \(R_{t+1}\). 

Hence, as learning progresses, different \(\tilde{R}_{j, t+1}\) values will be used to define the TD error based on which piecewise condition is satisfied at a given time step. Moreover, we know that the probability that \(\delta_{t} = \delta_{j, t}\) is equal to the probability that \(r_{j-1} \leq R_{t+1} < r_j\). This allows us to express the expected TD error associated with \eqref{eq_red_1} as follows: 
\begin{equation}
\label{eq_red_3}
\mathbb{E}_\pi[\delta_t] = \sum_{j=1}^{m}{\mathbb{P}(r_{j-1} \leq R_{t+1} < r_j)\mathbb{E}_\pi[\delta_{j, t}}].
\end{equation}
Now, let us consider the implications of \(\mathbb{E}_\pi[\delta_t] \to 0\) as it relates to \(\mathbb{E}_\pi[\delta_{j, t}]\). One possibility is that \(\mathbb{E}_\pi[\delta_{j, t}] \to 0 \; \forall j=1, 2,\ldots, m\). However, this may not necessarily be the case; it is possible that, for example, a pair of non-zero \(\mathbb{P}(r_{j-1} \leq R_{t+1} < r_j)\mathbb{E}_\pi[\delta_{j, t}]\) terms cancel each other out, such that \(\mathbb{E}_\pi[\delta_t] \to 0\) but \(\mathbb{E}_\pi[\delta_{j, t}] \to \lambda_j \; \forall j=1, 2, \ldots, m\), where \(\lambda_j \in \mathbb{R}\). In such a case, what we do know is that if \(\mathbb{E}_\pi[\delta_t] \to 0\), then the Bellman equation \eqref{eq_avg_reward_4} must be satisfied, such that: \(V_t(s) = \mathbb{E}_{\pi}[\tilde{R}_{t+1} - \bar{R}_t + V_t(S_{t+1}) \mid S_t = s]\). As such, we can write the following expression for \(\lambda_j\), and solve for an arbitrary subtask, \(z_i\), as follows:
\begin{subequations}
\label{eq_red_4}
\begin{align}
\label{eq_red_4_1}
\lambda_j &= \mathbb{E}_\pi[\tilde{R}_{j,t+1} - \bar{R}_t + V_t(S_{t+1}) - V_t(S_t)]\\ 
\label{eq_red_4_2}
&= \mathbb{E}_\pi\left[\tilde{R}_{j,t+1} - \bar{R}_t + V_t(S_{t+1}) - \left( \tilde{R}_{t+1} - \bar{R}_t + V_t(S_{t+1}) \right)\right]\\
\label{eq_red_4_3}
&= \mathbb{E}_\pi[\tilde{R}_{j,t+1}] - \mathbb{E}_\pi[\tilde{R}_{t+1}]\\
\label{eq_red_4_4}
&= \mathbb{E}_\pi[\tilde{R}_{j,t+1}] - \bar{r}_\pi\\
\label{eq_red_4_5}
&= \mathbb{E}_\pi[b^{j}_rR_{t+1} + b^{j}_0 + \ldots + b^{j}_{i-1}z_{i-1} + b^{j}_{i+1}z_{i+1} + \ldots + b^{j}_nz_n - \bar{r}_{\pi}]  + b^{j}_iz_i\\
\label{eq_red_4_6}
\implies z_i &= \mathbb{E}_{\pi}\left[-\frac{1}{b^{j}_i}\left(b^{j}_rR_{t+1} + b^{j}_0 + \ldots + b^{j}_{i-1}z_{i-1} + b^{j}_{i+1}z_{i+1} + \ldots + b^{j}_nz_n - \bar{r}_{\pi} - \lambda_j\right)\right]\\
\label{eq_red_4_7}
&\doteq \mathbb{E}_{\pi}[\phi_{i, j}],
\end{align}
\end{subequations}
where we used the fact that \(z_i\) is independent of the states and actions to pull it out of the expectation. Here, we use \(\phi_{i, j}\) to denote the expression inside the expectation in Equation \eqref{eq_red_4_6}.

\newpage

Hence, to learn \(z_i\) from experience, we can utilize the common RL update rule, using the term inside the expectation in Equation \eqref{eq_red_4_7}, \(\phi_{i, j}\), as the ‘target’, which yields the update:
\begin{subequations}
\label{eq_red_5}
\begin{align}
\label{eq_red_5_1}
Z_{i, t+1} & = Z_{i, t} + \alpha_{z_i,t}
\begin{cases} 
      \phi_{i, 1, t} - Z_{i, t}, \,\ r_0 \leq R_{t+1} < r_1 \\
      \vdots \\
      \phi_{i, m, t} - Z_{i, t}, \,\ r_{m-1} \leq R_{t+1} \leq r_m
   \end{cases}\\
\label{eq_red_5_2}
& = Z_{i, t} + \alpha_{z_i,t}
\begin{cases} 
      (-1/b^{1}_i)\left(\tilde{R}_{1,t+1} - \bar{R}_t - \delta_t\right), \,\ r_0 \leq R_{t+1} < r_1 \\
      \vdots \\
      (-1/b^{m}_i)\left(\tilde{R}_{m,t+1} - \bar{R}_t - \delta_t\right), \,\ r_{m-1} \leq R_{t+1} \leq r_m
   \end{cases}\\
\label{eq_red_5_3}
& \doteq Z_{i, t} + \alpha_{z_i,t}\beta_{i, t},
\end{align}
\end{subequations}
where \(Z_{i, t}\) is the estimate of subtask \(z_i\) at time \(t\), \(\phi_{i, j, t} \doteq (-1/b^{j}_i)(b^{j}_rR_{t+1} + b^{j}_0 + \ldots + b^{j}_{i-1}Z_{i-1, t} + b^{j}_{i+1}Z_{i+1, t} + \ldots + b^{j}_nZ_{n, t} - \bar{R}_t - \delta_t)\), and \(\alpha_{z_i,t}\) is the step size. 

As such, we now have an expression for the reward-extended TD error for subtask \(z_i\), \(\beta_{i, t}\). We will now show that this term satisfies the desired property: \(\mathbb{E}_\pi[\beta_{i,t}] \to 0 \; \forall i=1, 2, \ldots, n\) as \(\mathbb{E}_\pi[\delta_t] \to 0\), such that minimizing the regular TD error allows us to solve all the subtasks simultaneously:
\begin{theorem}
\label{theorem_4_1}
Consider an average-reward MDP with a set of reward-extended TD errors, \(\{\beta_{i, t}\}_{i=1}^{n}\), as defined in Equation \eqref{eq_red_5}, corresponding to a subtask function with \(n\) subtasks that satisfy Definition \ref{definition_1}. The set of reward-extended TD errors, \(\{\beta_{i, t}\}_{i=1}^{n}\), satisfies the following property: \(\mathbb{E}_\pi[\beta_{i,t}] \to 0 \; \forall i=1, 2, \ldots, n\) as \(\mathbb{E}_\pi[\delta_t] \to 0\), where \(\beta_{i, t}\) denotes the reward-extended TD error for subtask \(z_i\), and \(\delta_t\) denotes the regular TD error.
\end{theorem}
\begin{proof}
Let us consider the reward-extended TD error associated with an arbitrary \(j\)th segment of \(\tilde{R}_t\) for an arbitrary \(i\)th subtask: \(\beta_{i, j, t} \doteq (-1/b^{j}_i)(\tilde{R}_{j,t+1} - \bar{R}_t - \delta_t)\). As \(\mathbb{E}_\pi[\delta_t] \to 0\), \(\bar{R}_t \to \bar{r}_\pi\) (by Theorem 3 of \citet{Wan2021-re}; see Remark \ref{remark_3}) and \(\delta_t \to \lambda_j\) for this \(j\)th segment. Hence, \(\mathbb{E}_\pi[\beta_{i, j, t}] \to (-1/b^{j}_i)(\mathbb{E}_\pi[\tilde{R}_{j,t+1}] - \bar{r}_\pi - \lambda_j) = (-1/b^{j}_i)(\lambda_j - \lambda_j) = 0\). Now, because we chose \(j\) arbitrarily, we have, for all \(j \in \{1, 2, \dots, m\}\), that \(\mathbb{E}_\pi[\beta_{i, j, t}] \to 0\). As such, and because we chose \(i\) arbitrarily, we can conclude that \(\mathbb{E}_\pi[\beta_{i,t}] = \sum_{j=1}^{m}{\mathbb{P}(r_{j-1} \leq R_{t+1} < r_j)\mathbb{E}_\pi[\beta_{i, j, t}}] \to 0 \; \forall i=1, 2, \ldots, n\) as \(\mathbb{E}_\pi[\delta_t] \to 0\). This completes the proof.
\end{proof}

As such, we have derived the desired update rule that we can use to solve any given subtask in the prediction setting. The same logic can be applied in the control setting to derive equivalent updates, where we note that it directly follows from Definition \ref{definition_1} that the existence of an optimal average-reward, \(\bar{r}*\), implies the existence of corresponding optimal subtask values, \(z_i^{*} \; \forall z_i \in \mathcal{Z}\).\\

\begin{remark}
\label{remark_1}
In the case of a (non-piecewise) linear subtask function, the expression for the reward-extended TD error can be simplified to \(\beta_{i, t} \doteq (-1/b_i)\delta_t\) by setting \(\lambda=0\) in Equation \eqref{eq_red_4_1}, solving for the target, \(z_i\), and applying a similar process to the one described in Equation \eqref{eq_red_5}.\\
\end{remark}

\begin{remark}
\label{remark_2}
Given Remark \ref{remark_1}, it can be shown that if one treats the average-reward, \(\bar{r}_\pi\), as a subtask, and derives the reward-extended TD error for it, the process yields the average-reward update (e.g. Equation \eqref{eq_avg_reward_6_4}) from the Differential algorithms proposed in \citet{Wan2021-re}. Hence, our work can be viewed as a generalization of the work performed in \citet{Wan2021-re}.\\
\end{remark}

\begin{remark}
\label{remark_3}
Strictly speaking, \(\bar{R}_t \to \bar{r}_\pi + c, \; c \in \mathbb{R}\). This is because average-reward solution methods typically find the solutions to the Bellman equations \eqref{eq_avg_reward_4} and \eqref{eq_avg_reward_5} up to an additive constant, \(c\). This means that, like the average-reward estimate, our subtask estimates converge to the actual subtask values, up to an additive constant. For simplicity, we omit this additive constant in our work, unless strictly necessary, given that it is commonplace to assume that solutions in the average-reward setting are correct up to an additive constant.
\end{remark}

\subsection{The RED Algorithms}

In this section, we introduce the \emph{RED RL algorithms}, which integrate the update rules derived in the previous section into the average-reward RL framework from \citet{Wan2021-re}. The full algorithms, including algorithms that utilize function approximation, are included in Appendix \ref{appendix_RED_algs}.

\vspace{1pt}

\textbf{RED TD-learning algorithm (tabular):}  We update a table of estimates, \(V_t: \mathcal{S} \rightarrow \mathbb{R}\) as follows:
\begin{subequations}
\label{eq_alg_1}
\begin{align}
\label{eq_alg_1_1}
& \tilde{R}_{t+1} = f(R_{t+1}, Z_{1, t}, Z_{2, t},\ldots,Z_{n, t})\\
\label{eq_alg_1_2}
& \delta_{t} = \tilde{R}_{t+1} - \bar{R}_t + V_{t}(S_{t+1}) - V_{t}(S_t)\\
\label{eq_alg_1_3}
& V_{t+1}(S_t) = V_{t}(S_t) + \alpha_{t}\rho_{t}\delta_{t}\\
\label{eq_alg_1_5}
& \bar{R}_{t+1} = \bar{R}_t + \alpha_{\bar{r},t}\rho_{t}\delta_{t}\\
\label{eq_alg_1_6}
& Z_{i, t+1} = Z_{i, t} + \alpha_{z_i,t}\rho_t\beta_{i, t}, \quad \forall z_i \in \mathcal{Z}
\end{align}
\end{subequations}
where \(R_{t}\) is the observed reward, \(Z_{i, t}\) is an estimate of subtask \(z_i\), \(\beta_{i, t}\) is the reward-extended TD error for subtask \(z_i\), \(\bar{R}_t\) is an estimate of the long-run average-reward of \(\tilde{R}_{t}\), \(\bar{r}_{\pi}\), \(\delta_t\) is the TD error, \(\rho_t\) is the importance sampling ratio, and \(\alpha_t\), \(\alpha_{\bar{r},t}\), and \(\alpha_{z_i,t}\) are the step sizes.

\citet{Wan2021-re} showed for their Differential TD-learning algorithm that \(R_t\) converges to \(\bar{r}_{\pi}\), and \(V_t\) converges to a solution of \(v\) in Equation \eqref{eq_avg_reward_4} for a given policy, \(\pi\). We now provide an equivalent theorem for our RED TD-learning algorithm, which also shows that \(Z_{i, t}\) converges to \(z_{i, {\pi}} \; \forall z_i \in \mathcal{Z}\), where \(z_{i, {\pi}}\) denotes the subtask value induced when following policy \(\pi\): 

\vspace{2pt}

\begin{theorem}[informal]
\label{theorem_4_2}
The RED TD-learning algorithm \eqref{eq_alg_1} converges, almost surely, \(\bar{R}_t\) to \(\bar{r}_{\pi}\), \(Z_{i, t}\) to \(z_{i, {\pi}} \; \forall z_i \in \mathcal{Z}\), and \(V_t\) to a solution of \(v\) in the Bellman equation \eqref{eq_avg_reward_4}, up to an additive constant, \(c\), if the following assumptions hold: 1) the Markov chain induced by the target policy, \(\pi\), is unichain, 2) every state–action pair for which \(\pi(a | s) > 0\) occurs an infinite number of times under the behavior policy, 3) the step sizes are decreased appropriately, 4) \(V_{t}\) is updated an infinite number of times for all states, such that the ratio of the update frequency of the most-updated state to the least-updated state is finite, and 5) the subtasks are in accordance with Definition \ref{definition_1}.
\end{theorem}

\vspace{5pt}

\textbf{RED Q-learning algorithm (tabular):}  We update \(Q_t: \mathcal{S} \, \times\, \mathcal{A} \rightarrow \mathbb{R}\) as follows:
\begin{subequations}
\label{eq_alg_2}
\begin{align}
\label{eq_alg_2_1}
& \tilde{R}_{t+1} = f(R_{t+1}, Z_{1, t}, Z_{2, t},\ldots,Z_{n, t})\\
\label{eq_alg_2_2}
& \delta_{t} = \tilde{R}_{t+1} - \bar{R}_t +\max_a Q_{t}(S_{t+1}, a) - Q_{t}(S_t, A_t)\\
\label{eq_alg_2_3}
& Q_{t+1}(S_t, A_t) = Q_{t}(S_t, A_t) + \alpha_{t}\delta_{t}\\
\label{eq_alg_2_5}
& \bar{R}_{t+1} = \bar{R}_t + \alpha_{\bar{r},t}\delta_{t}\\
\label{eq_alg_2_6}
& Z_{i, t+1} = Z_{i, t} + \alpha_{z_i,t}\beta_{i, t}, \quad \forall z_i \in \mathcal{Z}
\end{align}
\end{subequations}
where \(R_{t}\) is the observed reward, \(Z_{i, t}\) is an estimate of subtask \(z_i\), \(\beta_{i, t}\) is the reward-extended TD error for subtask \(z_i\), \(\bar{R}_t\) is an estimate of the long-run average-reward of \(\tilde{R}_{t}\), \(\bar{r}_{\pi}\), \(\delta_t\) is the TD error, and \(\alpha_t\), \(\alpha_{\bar{r},t}\), and \(\alpha_{z_i,t}\) are the step sizes. \citet{Wan2021-re} showed for their Differential Q-learning algorithm that \(R_t\) converges to \(\bar{r}*\), and \(Q_t\) converges to a solution of \(q\) in Equation \eqref{eq_avg_reward_5}. We now provide an equivalent theorem for our RED Q-learning algorithm, which also shows that \(Z_{i, t}\) converges to the corresponding optimal subtask value \(z_i^{*} \; \forall z_i \in \mathcal{Z}\): 

\vspace{2pt}

\begin{theorem}[informal]
\label{theorem_4_3}
The RED Q-learning algorithm \eqref{eq_alg_2} converges, almost surely, \(\bar{R}_t\) to \(\bar{r}*\), \(Z_{i, t}\) to \(z_i^{*} \; \forall z_i \in \mathcal{Z}\), \(\bar{r}_{\pi_t}\) to \(\bar{r}*\), \(z_{i, {\pi_t}}\) to \(z_i^{*} \; \forall z_i \in \mathcal{Z}\), and \(Q_t\) to a solution of \(q\) in the Bellman optimality equation \eqref{eq_avg_reward_5}, up to an additive constant, \(c\), where \(\pi_t\) is any greedy policy with respect to \(Q_t\), if the following assumptions hold: 1) the MDP is communicating, 2) the solution of \(q\) in \eqref{eq_avg_reward_5} is unique up to a constant, 3) the step sizes are decreased appropriately, 4) \(Q_{t}\) is updated an infinite number of times for all state-action pairs, such that the ratio of the update frequency of the most-updated state–action pair to the least-updated state–action pair is finite, and 5) the subtasks are in accordance with Definition \ref{definition_1}.
\end{theorem}

\vspace{5pt}

See Appendix \ref{appendix_proofs} for the formal version of these theorems, along with the full convergence proofs.

\section{Case Study: RED RL for CVaR Optimization}
\label{risk_red}
In this section, we present a case-study which illustrates how the subtask-driven approach that was derived in Section \ref{red} can be used to successfully optimize the CVaR risk measure, \emph{without} the use of an explicit bi-level optimization scheme (as in Equation \eqref{eq_cvar_3}), or an augmented state-space.

First, in order to leverage the RED RL framework for CVaR optimization, we need to derive a valid subtask function for CVaR that satisfies the requirements of Definition \ref{definition_1}. It turns out that we can use Equation \eqref{eq_cvar_2} as a basis for the subtask function. The details of the adaptation of Equation \eqref{eq_cvar_2} into a subtask function are presented in Appendix \ref{appendix_RED_CVAR}. Critically, as discussed in Appendix \ref{appendix_RED_CVAR}, optimizing the long-run average of the \emph{extended} reward (\(\tilde{R}_t\)) from this subtask function corresponds to optimizing the long-run CVaR of the \emph{observed} reward (\(R_t\)). Hence, we can utilize CVaR-specific versions of the RED algorithms presented in Equations \eqref{eq_alg_1} and \eqref{eq_alg_2} (or their non-tabular equivalents) to optimize VaR and CVaR, such that CVaR corresponds to the primary control objective (i.e., the \(\bar{r}_{\pi}\) that we want to optimize), and VaR is the (single) subtask. We call the resulting algorithms, the \emph{RED CVaR algorithms}. These algorithms, which are shown in full in Appendix \ref{appendix_RED_CVAR}, update CVaR in an analogous way to the average-reward (i.e., CVaR corresponds to \(\bar{R}_t\) in Equations \eqref{eq_alg_1} or \eqref{eq_alg_2}), and update VaR using a VaR-specific version of Equation \eqref{eq_alg_1_6} or \eqref{eq_alg_2_6} as follows:
\begin{equation}
\label{eq_red_old_cvar_1}
\text{VaR}_{t+1} =
\begin{cases} 
      \text{VaR}_t + \alpha_{_\text{VaR},t} \left(\delta_t + \text{CVaR}_t - \text{VaR}_t \right), & R_{t+1} \geq \text{VaR}_t \\
      \text{VaR}_t + \alpha_{_\text{VaR},t} \left(\left(\frac{\tau}{ \tau - 1}\right)\delta_t + \text{CVaR}_t - \text{VaR}_t \right), & R_{t+1} < \text{VaR}_t
   \end{cases} \,,
\end{equation}
where \(\tau\) is the CVaR parameter, \(\delta_t\) is the TD error, \(R_t\) is the observed reward, and \(\alpha_{_\text{VaR},t} \) is the step size. As such, given Theorems \ref{theorem_4_1} - \ref{theorem_4_3}, we now have a subtask-driven approach for CVaR optimization that is able to simultaneously optimize VaR and CVaR without the use of an explicit bi-level optimization scheme or an augmented state-space (see Appendix \ref{appendix_RED_CVAR} for a more formal argument).

We now present empirical results obtained when applying the RED CVaR algorithms on two RL tasks. The first task corresponds to a two-state environment that we created to test the RED CVaR algorithms. It is called the \emph{red-pill blue-pill} task (see Appendix \ref{appendix_RPBP}), where at every time step an agent can take either a ‘red pill’, which takes them to the ‘red world’ state, or a ‘blue pill’, which takes them to the ‘blue world’ state. Each state has its own characteristic per-step reward distribution, and in this case, for a sufficiently low CVaR parameter, \(\tau\), the red world state has a reward distribution with a lower (worse) mean but a higher (better) CVaR compared to the blue world state. As such, this task allows us to answer the following question: \emph{can the RED CVaR algorithms successfully enable the agent to learn a policy that prioritizes optimizing the reward CVaR over the average-reward?} In particular, we would expect that the RED CVaR algorithms learn a policy that prefers to stay in the red world, and that regular, risk-neutral Differential algorithms (i.e., from \citet{Wan2021-re}) learn a policy that prefers to stay in the blue world. This task is illustrated in Figure \ref{fig_rpbp}.
\begin{figure}[htbp]
\centerline{\includegraphics[scale=0.45]{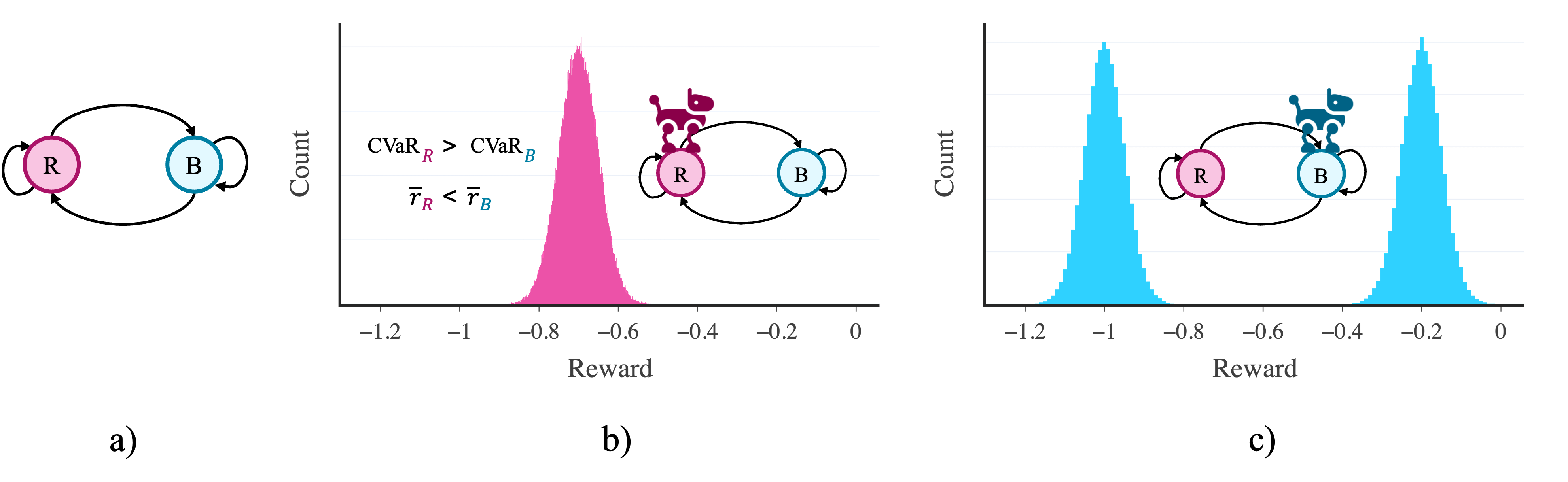}}
\caption{\textbf{a)} The \emph{red-pill blue-pill} environment. \textbf{b) + c)} The per-step reward distributions of the \textbf{b)} ‘red world’, and \textbf{c)} ‘blue world’ states. For a sufficiently low CVaR parameter, \(\tau\), the red world state has a lower (worse) average-reward but a higher (better) reward CVaR than the blue world state.}
\label{fig_rpbp}
\end{figure}

\newpage

The second task is the well-known \emph{inverted pendulum} task, where an agent learns how to optimally balance an inverted pendulum. We chose this task because it provides us with the opportunity to test our algorithms in an environment where: 1) we must use function approximation (given the high-dimensional state-space), and 2) where the optimal CVaR policy and the optimal average-reward policy are the same policy (i.e., the policy that best balances the pendulum will yield a limiting reward distribution with both the optimal average-reward and reward CVaR). This hence allows us to directly compare the performance of our RED CVaR algorithms to that of the Differential algorithms, as well as to gauge how function approximation affects the performance of our algorithms.

In terms of empirical results, Figure \ref{fig_results_1} shows rolling averages of the average-reward and reward CVaR as learning progresses in both tasks when using a regular (risk-neutral) Differential algorithm (to optimize the average-reward) vs. a RED CVaR algorithm (to optimize the reward CVaR). As shown in the figure, in the red-pill blue-pill task, the RED CVaR algorithm successfully enables the agent to learn a policy that prioritizes maximizing the reward CVaR over the average-reward, thereby achieving a sort of \emph{risk-awareness}. In the inverted pendulum task, both methods converge to the same policy, as expected.

\begin{figure}[htbp]
\centerline{\includegraphics[scale=0.485]{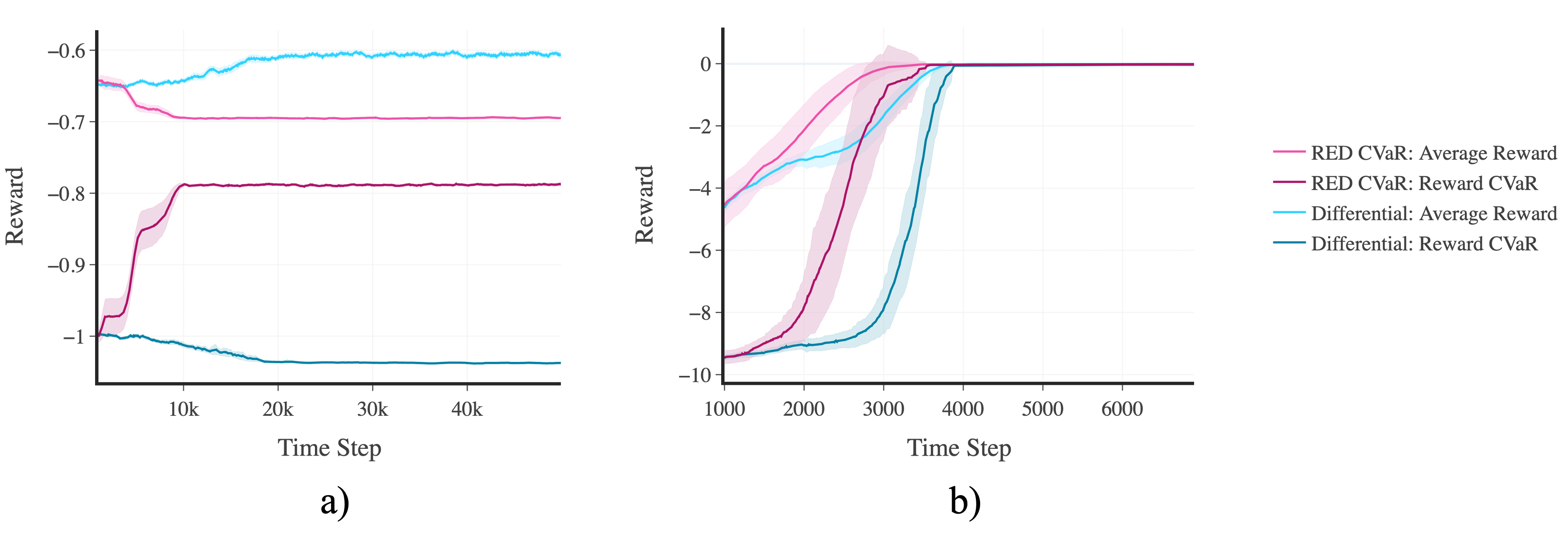}}
\caption{Rolling average-reward and reward CVaR as learning progresses when using a (risk-neutral) Differential algorithm vs. a (risk-aware) RED CVaR algorithm in the \textbf{a)} red-pill blue-pill, and \textbf{b)} inverted pendulum tasks. A solid line denotes the mean average-reward or reward CVaR, and the corresponding shaded region denotes a 95\% confidence interval over \textbf{a)} 50 runs, or \textbf{b)} 10 runs. In both tasks, the RED CVaR algorithms enable the agent to learn a policy that prioritizes maximizing the reward CVaR over the average-reward, thereby achieving a sort of \emph{risk-awareness}.}
\label{fig_results_1}
\end{figure}

\begin{figure}[htbp]
\centerline{\includegraphics[scale=0.49]{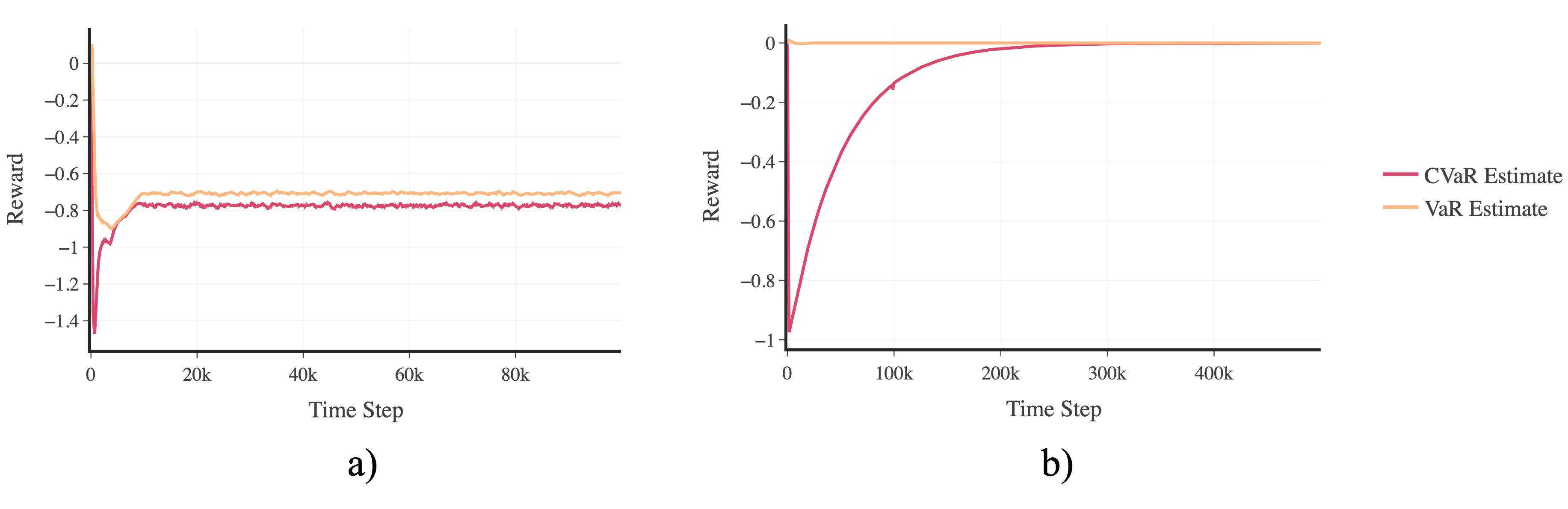}}
\caption{Typical convergence plots of the agent's VaR and CVaR estimates as learning progresses when using the RED CVaR algorithms in the \textbf{a)} red-pill blue-pill, and \textbf{b)} inverted pendulum tasks with an initial guess of 0.0 for both estimates. In both tasks, the estimates converge to the correct VaR and CVaR values, up to an additive constant, thereby yielding the optimal CVaR policy, and hence, the results shown in Figure \ref{fig_results_1}.}
\label{fig_results_2}
\end{figure}

\newpage

Figure \ref{fig_results_2} shows typical convergence plots of the agent's VaR and CVaR estimates as learning progresses in both tasks when using the RED CVaR algorithms. As shown in the figure, the estimates converge in both tasks. In particular, the estimates converge to the correct VaR and CVaR values, up to an additive constant, thereby yielding the optimal CVaR policy, and hence, the results in Figure \ref{fig_results_1}.

The full set of experimental details and results, including additional experiments performed, can be found in Appendix \ref{appendix_experiments}. 

\section{Discussion, Limitations, and Future Work}
\label{discussion}

In this work, we introduced \emph{reward-extended differential} (or \emph{RED}) reinforcement learning: a novel reinforcement learning framework that can be used to solve various learning objectives, or \emph{subtasks}, simultaneously in the average-reward setting. We introduced a family of RED RL algorithms for prediction and control, and then showcased how these algorithms could be utilized to effectively and efficiently tackle the CVaR optimization problem. More specifically, we were able to use the RED RL framework to derive a set of algorithms that can optimize the CVaR risk measure without using an explicit bi-level optimization scheme or an augmented state-space, thereby alleviating some of the computational challenges and complexities that arise when performing risk-based optimization in the discounted setting. Empirically, we showed that the RED-based CVaR algorithms fared well in both tabular and linear function approximation settings.

More broadly, our work has introduced a theoretically-sound framework that allows for a subtask-driven approach to reinforcement learning, where various learning objectives (or subtasks) are solved simultaneously to help solve a larger, central learning objective. In this work, we showed (both theoretically and empirically) how this framework can be utilized to predict and/or optimize any arbitrary number of subtasks simultaneously in the average-reward setting. Central to this result is the novel concept of the reward-extended TD error, which is utilized in our framework to develop learning rules for the subtasks, and satisfies key theoretical properties that make it possible to solve any given subtask in a fully-online manner by minimizing the regular TD error. Moreover, we built upon existing results from \citet{Wan2021-re} to show the almost sure convergence of tabular algorithms derived from our framework. While we have only begun to grasp the implications of our framework, we have already seen some promising indications in the CVaR case study: the ability to turn explicit bi-level optimization problems into implicit bi-level optimizations that can be solved in a fully-online manner, as well as the potential to turn certain states (that meet certain conditions) into subtasks, thereby reducing the size of the state-space.

Nonetheless, while these results are encouraging, they are subject to a number of limitations. Firstly, by nature of operating in the average-reward setting, we are subject to the somewhat-strict assumptions made about the Markov chain induced by the policy (e.g. unichain or communicating). These assumptions could restrict the applicability of our framework, as they may not always hold in practice. Similarly, our definition for a subtask requires that the associated subtask function be linear or piecewise linear with respect to the subtasks, which may limit the applicability of our framework to simpler subtask functions. Finally, it remains to be seen empirically how our framework performs when dealing with multiple subtasks, when taking on more complex tasks, and/or when utilizing nonlinear function approximation. Importantly, we emphasize that the empirical evaluation performed in this work is limited, and as such, a more comprehensive empirical study is needed to fully gauge the practical implications of the proposed framework and CVaR algorithms.

Future work should look to address these limitations, as well as explore how these promising results can be extended to other domains, beyond the risk-awareness problem. In particular, we believe that the ability to optimize various subtasks simultaneously, as well as the potential to reduce the size of the state-space, by converting certain states to subtasks (where appropriate), could help alleviate significant computational challenges in other areas moving forward. 

\newpage

\section*{Acknowledgments}
We gratefully acknowledge funding from NSERC Discovery Grant \# RGPIN-2021-02760. We thank Margaret P. Chapman for insightful conversations and feedback in the early stages of this work. We thank Noah Sheridan for providing an insightful review of this work. We thank anonymous ICLR reviewer dBpb for their valuable feedback and suggestions during an earlier review process. Finally, we thank the anonymous RLC/RLJ reviewers and area chair for their useful feedback and commentary during the review process.

\bibliographystyle{plainnat}
\bibliography{references}


\newpage
\appendix
\numberwithin{equation}{section}
\numberwithin{figure}{section}
\numberwithin{theorem}{subsection}

\section{RED RL Algorithms}
\label{appendix_RED_algs}

In this appendix, we provide pseudocode for our RED RL algorithms. We first present tabular algorithms, whose convergence proofs are included in Appendix \ref{appendix_proofs}, and then provide equivalent algorithms that utilize function approximation.

\begin{algorithm}
   \caption{RED TD-Learning (Tabular)}
   \label{alg_a_1}
\begin{algorithmic}
    \STATE {\bfseries Input:} the policy \(\pi\) to be evaluated, policy \(B\) to be used, piecewise linear subtask function \(f\) with \(n\) subtasks, \(m\) piecewise segments, piecewise conditions \(r_{j-1} \leq R < r_j\) such that \(f_j\) denotes the \(j\)th segment of \(f\) that satisfies \(r_{j-1} \leq R < r_j\), and constants \(b^{j}_1, b^{j}_2, \ldots, b^{j}_n \; \forall j=1, 2, \ldots, m\)
    \STATE {\bfseries Algorithm parameters:} step size parameters \(\alpha\), \(\alpha_{\bar{r}}\), \(\alpha_{z_1}, \alpha_{z_2}, \ldots, \alpha_{z_n}\)
    \STATE Initialize \(V(s) \: \forall s; \bar{R}\) arbitrarily (e.g. to zero)
    \STATE Initialize subtasks \(Z_1, Z_2, \ldots, Z_n\) arbitrarily  (e.g. to zero)
    \STATE Obtain initial \(S\)
    \WHILE{still time to train}
        \STATE \(A \leftarrow\) action given by \(B\) for \(S\)
        \STATE Take action \(A\), observe \(R, S'\)
        \STATE \(\tilde{R} = f(R, Z_1, Z_2, \ldots, Z_n)\)
        \STATE \(\delta = \tilde{R} - \bar{R} + V(S') - V(S)\)
        \STATE \(\rho = \pi(A \mid S) / B(A \mid S)\)
        \STATE \(V(S) = V(S) + \alpha\rho\delta\)
        \STATE \(\bar{R} = \bar{R} + \alpha_{\bar{r}}\rho\delta\)
        \FOR{$i = 1, 2, \ldots, n$}
            \STATE $\beta_{i} = \sum_{j=1}^{m}(-1/b^{j}_i)(f_j - \bar{R} - \delta) \mathds{1}\{r_{j-1} \leq R < r_j\}$ (see Remark \ref{remark_1} for non-piecewise \(f\))
            \STATE $Z_i = Z_i + \alpha_{z_i}\rho\beta_{i}$
        \ENDFOR
        \STATE \(S = S'\)
    \ENDWHILE
    \STATE return V
\end{algorithmic}
\end{algorithm}

\begin{algorithm}
   \caption{RED Q-Learning (Tabular)}
   \label{alg_a_2}
\begin{algorithmic}
    \STATE {\bfseries Input:} the policy \(\pi\) to be used (e.g., \(\varepsilon\)-greedy), piecewise linear subtask function \(f\) with \(n\) subtasks, \(m\) piecewise segments, piecewise conditions \(r_{j-1} \leq R < r_j\) such that \(f_j\) denotes the \(j\)th segment of \(f\) that satisfies \(r_{j-1} \leq R < r_j\), and constants \(b^{j}_1, b^{j}_2, \ldots, b^{j}_n \; \forall j=1, 2, \ldots, m\)
    \STATE {\bfseries Algorithm parameters:} step size parameters \(\alpha\), \(\alpha_{\bar{r}}\), \(\alpha_{z_1}, \alpha_{z_2}, \ldots, \alpha_{z_n}\)
    \STATE Initialize \(Q(s, a) \: \forall s, a; \bar{R}\) arbitrarily (e.g. to zero)
    \STATE Initialize subtasks \(Z_1, Z_2, \ldots, Z_n\) arbitrarily  (e.g. to zero)
    \STATE Obtain initial \(S\)
    \WHILE{still time to train}
        \STATE \(A \leftarrow\) action given by \(\pi\) for \(S\)
        \STATE Take action \(A\), observe \(R, S'\)
        \STATE \(\tilde{R} = f(R, Z_1, Z_2, \ldots, Z_n)\)
        \STATE \(\delta = \tilde{R} - \bar{R} + \max_a Q(S', a) - Q(S, A)\)
        \STATE \(Q(S, A) = Q(S, A) + \alpha\delta\)
        \STATE \(\bar{R} = \bar{R} + \alpha_{\bar{r}} \delta\)
        \FOR{$i = 1, 2, \ldots, n$}
            \STATE $\beta_{i} = \sum_{j=1}^{m}(-1/b^{j}_i)(f_j - \bar{R} - \delta) \mathds{1}\{r_{j-1} \leq R < r_j\}$ (see Remark \ref{remark_1} for non-piecewise \(f\))
            \STATE $Z_i = Z_i + \alpha_{z_i}\beta_{i}$
        \ENDFOR
        \STATE \(S = S'\)
    \ENDWHILE
    \STATE return Q
\end{algorithmic}
\end{algorithm}

\newpage

\begin{algorithm}
   \caption{RED TD-Learning (Function Approximation)}
   \label{alg_a_3}
\begin{algorithmic}
    \STATE {\bfseries Input:} the policy \(\pi\) to be evaluated, policy \(B\) to be used, a differentiable state-value function parameterization: \(\hat{v}(s, \boldsymbol{w})\), piecewise linear subtask function \(f\) with \(n\) subtasks, \(m\) piecewise segments, piecewise conditions \(r_{j-1} \leq R < r_j\) such that \(f_j\) denotes the \(j\)th segment of \(f\) that satisfies \(r_{j-1} \leq R < r_j\), and constants \(b^{j}_1, b^{j}_2, \ldots, b^{j}_n \; \forall j=1, 2, \ldots, m\)
    \STATE {\bfseries Algorithm parameters:} step size parameters \(\alpha\), \(\alpha_{\bar{r}}\), \(\alpha_{z_1}, \alpha_{z_2}, \ldots, \alpha_{z_n}\)
    \STATE Initialize state-value weights \(\boldsymbol{w} \in \mathbb{R}^{d}\) arbitrarily (e.g. to \(\boldsymbol{0}\))
    \STATE Initialize subtasks \(Z_1, Z_2, \ldots, Z_n\) arbitrarily  (e.g. to zero)
    \STATE Obtain initial \(S\)
    \WHILE{still time to train}
        \STATE \(A \leftarrow\) action given by \(B\) for \(S\)
        \STATE Take action \(A\), observe \(R, S'\)
        \STATE \(\tilde{R} = f(R, Z_1, Z_2, \ldots, Z_n)\)
        \STATE \(\delta = \tilde{R} - \bar{R} + \hat{v}(S', \boldsymbol{w}) - \hat{v}(S, \boldsymbol{w})\)
        \STATE \(\rho = \pi(A \mid S) / B(A \mid S)\)
        \STATE \(\boldsymbol{w} = \boldsymbol{w} + \alpha\rho\delta\nabla\hat{v}(S, \boldsymbol{w})\)
        \STATE \(\bar{R} = \bar{R} + \alpha_{\bar{r}}\rho\delta\)
        \FOR{$i = 1, 2, \ldots, n$}
            \STATE $\beta_{i} = \sum_{j=1}^{m}(-1/b^{j}_i)(f_j - \bar{R} - \delta) \mathds{1}\{r_{j-1} \leq R < r_j\}$ (see Remark \ref{remark_1} for non-piecewise \(f\))
            \STATE $Z_i = Z_i + \alpha_{z_i}\rho\beta_{i}$
        \ENDFOR
        \STATE \(S = S'\)
    \ENDWHILE
    \STATE return \(\boldsymbol{w}\)
\end{algorithmic}
\end{algorithm}

\begin{algorithm}
   \caption{RED Q-Learning (Function Approximation)}
   \label{alg_a_4}
\begin{algorithmic}
    \STATE {\bfseries Input:} the policy \(\pi\) to be used (e.g., \(\varepsilon\)-greedy), a differentiable state-action value function parameterization: \(\hat{q}(s, a, \boldsymbol{w})\), piecewise linear subtask function \(f\) with \(n\) subtasks, \(m\) piecewise segments, piecewise conditions \(r_{j-1} \leq R < r_j\) such that \(f_j\) denotes the \(j\)th segment of \(f\) that satisfies \(r_{j-1} \leq R < r_j\), and constants \(b^{j}_1, b^{j}_2, \ldots, b^{j}_n \; \forall j=1, 2, \ldots, m\)
    \STATE {\bfseries Algorithm parameters:} step size parameters \(\alpha\), \(\alpha_{\bar{r}}\), \(\alpha_{z_1}, \alpha_{z_2}, \ldots, \alpha_{z_n}\)
    \STATE Initialize state-action value weights \(\boldsymbol{w} \in \mathbb{R}^{d}\) arbitrarily (e.g. to \(\boldsymbol{0}\))
    \STATE Initialize subtasks \(Z_1, Z_2, \ldots, Z_n\) arbitrarily  (e.g. to zero)
    \STATE Obtain initial \(S\)
    \WHILE{still time to train}
        \STATE \(A \leftarrow\) action given by \(\pi\) for \(S\)
        \STATE Take action \(A\), observe \(R, S'\)
        \STATE \(\tilde{R} = f(R, Z_1, Z_2, \ldots, Z_n)\)
        \STATE \(\delta = \tilde{R} - \bar{R} +\max_a \hat{q}(S', a, \boldsymbol{w}) - \hat{q}(S, A, \boldsymbol{w})\)
        \STATE \(\boldsymbol{w} = \boldsymbol{w} + \alpha\delta\nabla\hat{q}(S, A, \boldsymbol{w})\)
        \STATE \(\bar{R} = \bar{R} + \alpha_{\bar{r}} \delta\)
        \FOR{$i = 1, 2, \ldots, n$}
            \STATE $\beta_{i} = \sum_{j=1}^{m}(-1/b^{j}_i)(f_j - \bar{R} - \delta) \mathds{1}\{r_{j-1} \leq R < r_j\}$ (see Remark \ref{remark_1} for non-piecewise \(f\))
            \STATE $Z_i = Z_i + \alpha_{z_i}\beta_{i}$
        \ENDFOR
        \STATE \(S = S'\)
    \ENDWHILE
    \STATE return \(\boldsymbol{w}\)
\end{algorithmic}
\end{algorithm}


\newpage
\section{Convergence Proofs}
\label{appendix_proofs}

In this appendix, we present the full convergence proofs for the tabular RED TD-learning and tabular RED Q-learning algorithms. Our general strategy is as follows: we first show that the results from \citet{Wan2021-re}, which show the almost sure convergence of the value function and average-reward estimates of differential algorithms, are applicable to our algorithms. We then build upon these results to show that the subtask estimates of our algorithms converge as well.

For consistency, we adopt similar notation as \citet{Wan2021-re} for our proofs:
\begin{itemize}\itemsep0mm
    \item For a given vector \(x\), let \(\sum x\) denote the sum of all elements in \(x\), such that \(\sum x \doteq \sum_{i} x(i)\).
    \item Let \(\bar{r}_*\) denote the optimal average-reward.
    \item Let \(z_{i_*}\) denote the corresponding optimal subtask value for subtask \(z_i \in \mathcal{Z}\).
\end{itemize}

\subsection{Convergence Proof for the Tabular RED TD-learning Algorithm} 
\label{proof_red_td}

In this section, we present the proof for the convergence of the value function, average-reward, and subtask estimates of the RED TD-learning algorithm. Similar to what was done in \citet{Wan2021-re}, we will begin by considering a general algorithm, called \emph{General RED TD}. We will first define General RED TD, then show how the RED TD-learning algorithm is a special case of this algorithm. We will then provide the necessary assumptions, state the convergence theorem of General RED TD, and then provide a proof for the theorem, where we show that the value function, average-reward, and subtask estimates converge, thereby showing that the RED TD-learning algorithm converges. We begin by introducing the General RED TD algorithm: 

Consider an MDP \(\mathcal{M} \doteq \langle\mathcal{S}, \mathcal{A}, \mathcal{R}, p \rangle\), a behavior policy, \(B\), and a target policy, \(\pi\). Given a state \(s \in \mathcal{S}\) and discrete step \(n \geq 0\), let \(A_n(s) \sim B(\cdot \mid s)\) denote the action selected using the behavior policy, let \(R_n(s, A_n(s)) \in \mathcal{R}\) denote a sample of the resulting reward, and let \(S'_n(s, A_n(s)) \sim p(\cdot, \cdot \mid s, a)\) denote a sample of the resulting state. Let \(\{Y_n\}\) be a set-valued process taking values in the set of nonempty subsets of \(\mathcal{S}\), such that: \(Y_n = \{s: s\) component of the \(\vert \mathcal{S} \vert\)-sized table of state-value estimates, \(V\), that was updated at step \(n\}\). Let \(\nu(n, s) \doteq \sum_{j=0}^n I\{s \in Y_j\}\), where \(I\) is the indicator function, such that \(\nu(n, s)\) represents the number of times that \(V(s)\) was updated up until step \(n\). 

Now, let \(f\) be a valid subtask function (see Definition \ref{definition_1}), such that \(\tilde{R}_n(s, A_n(s)) \doteq f(R_n(s, A_n(s)), Z_{1, n}, Z_{2, n}, \ldots, Z_{k, n})\) for \(k\) subtasks \(\in \mathcal{Z}\), where \(\tilde{R}_n(s, A_n(s))\) is the extended reward, \(\mathcal{Z}\) is the set of subtasks, and \(Z_{i, n}\) denotes the estimate of subtask \(z_i \in \mathcal{Z}\) at step \(n\). Consider an MDP with the extended reward: \(\mathcal{\tilde{M}} \doteq \langle\mathcal{S}, \mathcal{A}, \mathcal{\tilde{R}}, \tilde{p} \rangle\), such that \(\tilde{R}_n(s, A_n(s)) \in \mathcal{\tilde{R}}\). The update rules of General RED TD for this MDP are as follows, \(\forall n \geq 0\):
\begin{align}
    V_{n+1}(s) & \doteq V_n(s) + \alpha_{\nu(n, s)} \rho_n(s) \delta_n(s) I\{s \in Y_n\}, \quad \forall s \in \mathcal{S}, \label{async_td_value_update_eqn}\\
    \bar{R}_{n+1} & \doteq \bar{R}_n + \sum_s \alpha_{\bar{r}, \nu(n, s)} \rho_n(s) \delta_n(s) I\{s \in Y_n\} \label{async_td_r_bar_update_eqn},\\
    Z_{i, n+1} & \doteq Z_{i, n} + \sum_{s} \alpha_{z_i, \nu(n, s)} \rho_n(s)\beta_{i, n}(s) I\{s \in Y_n\}, \quad \forall z_i \in \mathcal{Z}, \label{async_td_z_update_eqn}
\end{align}
where,
\begin{align}
\begin{split}
    \delta_n(s) & \doteq \tilde{R}_n(s, A_n(s)) - \bar{R}_n + V_n(S_n'(s, A_n(s))) - V_n(s)\\
    & = f(R_n(s, A_n(s)), Z_{1, n}, Z_{2, n}, \ldots, Z_{k, n}) - \bar{R}_n + V_n(S_n'(s, A_n(s))) - V_n(s),
\end{split}\label{async_td_td_error_eqn}
\end{align}
and,
\begin{align}
    \beta_{i, n}(s) & \doteq \phi_{i, n}(s) -  Z_{i, n}, \quad \forall z_i \in \mathcal{Z}.\label{async_td_red_td_error_eqn}
\end{align}

\newpage

Here, \(\rho_n(s) \doteq \pi(A_n(s) \mid s) \, / \, B(A_n(s) \mid s)\) denotes the importance sampling ratio (with behavior policy, \(B\)), \(\bar{R}_n\) denotes the estimate of the average-reward (see Equation \eqref{eq_avg_reward_2}), \(\delta_n(s)\) denotes the TD error, \(\phi_{i, n}(s)\) denotes the (potentially-piecewise) subtask target, as defined in Section \ref{reward_extended_td}, and \(\alpha_{\nu(n, s)}\), \(\alpha_{\bar{r}, \nu(n, s)}\), and \(\alpha_{z_i, \nu(n, s)}\) denote the step sizes at time step \(n\) for state \(s\).

We now show that the RED TD-learning algorithm is a special case of the General RED TD algorithm. Consider a sequence of experience from our MDP, \(\mathcal{\tilde{M}}\): \(S_t, A_t(S_t), \tilde{R}_{t+1}, S_{t+1}, \ldots\)\, . Now recall the set-valued process \(\{Y_n\}\). If we let \(n\) = time step \(t\), we have: 
\begin{align*}
Y_t(s) = 
\begin{cases}
    1, s = S_t,\\
    0, \text{ otherwise,}
\end{cases}
\end{align*}
as well as \(S'_n(S_t, A_t(S_t)) = S_{t+1}\), \(R_n(S_t, A_t) = R_{t+1}\), and \(\tilde{R}_n(S_t, A_t(S_t)) = \tilde{R}_{t+1}\).\\ 

Hence, update rules \eqref{async_td_value_update_eqn}, \eqref{async_td_r_bar_update_eqn}, \eqref{async_td_z_update_eqn}, \eqref{async_td_td_error_eqn}, and \eqref{async_td_red_td_error_eqn} become:
\begin{align}
    V_{t+1}(S_t) & \doteq V_t (S_t) + \alpha_{\nu(t, S_t)}\rho_t(S_t) \delta_t \text{\  and\ } V_{t+1}(s) \doteq V_t (s), \forall s \neq S_t, \\
    \bar{R}_{t+1} & \doteq \bar{R}_t + \alpha_{\bar{r}, \nu(t, S_t)} \rho_t(S_t) \delta_t,\\
    Z_{i,t+1} & \doteq Z_{i, t} + \alpha_{z_i, \nu(t, S_t)} \rho_t(S_t) \beta_{i, t}, \quad \forall z_i \in \mathcal{Z},\\
    \begin{split}
    \delta_t & \doteq \tilde{R}_{t+1} - \bar{R}_t + V_t (S_{t+1}) - V_t (S_t),\\
    & = f(R_{t+1}, Z_{1, t}, Z_{2, t}, \ldots, Z_{k, t}) - \bar{R}_t + V_t (S_{t+1}) - V_t (S_t),
    \end{split}\\
    \beta_{i, t} & \doteq \phi_{i, t} - Z_{i, t}, \quad \forall z_i \in \mathcal{Z},
\end{align}

which are RED TD-learning's update rules with \(\alpha_{\nu(t, S_t)}\), \(\alpha_{\bar{r}, \nu(t, S_t)}\), and \(\alpha_{z_i, \nu(t, S_t)}\) denoting the step sizes at time \(t\).\\

We now specify the assumptions on General RED TD that are needed to ensure convergence:\\ 

\begin{assumption}[Unichain Assumption]\label{assumption_unichain}
The Markov chain induced by the policy is unichain. That is, the induced Markov chain consists of a single recurrent class and a potentially-empty set of transient states.\\
\end{assumption}

\begin{assumption}[Coverage Assumption] \label{assumption_coverage}
\(B(a \mid s) > 0\) if \(\pi(a \mid s) > 0\) for all \(s \in \mathcal{S}\), \(a \in \mathcal{A}\).\\ 
\end{assumption}

\begin{assumption}[Step Size Assumption] \label{assumption_step_size} \(\alpha_n > 0\), \(\sum_{n = 0}^\infty \alpha_n = \infty\), \(\sum_{n = 0}^\infty \alpha_n^2 < \infty\).\\
\end{assumption}

\begin{assumption}[Asynchronous Step Size Assumption 1] \label{assumption_async_step_size_1}
Let \([\cdot]\) denote the integer part of \((\cdot)\). For \(x \in (0, 1)\), 
\begin{align*}
    \sup_i \frac{\alpha_{[xi]}}{\alpha_i} < \infty
\end{align*}
and 
\begin{align*}
    \frac{\sum_{j=0}^{[yi]} \alpha_j}{\sum_{j=0}^i \alpha_j} \to 1
\end{align*} 
uniformly in \(y \in [x, 1]\).\\
\end{assumption}

\newpage

\begin{assumption}[Asynchronous Step Size Assumption 2] \label{assumption_async_step_size_td_2}
There exists \(\Delta > 0\) such that 
\begin{align*}
    \liminf_{n \to \infty} \frac{\nu(n, s)}{n+1} \geq \Delta,
\end{align*}
a.s., for all \(s \in \mathcal{S}\).\\

Furthermore, for all \(x > 0\), and
\begin{align*}
    N(n, x) = \min \Bigg \{m \geq n: \sum_{i = n+1}^m \alpha_i \geq x \Bigg\},
\end{align*}
the limit 
\begin{align*}
    \lim_{n \to \infty} \frac{\sum_{i = \nu(n, s)}^{\nu(N(n, x), s)} \alpha_i}{\sum_{i = \nu(n, s')}^{\nu(N(n, x), s')} \alpha_i}
\end{align*} 

exists a.s. for all \(s, s'\).\\
\end{assumption}

\begin{assumption}[Average-Reward Step Size Assumption] \label{assumption_r_step_size}
The average-reward step size, \(\alpha_{\bar{r}, n}\), can be written as a constant fraction of the value function step size, \(\alpha_{n}\), such that \(\alpha_{\bar{r}, n} \doteq \eta_{_r}\alpha_{n}\), where \(\eta_{_r}\) is a positive scalar.\\
\end{assumption}

\begin{assumption}[Subtask Function Assumption] \label{assumption_subtask_function}
The subtask function, \(f\), is 1) linear or piecewise linear, and 2) is invertible with respect to each input given all other inputs.\\
\end{assumption}

\begin{assumption}[Subtask Independence Assumption] \label{assumption_subtask_independence}
Each subtask, \(z_i \in \mathcal{Z}\), in \(f\) is independent of the states and actions, and hence independent of the observed reward, \(R_n\), such that \(\tilde{p}(s', f(r, z_1, \ldots, z_k) | s, a) = p(s', r | s, a)\), and \(\mathbb{E}[f_j(R_n, Z_{1, n}, Z_{2, n}, \ldots, Z_{k, n})] = f_j(\mathbb{E}[R_n],Z_{1, n}, Z_{2, n}, \ldots, Z_{k, n})\), where \(f_j\) denotes the \(j\)th piecewise segment of \(f\), and \(\mathbb{E}\) denotes any expectation taken with respect to the states and actions.\\
\end{assumption}

\begin{assumption}[Subtask Uniqueness Assumption] \label{assumption_subtask_unique}
If the Bellman equation associated with \(f(R_n,z_1, z_2, \ldots, z_k)\) admits a unique solution, then that unique solution corresponds to a unique combination of subtasks, \(z_1, z_2, \ldots, z_k\).\\
\end{assumption}

\begin{assumption}[Subtask Step Size Assumptions] \label{assumption_subtask_stepsize} If the subtask function is strictly (i.e., non-piecewise) linear, the subtask step sizes, \(\{\alpha_{z_i, n}\}_{i=1}^{k}\), can be written as constant, subtask-specific fractions of the value function step size, \(\alpha_{n}\), such that \(\alpha_{z_i, n} \doteq \eta_{z_i}\alpha_{n} \; \forall z_i \in \mathcal{Z}\), where \(\{\eta_{z_i}\}_{i=1}^{k}\) are positive scalars. Alternatively, if the subtask function is piecewise linear with at least two piecewise segments, the subtask step sizes satisfy the following properties: 
\(\alpha_{z_1, n} / \alpha_{n} \to 0\), \(\alpha_{z_1, n} / \alpha_{\bar{r}, n} \to 0\), \(\{\alpha_{z_i, n} / \alpha_{z_{i - 1}, n} \to 0\}_{i=2}^{k}\), and \(\sum_{n = 0}^\infty (\alpha_{n}^2 + \alpha_{\bar{r}, n}^2 + \alpha_{z_1, n}^2 + \alpha_{z_2, n}^2 + \ldots + \alpha_{z_k, n}^2)< \infty\).
\end{assumption}

We refer the reader to \citet{Wan2021-re} for an in-depth discussion on Assumptions \ref{assumption_unichain} -- \ref{assumption_r_step_size}. Note that Assumptions \ref{assumption_step_size} -- \ref{assumption_async_step_size_td_2} apply to the value function, average-reward, and subtask step sizes. Assumptions~\ref{assumption_subtask_function} -- \ref{assumption_subtask_stepsize} outline the subtask-related requirements needed to show convergence. In particular, Assumption~\ref{assumption_subtask_function} ensures that we can explicitly write out the update \eqref{async_td_z_update_eqn}, and Assumption~\ref{assumption_subtask_independence} ensures that we do not break the Markov property in the process (i.e., we preserve the Markov property by ensuring that the subtasks are independent of the states and actions, and thereby also independent of the observed reward). Assumption \ref{assumption_subtask_unique} ensures that only a unique combination of subtasks yields the solution to a given Bellman equation. Finally, Assumption~\ref{assumption_subtask_stepsize} outlines additional step size requirements needed to show convergence.\\

Having stated the necessary assumptions, we next point out that it is easy to verify that under Assumption~\ref{assumption_unichain}, the following system of equations:
\begin{align}
    \begin{split}
    v_{\pi}(s) & = \sum_a \pi(a \mid s) \sum_{s', \tilde{r}} \tilde{p}(s', \tilde{r} \mid s, a) (\tilde{r} - \bar{r}_{\pi} + v_{\pi}(s')), \quad \forall s\in\mathcal{S}\\
    & =  \sum_a \pi(a \mid s) \sum_{s', r} p(s', r \mid s, a) (f(r, z_{1,\pi}, z_{2,\pi}, \ldots, z_{k,\pi}) - \bar{r}_{\pi} + v_{\pi}(s')), \quad \forall s\in\mathcal{S},\\
    \end{split}\label{td_system_of_eqns}
\end{align}
and,
\begin{align}
    \bar{r}_{\pi} - \bar{R}_0 & = \eta_{_r} \left(\sum v_{\pi} - \sum V_0 \right),  \label{td_r_bar_sys_eqns}\\
    z_{i, \pi} - Z_{i,0} & = \eta_{_i} \left(\sum v_{\pi} - \sum V_0\right), \text{\ for all } z_i \in \mathcal{Z},  \label{td_z_sys_eqns}
\end{align}

has a unique solution of \(v_{\pi}\), where \(\bar{r}_{\pi}\) denotes the average-reward induced by following a given policy, \(\pi\), and \(z_{i, \pi}\) denotes the corresponding subtask value for subtask \(z_i \in \mathcal{Z}\). Denote this unique solution of \(v_{\pi}\) as \(v_\infty\).\\

We are now ready to state the convergence theorem:\\

\begin{theorem}[Convergence of General RED TD]\label{theorem_convergence_of_red_td_update}

If Assumptions~\ref{assumption_unichain} -- \ref{assumption_subtask_stepsize} hold, then General RED TD (Equations \eqref{async_td_value_update_eqn} -- \eqref{async_td_red_td_error_eqn}) converges a.s., \(\bar{R}_n\) to \(\bar{r}_{\pi}\), \(Z_{i, n}\) to \(z_{i, \pi} \; \forall z_i \in \mathcal{Z}\), and \(V_n\) to \(v_\infty\).
\end{theorem}

We prove this theorem in Sections \ref{proof_v_linear} and \ref{proof_v_piecewise}. To do so, we first show that General RED TD is of the same form as \emph{General Differential TD} from \citet{Wan2021-re}, thereby allowing us to apply their convergence results for the value function and average-reward estimates of General Differential TD to General RED TD. We then build upon these results, using similar techniques as \citet{Wan2021-re}, to show that the subtask estimates converge as well.

\subsubsection{Proof of Theorem \ref{theorem_convergence_of_red_td_update} (for \emph{Linear} Subtask Functions)}
\label{proof_v_linear}

We first provide the proof for \emph{linear} subtask functions, where the reward-extended TD error can be expressed as a constant, subtask-specific fraction of the regular TD error, such that \(\beta_{i,n}(s) = (-1/b_i)\delta_n(s)\). We consider the \emph{piecewise linear} case in Section \ref{proof_v_piecewise}.\\

\textbf{Convergence of the value function and average-reward estimates:}

Consider the increment to \(\bar{R}_n\) at each step. Given Assumption \ref{assumption_r_step_size}, we can see from Equation \eqref{async_td_r_bar_update_eqn} that the increment is \(\eta_{_r}\) times the increment to \(V_n\). As such, as was done in \citet{Wan2021-re}, we can write the cumulative increment as follows:
\begin{align}
    \bar{R}_n - \bar{R}_0 & = \eta_{_r} \sum_{j = 0}^{n-1} \sum_{s} \alpha_{\nu(j, s)} \rho_j(s) \delta_j (s) I\{s \in Y_j\} \nonumber\\\nonumber\\
    & = \eta_{_r} \left (\sum V_{n} - \sum V_0 \right) \nonumber\\\nonumber\\
    \implies \bar{R}_n & = \eta_{_r} \sum V_n - \eta_{_r} \sum V_0 + \bar{R}_0 = \eta_{_r} \sum V_n - c_r, \label{td_r_bar_incremental_1} \\\nonumber\\
    \text{ where } c_r & \doteq \eta_{_r} \sum V_0 - \bar{R}_0. \label{td_r_bar_incremental_2}
\end{align}

Similarly, consider the increment to \(Z_{i, n}\) (for an arbitrary subtask \(z_i \in \mathcal{Z}\)) at each step. As per Remark \ref{remark_1}, and given Assumption \ref{assumption_subtask_stepsize}, we can write the increment in Equation \eqref{async_td_z_update_eqn} as some constant, subtask-specific fraction of the increment to \(V_n\). Consequently, we can write the cumulative increment as follows:
\begin{align}
    Z_{i, n} - Z_{i, 0} &= \eta_{z_i} \sum_{j = 0}^{n-1} \sum_{s} \alpha_{\nu(j, s)} \rho_j(s) \beta_{i, j}(s) I\{s \in Y_j\} \nonumber\\\nonumber\\
    &= \eta_{z_i} \sum_{j = 0}^{n-1} \sum_{s} \alpha_{\nu(j, s)} \rho_j(s) (-1/b_i) \delta_j(s) I\{s \in Y_j\} \nonumber \\\nonumber\\
    & = \eta_{_i} \left (\sum V_{n} - \sum V_0 \right) \nonumber\\\nonumber\\
    \implies Z_{i, n} & = \eta_{_i} \sum V_n - \eta_{_i} \sum V_0 + Z_{i, 0} = \eta_{_i} \sum V_n - c_i, \label{td_z_incremental_1}
\end{align}
where,
\begin{align}
    c_i &\doteq \eta_{_i} \sum V_0 - Z_{i, 0}, \text{ and} \label{td_z_incremental_2}\\\nonumber\\
    \eta_{_i} &\doteq (-1/b_i)\eta_{z_i}. \label{td_defn_eta_n_subtask}
\end{align}

Now consider the subtask function, \(f\). At any given time step, the subtask function can be written as: \(f_n = \tilde{R}_n(s, A_n(s)) = b_rR_n(s,A_n(s)) + b_0 +b_1Z_{1, n} + \ldots +b_kZ_{k,n}\), where \(b_r, b_0 \in \mathbb{R}\) and \(b_i \in \mathbb{R}\setminus{\{0\}}\). Given Equation \eqref{td_z_incremental_1}, we can write the subtask function as follows: 
\begin{align}
    f_n &= b_rR_n(s,A_n(s)) + b_0 +b_1(\eta_{_1} \sum V_n - c_1) + \ldots +b_k(\eta_{_k} \sum V_n - c_k)\nonumber\\\nonumber\\
    \label{td_z_incremental_f}
    &= b_rR_n(s,A_n(s)) + \eta_{_f} \sum V_n - c_f,
\end{align}\\
where, \(\eta_{_f} = \sum_{j=1}^{k}b_j\eta_{_j}\) and \(c_f = \sum_{j=1}^{k}b_jc_j - b_0\).\\

As such, we can substitute \(\bar{R}_n\) and \(Z_{i,n} \; \forall z_i \in \mathcal{Z}\) in \eqref{async_td_value_update_eqn} with \eqref{td_r_bar_incremental_1} and \eqref{td_z_incremental_f}, respectively, \(\forall s \in \mathcal{S}\), which yields: 
\begin{align}
\begin{split}
    & V_{n+1}(s) = V_{n}(s) + \ldots\,\\
    & \quad \alpha_{\nu(n, s)} \rho_n(s) \left(b_rR_n(s, A_n(s)) + V_n(S_n'(s, A_n(s))) - V_n(s) - \eta_{_r} \sum V_n + c_{r} + \eta_{_f} \sum V_n - c_{_f}\right) I\{s \in Y_n\} \nonumber\\
    \\
    & V_{n+1}(s) = V_{n}(s) + \ldots\,\\
    & \quad \alpha_{\nu(n, s)} \rho_n(s) \left(b_rR_n(s, A_n(s)) + V_n(S_n'(s, A_n(s))) - V_n(s) - \eta_{_T} \sum V_n + c_{_T} \right) I\{s \in Y_n\} \nonumber\\
    \\
    & V_{n+1}(s) = V_{n}(s) + \ldots\,\\ 
    & \quad \alpha_{\nu(n, s)} \rho_n(s) \left(\widehat{R}_n(s, A_n(s)) + V_n(S_n'(s, A_n(s))) - V_n(s) - \eta_{_T} \sum V_n \right) I\{s \in Y_n\}  \label{td_r_bar_shifted_by_c},
\end{split}\\
\end{align}
where \(\eta_{_T} = \eta_{_r} - \eta_{_f}\), \(c_{_T} = c_r - c_f\), and \(\widehat{R}_n(s, A_n(s)) \doteq b_rR_n(s, A_n(s)) + c_{_T}\).

\clearpage

Equation \eqref{td_r_bar_shifted_by_c} is now in the same form as Equation (B.37) (i.e., General Differential TD) from \citet{Wan2021-re}, who showed that the equation converges a.s. \(V_n\) to \(v_\infty\) as \( n \to \infty\). Moreover, from this result, \citet{Wan2021-re} showed that \(\bar{R}_n\) converges a.s. to \(\bar{r}_{\pi}\) as \( n \to \infty\). Given that General RED TD adheres to all the assumptions listed for General Differential TD in \citet{Wan2021-re}, these convergence results apply to General RED TD.\\

\textbf{Convergence of the subtask estimates:}

Consider Equation \eqref{td_r_bar_shifted_by_c}. We can rewrite this equation, \(\forall s \in \mathcal{S}\), as follows:
\begin{align}
\begin{split}
    & V_{n+1}(s) = V_{n}(s) + \,\ldots\\
    & \quad \alpha_{\nu(n, s)}\rho_n(s) \left(b_rR_n(s, A_n(s)) + \eta_{_f} \sum V_n + c_{_T} - \eta_{_r} \sum V_n + V_n(S_n'(s, A_n(s))) - V_n(s)\right) I\{s \in Y_n\} \nonumber\\\\
    & V_{n+1}(s) = V_{n}(s) + \,\ldots\\
    & \quad \alpha_{\nu(n, s)}\rho_n(s) \left(\hat{R}_n(s, A_n(s)) + c_{_T} - \eta_{_r} \sum V_n + V_n(S_n'(s, A_n(s))) - V_n(s)\right) I\{s \in Y_n\}, \label{td_r_hat}
    \end{split}\\
\end{align}
where,
\begin{align}
\hat{R}_n(s, A_n(s)) &\doteq b_rR_n(s, A_n(s)) + \eta_{_f} \sum V_n \label{eqn_r_hat_td_1}\\
&= b_rR_n(s,A_n(s)) + b_1(\eta_{_1} \sum V_n) + \ldots +b_k(\eta_{_k} \sum V_n) \label{eqn_r_hat_td_2}\\
&\doteq b_rR_n(s,A_n(s)) + b_1\hat{Z}_{1, n} + \ldots +b_k\hat{Z}_{k, n}. \label{eqn_r_hat_td_3}
\end{align}

Now consider an MDP, \(\hat{\mathcal{M}}\), which has rewards, \(\mathcal{\hat{R}}\), as defined in Equation \eqref{eqn_r_hat_td_1}, has the same state and action spaces as the MDP \(\mathcal{\tilde{M}}\), and has the transition probabilities defined as:
\begin{align}
    \hat{p}(s', \hat{r} \mid s, a) &\doteq p(s', r \mid s, a)\\
    & = \tilde{p}(s', \tilde{r} \mid s, a) \quad \text{(by Definition \ref{definition_1})}, \label{td_z_probs_r_hat}
\end{align}
such that \(\hat{\mathcal{M}} \doteq \langle\mathcal{S}, \mathcal{A}, \mathcal{\hat{R}}, \hat{p}\rangle\). It is easy to check that the unichain assumption holds for the MDP, \(\hat{\mathcal{M}}\). Moreover, given Equation \eqref{td_r_hat} and Assumptions \ref{assumption_subtask_function} and \ref{assumption_subtask_independence}, the average-reward induced by following policy \(\pi\) for the MDP, \(\hat{\mathcal{M}}\), \(\hat{\bar{r}}_{\pi}\), can be written as follows: 
\begin{align}
    \hat{\bar{r}}_{\pi} = \bar{r}_{\pi} - c_{_T}. \label{td_z_star_modified_vs_z_star}
\end{align}

Now, because 
\begin{align}
    v_\infty(s) & = \sum_a \pi(a \mid s) \sum_{s', \tilde{r}} \tilde{p}(s', \tilde{r} \mid s, a) (\tilde{r} - \bar{r}_{\pi} + v_\infty (s')) \quad \text{(from \eqref{td_system_of_eqns})} \nonumber\\
    & = \sum_a \pi(a \mid s) \sum_{s', \tilde{r}} \tilde{p}(s', \tilde{r} \mid s, a) (\tilde{r} - (\hat{\bar{r}}_{\pi} + c_{_T}) + v_\infty (s')) \quad \text{(from \eqref{td_z_star_modified_vs_z_star})} \nonumber\\
    & = \sum_a \pi(a \mid s) \sum_{s', \tilde{r}} \tilde{p}(s', \tilde{r} \mid s, a) (\tilde{r} - c_{_T} - \hat{\bar{r}}_{\pi} + v_\infty (s'))\nonumber\\
    & = \sum_a \pi(a \mid s) \sum_{s', \tilde{r}} \tilde{p}(s', \tilde{r} \mid s, a) (\hat{r}  - \hat{\bar{r}}_{\pi} + v_\infty (s')) \quad \text{(from \eqref{td_r_hat})} \nonumber\\
    & = \sum_a \pi(a \mid s) \sum_{s', \hat{r}} \hat{p}(s', \hat{r} \mid s, a) (\hat{r} - \hat{\bar{r}}_{\pi} + v_\infty (s')) \quad \text{(from \eqref{td_z_probs_r_hat})} \nonumber, 
\end{align}
we can see that \(v_\infty\) is a solution of not just the state-value Bellman equation for the MDP, \(\mathcal{\tilde{M}}\), but also the state-value Bellman equation for the MDP, \(\hat{\mathcal{M}}\).

Next, consider an arbitrary \(i\)th subtask. As per Equations \eqref{eqn_r_hat_td_2} and \eqref{eqn_r_hat_td_3}, we can write the subtask value induced by following policy \(\pi\) for the MDP, \(\hat{\mathcal{M}}\), \(\hat{z_i}_{, \pi}\), as follows: 
\begin{align}
    \hat{z_i}_{, \pi} = z_{i, \pi} + c_i \label{td_z_star_shifted_vs_z_star}.
\end{align}

We can then combine Equations \eqref{td_z_sys_eqns}, \eqref{td_z_incremental_1}, and \eqref{td_z_star_shifted_vs_z_star}, which yields:
\begin{align}
    \hat{z_i}_{, \pi} = \eta_{_i} \sum v_\infty \label{td_z_star_shifted_vs_v_infty}.
\end{align}

Next, we can combine Equation \eqref{td_z_incremental_1} with the result from \citet{Wan2021-re} which shows that \(V_n \to v_\infty\), which yields: 
\begin{align}
    Z_{i, n} \to \eta_{_i} \sum v_\infty - c_i \label{td_z_convergence_final_1}.
\end{align}

Moreover, because \(\eta_{_i} \sum v_\infty = \hat{z_i}_{, \pi}\) (Equation \eqref{td_z_star_shifted_vs_v_infty}), we have:
\begin{align}
   Z_{i, n} \to \hat{z_i}_{, \pi} - c_i \label{td_z_convergence_final_2}.
\end{align}

Finally, because \(\hat{z_i}_{, \pi} = z_{i, \pi} + c_i\) (Equation \eqref{td_z_star_shifted_vs_z_star}), we have: 
\begin{align}
    Z_{i,n} \to z_{i, \pi} \text{\ \  a.s. as \ \  } n \to \infty \label{TD: convergence of z_n to z(pi)}.
\end{align}

\vspace{12pt}

\subsubsection{Proof of Theorem \ref{theorem_convergence_of_red_td_update} (for \emph{Piecewise Linear} Subtask Functions)}
\label{proof_v_piecewise}

We now provide the proof for \emph{piecewise linear} subtask functions, where the reward-extended TD error can be expressed as follows: 
\begin{equation}
\beta_{i,n}(s) = 
\begin{cases} 
(-1/b^{1}_i)\left(\tilde{R}_{1,n}(s, A_n(s)) - \bar{R}_n - \delta_n(s)\right), \,\ r_0 \leq R_{n}(s, A_n(s)) < r_1 \\
\vdots \\
(-1/b^{m}_i)\left(\tilde{R}_{m,n}(s, A_n(s)) - \bar{R}_n - \delta_n(s)\right), \,\ r_{m-1} \leq R_{n}(s, A_n(s)) \leq r_m \nonumber
\end{cases},
\end{equation}
where \(r_u \in \mathcal{R} \; \forall \, u = 0, 1, \ldots, m\), and \(r_0 \leq r_1 \leq \ldots \leq r_m\), such that \(r_0, r_m\) represent the lower and upper bounds of the observed per-step reward, \(R_{n}(s, A_n(s))\), respectively. Our general strategy in this case is to use a \emph{multiple-timescales} argument, such that we leverage Theorem 2 in Section 6 of \citet{Borkar2009-sr}, along with the results from Theorem B.3 of \citet{Wan2021-re}.

To begin, let us consider Assumption \ref{assumption_subtask_stepsize}, which enables the formulation of a multiple-timescales argument. In particular, the \(\alpha_{z_1, n}/ \alpha_n \to 0\), \(\alpha_{z_1, n}/ \alpha_{\bar{r},n} \to 0\), and \(\{\alpha_{z_i, n} / \alpha_{z_{i - 1}, n} \to 0\}_{i=2}^{k}\) conditions imply that the subtask step sizes, \(\{\alpha_{z_i, n}\}_{n=0}^\infty \, \forall z_i \in \mathcal{Z}\), decrease to 0 at faster rates than the value function and average-reward step sizes, \(\{\alpha_n\}_{n=0}^\infty\) and \(\{\alpha_{\bar{r},n}\}_{n=0}^\infty\), respectively. This implies that the subtask updates move on slower timescales compared to the value function and average-reward updates. Hence, as argued in Section 6 of \citet{Borkar2009-sr}, the (faster) value function and average-reward updates, \eqref{async_td_value_update_eqn} and \eqref{async_td_r_bar_update_eqn}, view the (slower) subtask updates, \eqref{async_td_z_update_eqn}, as quasi-static, while the (slower) subtask updates view the (faster) value function and average-reward updates as nearly equilibrated (as we will show below, the results from \citet{Wan2021-re} imply the existence of such an equilibrium point). Similarly, the \(\{\alpha_{z_i, n} / \alpha_{z_{i - 1}, n} \to 0\}_{i=2}^{k}\) condition implies that each subtask update views the other subtask updates as either quasi-static or nearly-equilibrated.

As such, having established the multiple-timescales argument, we now proceed to show the convergence of the value function, average-reward, and subtask estimates:

\newpage

\textbf{Convergence of the value function and average-reward estimates:}

Given the multiple-timescales argument, such that the subtask estimates are viewed as quasi-static (i.e., constant), Equation \eqref{async_td_value_update_eqn} can be viewed as being of the same form as Equation (B.30) (i.e., General Differential TD) from \citet{Wan2021-re}, who showed (via Theorem B.3) that the equation converges, almost surely, \(V_n\) to \(v_\infty\) as \( n \to \infty\). Moreover, from this result, \citet{Wan2021-re} showed that \(\bar{R}_n\) converges, almost surely, to \(\bar{r}_{\pi}\) as \( n \to \infty\). Given that General RED TD adheres to all the assumptions listed for General Differential TD in \citet{Wan2021-re}, these convergence results apply to General RED TD.\\

\textbf{Convergence of the subtask estimates:}

Let us consider the asynchronous subtask updates \eqref{async_td_z_update_eqn}. Each update in \eqref{async_td_z_update_eqn} is of the same form as Equation 7.1.2 of \citet{Borkar2009-sr}. Accordingly, to show the convergence of the subtask estimates, we can apply the result in Section 7.4 of \citet{Borkar2009-sr}, which shows the convergence of asynchronous updates that are of the same form as Equation 7.1.2. To apply this result, given Assumptions \ref{assumption_async_step_size_1} and \ref{assumption_async_step_size_td_2}, we only need to show the convergence of the \emph{synchronous} version of the subtask updates:
\begin{equation}
\label{sync_z_update_td}
Z_{i, n+1} = Z_{i, n} + \alpha_{z_i, n} \left(g_i(Z_{i, n}) + M^{z_i}_{n+1}\right) \; \forall z_i \in \mathcal{Z},
\end{equation}
where, 
\begin{align*}
& g_i(Z_{i, n})(s) \doteq \sum_{a}\pi(a\mid s)\sum_{s', r} p(s', r \mid s, a)\phi_{i,n}(s) - Z_{i, n},\\
& \phi_{i,n}(s) \doteq
\begin{cases}
-\frac{1}{b^{1}_i}\left(b^{1}_rR_n(s, A_n(s)) + \ldots + b^{1}_{i-1}Z_{i-1,n} + b^{1}_{i+1}Z_{i+1,n} + \ldots + b^{1}_kZ_{k,n} - \bar{R}_n - \delta_n(s)\right), \ldots\\
\hspace{1.5cm} \ldots, r_0 \leq R_n(s, A_n(s)) < r_1, \\
\vdots \\
-\frac{1}{b^{m}_i}\left(b^{m}_rR_n(s, A_n(s)) + \ldots + b^{m}_{i-1}Z_{i-1,n} + b^{m}_{i+1}Z_{i+1,n} + \ldots + b^{m}_kZ_{k,n} - \bar{R}_n - \delta_n(s)\right), \ldots\\
\hspace{1.5cm} \ldots, r_{m-1} \leq R_n(s, A_n(s)) \leq r_m,
\end{cases},\\
& M^{z_i}_{n+1}(s) \doteq \rho_n(s) \left(\phi_{i,n}(s) - Z_{i,n}\right) - g_i(Z_{i, n})(s).
\end{align*}

To show the convergence of the synchronous update \eqref{sync_z_update_td} under the multiple-timescales argument, we can apply the result of Theorem 2 in Section 6 of \citet{Borkar2009-sr} to show that \(Z_{i,n} \to z_{i, \pi} \forall z_i \in \mathcal{Z}\) a.s. as \(n \to \infty\). This theorem requires that 3 assumptions be satisfied. As such, we will now show, via Lemmas \ref{lemma_td_1} - \ref{lemma_td_3}, that these 3 assumptions are indeed satisfied:\\

\begin{lemma}
\label{lemma_td_1}
The value function update, \( V_{n+1} = V_{n} + \alpha_n (h(V_n) + M_{n+1})\), where
\begin{align*}
h(V_n)(s) & \doteq \sum_{a}\pi(a\mid s)\sum_{s', \tilde{r}} \tilde{p}(s', \tilde{r} \mid s, a) (\tilde{r} - \bar{R}_n + V_n(s') - V_n(s)), \\\nonumber
& = \sum_{a}\pi(a\mid s)\sum_{s', r} p(s', r \mid s, a) (f(r, Z_{1,n}, Z_{2,n}, \ldots, Z_{k,n}) - \psi(V_n) + V_n(s') - V_n(s)), \\\nonumber
M_{n + 1}(s) & \doteq \rho_n(s) \left(\tilde{R}_n(s, A_n(s)) - \psi(V_n) + V_n(S_n'(s, A_n(s))) - V_n(s) \right) - h(V_n)(s),\\\nonumber
\psi(V_n) & = \bar{R}_n \; \text{is a ‘reference function’ as defined in \citet{Wan2021-re}}, \text{ and}\\\nonumber
Z_{1,n}, Z_{2,n}&, ..., Z_{k,n} \; \text{are quasi-static under the multiple-timescales argument},
\end{align*}
has a globally asymptotically stable equilibrium, \(v_\infty(Z_{1,n}, Z_{2,n}, \ldots, Z_{k,n})\), where \(v_\infty\) is a Lipschitz map.
\end{lemma}

\begin{proof}
This was shown in Theorem B.3 of \citet{Wan2021-re}.
\end{proof}

\newpage

\begin{lemma}
\label{lemma_td_2}
The subtask update rules \eqref{sync_z_update_td} each have a globally asymptotically stable equilibrium, \(z_{i, \pi}\).
\end{lemma}
\begin{proof}
Let us begin by considering the ‘fastest’ subtask, \(Z_{1,n}\), under the multiple-timescales argument, such that all other subtasks, \(Z_{2,n}, \ldots, Z_{k,n}\), operate on slower timescales and can be considered quasi-static. Applying the results of Theorem B.3 of \citet{Wan2021-re} under the multiple-timescales argument, we have that: \(h(V_n) \to 0\) and \(\bar{R}_n \to \bar{r}_{\pi}\) as \(V_n \to v_\infty\). Importantly, as argued in Section \ref{red}, \(h(V_n) \to 0\) implies that \(\delta_n(s) \to \lambda_j \in \mathbb{R}\) for each \(j\)th piecewise segment in \(\phi_{i,n}(s)\). As such, we can interpret \(Z_{1,n}\) as being the only ‘moving’ variable in \(g_1(Z_{1,n})\), such that all other parameters in the update are static, quasi-static, or nearly-equilibrated. 

Let us now consider the ODE associated with \(g_1(Z_{1,n})\), 
\begin{equation}
\label{eqn_td_z1_ode}
\dot{x}_t = g_1(x_t).
\end{equation} 
To show that the update rule for \(Z_{1,n}\) has a globally asymptotically stable equilibrium, \(z_{1, \pi}\), it suffices to show that there exists a \emph{Lyapunov function} for the associated ODE \eqref{eqn_td_z1_ode}. 

To this end, we first note, given the discussion in Section \ref{red}, that \(z_{1, \pi}\) is an equilibrium point for the update rule associated with \(Z_{1,n}\). Moreover, given Assumptions \ref{assumption_unichain} and \ref{assumption_subtask_unique}, we know that \(z_{1, \pi}\) is the unique equilibrium point.

We now show the existence of a Lyapunov function with respect to the aforementioned equilibrium point, \(z_{1, \pi}\). In particular, we consider the function, \(L\), defined by:
\[
L(Z_{1}) = \frac{1}{2}(Z_{1} - z_{1, \pi})^2.
\]

To establish that \(L\) is a Lyapunov function, we must show that:
\begin{enumerate}
\item \(L\) is continuous,
\item \(L(Z_{1}) = 0\) if \(Z_{1} = z_{1, \pi}\),
\item \(L(Z_1) > 0\) if \(Z_{1} \neq z_{1, \pi}\), and
\item For any solution \(\{x_t\}_{t \geq 0}\) of the associated ODE \eqref{eqn_td_z1_ode} and \(0 \leq s < t\), we have \(L(x_t) < L(x_s)\) for all \(x_s \neq z_{1, \pi}\).\\
\end{enumerate}

It directly follows from the definition of \(L\) that the first three conditions are satisfied.\\ 

We now show that fourth condition is also satisfied:

Let \(\{x_t\}_{t \geq 0}\) be a solution to the associated ODE \eqref{eqn_td_z1_ode}, and let \(0 \leq s < t\). By the chain rule, we have that:
\begin{equation}
\frac{d}{dt} L(x_t) = \frac{\partial L}{\partial x_t} \cdot \frac{dx_t}{dt} = (x_t - z_{1, \pi}) \cdot g_1(x_t) = (x_t - z_{1, \pi}) \cdot (\mathbb{E}_\pi[\phi_1] - x_t)= (x_t - z_{1, \pi}) \cdot (z_{1, \pi} - x_t).
\end{equation}

We will now analyze the sign of the term \((x_t - z_{1, \pi}) \cdot (z_{1, \pi} - x_t)\) when \(x_t \neq z_{1, \pi}\). There are two cases to consider:

\textbf{Case 1: \(x_t > z_{1, \pi}\):} We have that \((x_t - z_{1, \pi}) > 0\) and \((z_{1, \pi} - x_t) < 0\). Hence, we can conclude that \((x_t - z_{1, \pi}) \cdot (z_{1, \pi} - x_t) < 0\).

\textbf{Case 2: \(x_t < z_{1, \pi}\):} We have that \((x_t - z_{1, \pi}) < 0\) and \((z_{1, \pi} - x_t) > 0\). Hence, we can conclude that \((x_t - z_{1, \pi}) \cdot (z_{1, \pi} - x_t) < 0\).

As such, in both cases we can conclude that \((x_t - z_{1, \pi}) \cdot (z_{1, \pi} - x_t) < 0\), which implies that \(L(x_t) < L(x_s)\) for all \(x_s \neq z_{1, \pi}\) and \(0 \leq s < t\).\\

As such, we have now verified the four conditions, and can therefore conclude that \(L\) is a valid Lyapunov function. Consequently, we can conclude, under the multiple-timescales argument, that \(Z_{1,n}\) has a globally asymptotically stable equilibrium, \(z_{1, \pi}\).

Critically, we can leverage the above result to show, using the same techniques as above, that \(Z_{2,n}\) has a globally asymptotically stable equilibrium, \(z_{2, \pi}\), where \(V_n, \bar{R}_n, \text{ and } Z_{1,n}\) are considered to be nearly equilibrated, and all remaining subtasks are considered quasi-static. The process can then be repeated for \(Z_{3,n}\) and so forth, thereby showing that the subtask update rules \eqref{sync_z_update_td} each have a globally asymptotically stable equilibrium, \(z_{i, \pi}\). This completes the proof.
\end{proof}

\vspace{10pt}

\begin{lemma}
\label{lemma_td_3}
\(\sup_n(\vert\vert V_n \vert\vert + \vert\vert Z_{1, n} \vert\vert) + \vert\vert Z_{2,n} \vert\vert + \ldots + \vert\vert Z_{k,n} \vert\vert) < \infty\) a.s.
\end{lemma}
\begin{proof}
It was shown in Theorem B.3 of \citet{Wan2021-re} that \(\sup_n(\vert\vert V_n \vert\vert) < \infty\) a.s. Hence, we only need to show that \(\sup_n(\vert\vert Z_{i,n} \vert\vert) < \infty \; \forall z_i \in \mathcal{Z}\) a.s. To this end, we can apply Theorem 7 in Section 3 of \citet{Borkar2009-sr}. This theorem requires 4 assumptions for each \(z_i \in \mathcal{Z}\):
\begin{itemize}
    \item \textbf{(A1)} The function \(g_i\) is Lipschitz. That is, \(\vert \vert g_i(x) - g_i(y)\vert \vert \leq U_i \vert \vert x - y\vert \vert\) for some \(0 < U_i < \infty\).
    \item \textbf{(A2)} The sequence \(\{ \alpha_{z_i, n}\}_{n=0}^\infty\) satisfies \(\alpha_{z_i, n} > 0 \; \forall n \geq 0\), \(\sum \alpha_{z_i, n} = \infty\), and \(\sum \alpha_{z_i,n}^2 < \infty\).
    \item \textbf{(A3)} \(\{M_n^{z_i}\}_{n=0}^\infty\) is a martingale difference sequence that is square-integrable.
    \item \textbf{(A4)} The functions \(g_i(x)_d \doteq g_i(dx)/d\), \(d \geq 1, x \in \mathbb{R}\), satisfy \(g_i(x)_d \to g_i(x)_\infty\) as \(d \to \infty\), uniformly on compacts for some \(g_{i_\infty} \in C(\mathbb{R})\). Furthermore, the ODE \(\dot x_t = g_i(x_t)_\infty\) has the origin as its unique globally asymptotically stable equilibrium.
\end{itemize}

Consider an arbitrary \(i\)th subtask. We note that Assumption \textbf{(A1)} is satisfied given that all operators in \(g_i\) are Lipschitz. Moreover, we note that Assumption \ref{assumption_step_size} satisfies Assumption \textbf{(A2)}. 

We now show that \(\{M_n^{z_i}\}_{n=0}^\infty\) is indeed a martingale difference sequence that is square-integrable, such that \(\mathbb{E}_B[M^{z_i}_{n+1} \mid \mathcal{F}_n] = 0 \text{ \ a.s., } n \geq 0\), and 
\(\mathbb{E}_B[(M_{n+1}^{z_i})^2 \mid \mathcal{F}_n] < \infty \text{ \ a.s., \ } n \geq 0\), where \(B\) is the behaviour policy (and \(\pi\) is the target policy). To this end, let \(\mathcal{F}_n \doteq \sigma(Z_{i,u}, M^{z_i}_u, u \leq n), n \geq 0\) denote an increasing family of \(\sigma\)-fields. We have that, for any \(s \in \mathcal{S}\): 
\begin{align}
\begin{split}
    \nonumber
    \mathbb{E}_B[M^{z_i}_{n+1}(s) \mid \mathcal{F}_n]
    & = \mathbb{E}_B \left[\rho_n(s) \left(\phi_{i,n}(s) - Z_{i,n}\right) - g_i(Z_{i, n})(s) \mid \mathcal{F}_n \right]\\ \nonumber
    & = \mathbb{E}_\pi \left[\phi_{i,n}(s) - Z_{i,n}\right] - g_i(Z_{i, n})(s)\\
     &= 0.
\end{split}
\end{align}
Moreover, since all components of \(\{M_n^{z_i}\}_{n=0}^\infty\) involve bounded quantities, it directly follows that \(\mathbb{E}_B[\vert \vert M_{n+1}^{z_i} \vert \vert^2 \mid \mathcal{F}_n] \leq K\) for some finite constant \(K > 0\). Hence, Assumption \textbf{(A3)} is verified.

We now verify Assumption \textbf{(A4)}. Under the multiple-timescales argument, we have that:
\begin{align}
    g_i(x)_\infty = \lim_{d \to \infty} g_i(x)_d = \lim_{d \to \infty} \frac{\mathbb{E}_\pi \left[\phi_{i}\right] - dx}{d} = 0 - x = -x.
\end{align}
Clearly, \(g_i(x)_\infty\) is continuous in every \(x \in \mathbb{R}\). As such, we have that \(g_i(x)_\infty \in C(\mathbb{R})\). Now consider the ODE \(\dot{x} = g_i(x)_\infty\). This ODE has the origin as an equilibrium since \(g_i(0)_\infty = 0\). Furthermore, given Lemma \ref{lemma_td_2}, we can conclude that this equilibrium must be the unique globally asymptotically stable equilibrium, thereby satisfying Assumption \textbf{(A4)}.

Assumptions \textbf{(A1)} - \textbf{(A4)} are hence verified, meaning that we can apply the results of Theorem 7 in Section 3 of \citet{Borkar2009-sr} to conclude that \(\sup_n(\vert\vert Z_{i,n} \vert\vert) < \infty \; \forall z_i \in \mathcal{Z}\), almost surely, and hence, that \(\sup_n(\vert\vert V_n \vert\vert + \vert\vert Z_{1, n} \vert\vert) + \vert\vert Z_{2,n} \vert\vert + \ldots + \vert\vert Z_{k,n} \vert\vert) < \infty\), almost surely.
\end{proof}

As such, we have now verified the 3 assumptions required by Theorem 2 in Section 6 of \citet{Borkar2009-sr}, which means that we can apply the result of the theorem to conclude that \(Z_{i,n} \to z_{i, \pi} \, \forall z_i \in \mathcal{Z}\), almost surely, as \(n \to \infty\).\\

This completes the proof of Theorem \ref{theorem_convergence_of_red_td_update}.

\newpage

\subsection{Convergence Proof for the Tabular RED Q-learning Algorithm}
\label{proof_red_q}

In this section, we present the proof for the convergence of the value function, average-reward, and subtask estimates of the RED Q-learning algorithm. Similar to what was done in \citet{Wan2021-re}, we will begin by considering a general algorithm, called \emph{General RED Q}. We will first define General RED Q, then show how the RED Q-learning algorithm is a special case of this algorithm. We will then provide the necessary assumptions, state the convergence theorem of General RED Q, and then provide a proof for the theorem, where we show that the value function, average-reward, and subtask estimates converge, thereby showing that the RED Q-learning algorithm converges. We begin by introducing the General RED Q algorithm:

Consider an MDP \(\mathcal{M} \doteq \langle\mathcal{S}, \mathcal{A}, \mathcal{R}, p \rangle\). Given a state \(s \in \mathcal{S}\), action \(a \in \mathcal{A}\), and discrete step \(n \geq 0\), let \(R_n(s, a) \in \mathcal{R}\) denote a sample of the resulting reward, and let \(S'_n(s, a) \sim p(\cdot, \cdot \mid s, a)\) denote a sample of the resulting state. Let \(\{Y_n\}\) be a set-valued process taking values in the set of nonempty subsets of \(\mathcal{S} \times \mathcal{A}\), such that: \(Y_n = \{(s, a): (s, a)\) component of the \(\vert \mathcal{S} \times \mathcal{A} \vert\)-sized table of state-action value estimates, \(Q\), that was updated at step \(n\}\). Let \(\nu(n, s, a) \doteq \sum_{j=0}^n I\{(s, a) \in Y_j\}\), where \(I\) is the indicator function, such that \(\nu(n, s, a)\) represents the number of times that the \((s, a)\) component of \(Q\) was updated up until step \(n\). 

Now, let \(f\) be a valid subtask function (see Definition \ref{definition_1}), such that \(\tilde{R}_n(s, a) \doteq f(R_n(s, a), Z_{1,n}, Z_{2,n}, \ldots, Z_{k,n})\) for \(k\) subtasks \(\in \mathcal{Z}\), where \(\tilde{R}_n(s, a)\) is the extended reward, \(\mathcal{Z}\) is the set of subtasks, and \(Z_{i, n}\) denotes the estimate of subtask \(z_i \in \mathcal{Z}\) at step \(n\). Consider an MDP with the extended reward: \(\mathcal{\tilde{M}} \doteq \langle\mathcal{S}, \mathcal{A}, \mathcal{\tilde{R}}, \tilde{p} \rangle\), such that \(\tilde{R}_n(s, a) \in \mathcal{\tilde{R}}\). The update rules of General RED Q for this MDP are as follows, \(\forall n \geq 0\):
\begin{align}
    Q_{n+1}(s, a) & \doteq Q_n(s, a) + \alpha_{\nu(n, s, a)} \delta_n(s, a) I\{(s, a) \in Y_n\}, \quad \forall s \in \mathcal{S}, a \in \mathcal{A}, \label{async_q_value_update_eqn}\\
    \bar{R}_{n+1} & \doteq \bar{R}_n + \sum_{s, a} \alpha_{\bar{r}, \nu(n, s, a)} \delta_n(s, a) I\{(s, a) \in Y_n\}, \label{async_q_r_bar_update_eqn}\\
    Z_{i, n+1} & \doteq Z_{i, n} + \sum_{s, a} \alpha_{z_i, \nu(n, s, a)} \beta_{i, n}(s, a) I\{(s, a) \in Y_n\}, \quad \forall z_i \in \mathcal{Z} \label{async_q_z_update_eqn}
\end{align}
where,
\begin{align}
\begin{split}
    \delta_n(s, a) & \doteq \tilde{R}_n(s, a) - \bar{R}_n + \max_{a'} Q_n(S_n'(s, a), a') - Q_n(s, a)\\
    & = f(R_n(s, a), Z_{1, n}, Z_{2, n}, \ldots, Z_{k, n}) - \bar{R}_n + \max_{a'} Q_n(S_n'(s, a), a') - Q_n(s, a),
\end{split}\label{async_q_td_error_eqn}
\end{align}
and,
\begin{align}
    \beta_{i, n}(s, a) & \doteq \phi_{i, n}(s, a) -  Z_{i, n}, \quad \forall z_i \in \mathcal{Z}. \label{async_q_red_td_error_eqn}
\end{align}

Here, \(\bar{R}_n\) denotes the estimate of the average-reward (see Equation 
\eqref{eq_avg_reward_2}), \(\delta_n(s, a)\) denotes the TD error, \(\phi_{i, n}(s, a)\) denotes the (potentially-piecewise) subtask target, as defined in Section \ref{reward_extended_td}, and \(\alpha_{\nu(n, s, a)}\), \(\alpha_{\bar{r}, \nu(n, s, a)}\), and \(\alpha_{z_i, \nu(n, s, a)}\) denote the step sizes at time step \(n\) for state-action pair \((s, a)\). 

We now show that the RED Q-learning algorithm is a special case of the General RED Q algorithm. Consider a sequence of experience from our MDP, \(\mathcal{\tilde{M}}\): \(S_t, A_t, \tilde{R}_{t+1}, S_{t+1}, \ldots\)\, . Now recall the set-valued process \(\{Y_n\}\). If we let \(n\) = time step \(t\), we have: 
\begin{align*}
Y_t(s, a) = 
\begin{cases}
    1, s = S_t\text{ and }a = A_t,\\
    0, \text{ otherwise,}
\end{cases}
\end{align*}
as well as \(S'_n(S_t, A_t) = S_{t+1}\), \(R_n(S_t, A_t) = R_{t+1}\), and \(\tilde{R}_n(S_t, A_t) = \tilde{R}_{t+1}\).\\

Hence, update rules \eqref{async_q_value_update_eqn}, \eqref{async_q_r_bar_update_eqn}, \eqref{async_q_z_update_eqn}, \eqref{async_q_td_error_eqn}, and \eqref{async_q_red_td_error_eqn} become:
\begin{align}
    Q_{t+1}(S_t, A_t) & \doteq Q_t (S_t, A_t) + \alpha_{\nu(t, S_t, A_t)} \delta_t \text{; \; } Q_{t+1}(s, a) \doteq Q_t (s, a), \forall s \neq S_t, a \neq A_t, \\
    \bar{R}_{t+1} & \doteq \bar{R}_t + \alpha_{\bar{r}, \nu(t, S_t, A_t)} \delta_t,\\
    Z_{i,t+1} & \doteq Z_{i, t} + \alpha_{z_i, \nu(t, S_t, A_t)} \beta_{i, t}, \quad \forall z_i \in \mathcal{Z},\\
    \begin{split}
    \delta_t & \doteq \tilde{R}_{t+1} - \bar{R}_t + \max_{a'} Q_t (S_{t+1}, a') - Q_t (S_t, A_t),\\
    & = f(R_{t+1}, Z_{1, t}, Z_{2, t}, \ldots, Z_{k, t}) - \bar{R}_t + \max_{a'} Q_t (S_{t+1}, a') - Q_t (S_t, A_t),
    \end{split}\\
    \beta_{i, t} & \doteq \phi_{i, t} - Z_{i, t}, \quad \forall z_i \in \mathcal{Z},
\end{align}

which are RED Q-learning's update rules with \(\alpha_{\nu(t, S_t, A_t)}\), \(\alpha_{\bar{r}, \nu(t, S_t, A_t)}\), and \(\alpha_{z_i, \nu(t, S_t, A_t)}\) denoting the step sizes at time \(t\).\\

We now specify the assumptions specific to General RED Q that are needed to ensure convergence. We refer the reader to \citet{Wan2021-re} for an in-depth discussion on these assumptions:\\

\begin{assumption}[Communicating Assumption] \label{assumption_communicating}
The MDP has a single communicating class. That is, each state in the MDP is accessible from every other state under some deterministic stationary policy.\\
\end{assumption}

\begin{assumption}[State-Action Value Function Uniqueness] \label{assumption_action_value_function_uniqueness}
There exists a unique solution of \(q\) only up to a constant in the Bellman equation \eqref{eq_avg_reward_5}.\\
\end{assumption}

\begin{assumption}[Asynchronous Step Size Assumption 3] \label{assumption_async_step_size_q_2}
There exists \(\Delta > 0\) such that 
\begin{align*}
    \liminf_{n \to \infty} \frac{\nu(n, s, a)}{n+1} \geq \Delta,
\end{align*}
a.s., for all \(s \in \mathcal{S}, a \in \mathcal{A}\).\\

Furthermore, for all \(x > 0\), and
\begin{align*}
    N(n, x) = \min \Bigg \{m > n: \sum_{i = n+1}^m \alpha_i \geq x \Bigg \},
\end{align*}
the limit 
\begin{align*}
    \lim_{n \to \infty} \frac{\sum_{i = \nu(n, s, a)}^{\nu(N(n, x), s, a)} \alpha_i}{\sum_{i = \nu(n, s', a')}^{\nu(N(n, x), s', a')} \alpha_i}
\end{align*}

exists a.s. for all \(s, s', a, a'\).\\
\end{assumption}

Having stated the necessary assumptions, we next point out that it is easy to verify that under Assumption~\ref{assumption_communicating}, the following system of equations:
\begin{align}
    \begin{split}
    q_*(s, a) & = \sum_{s', \tilde{r}} \tilde{p}(s', \tilde{r} \mid s, a) (\tilde{r} - \bar{r}_* + \max_{a'} q_*(s, a)), \quad \forall s \in \mathcal{S}, a \in \mathcal{A}\\
    & = \sum_{s', r} p(s', r \mid s, a) (f(r, z_{1_*}, z_{2_*}, \ldots, z_{k_*}) - \bar{r}_* + \max_{a'} q_*(s, a)), \quad \forall s \in \mathcal{S}, a \in \mathcal{A},\\
    \end{split}\label{q_system_of_eqns}
\end{align}
and,
\begin{align}
    \bar{r}_* - \bar{R}_0 & = \eta_{_r} \left(\sum q_* - \sum Q_0 \right),  \label{q_r_bar_sys_eqns}\\
    z_{i_*} - Z_{i,0} & = \eta_{_i} \left(\sum q_* - \sum Q_0\right), \quad \forall z_i \in \mathcal{Z}, \label{q_z_sys_eqns}
\end{align}
has a unique solution for \(q_*\), where \(\bar{r}_*\) denotes the optimal average-reward, and \(z_{i_*}\) denotes the corresponding optimal subtask value for subtask \(z_i \in \mathcal{Z}\). Denote this unique solution solution of \(q_*\) as \(q_\infty\).\\

We are now ready to state the convergence theorem:\\

\begin{theorem}[Convergence of General RED Q]\label{theorem_convergence_of_red_q_update}
If Assumptions~\ref{assumption_step_size}, \ref{assumption_async_step_size_1}, 
\ref{assumption_r_step_size},
\ref{assumption_subtask_function}, \ref{assumption_subtask_independence},
\ref{assumption_subtask_unique},
\ref{assumption_subtask_stepsize},
\ref{assumption_communicating}, \ref{assumption_action_value_function_uniqueness}, and \ref{assumption_async_step_size_q_2} hold, then the General RED Q algorithm (Equations~\ref{async_q_value_update_eqn}--\ref{async_q_red_td_error_eqn}) converges a.s. \(\bar{R}_n\) to \(\bar{r}_*\), \(Z_{i, n}\) to \(z_{i_*} \; \forall z_i \in \mathcal{Z}\), \(\bar{r}_{\pi_t}\) to \(\bar{r}_*\), \(z_{i, \pi_t}\) to \(z_{i_*}  \; \forall z_i \in \mathcal{Z}\), and \(Q_n\) to \(q_\infty\), where \(\pi_t\) is any greedy policy with respect to \(Q_t\), and \(z_{i, \pi_t}\) denotes the subtask value induced by following policy \(\pi_t\).
\end{theorem}

We prove this theorem in Sections \ref{proof_q_linear} and \ref{proof_q_piecewise}. To do so, we first show that General RED Q is of the same form as \emph{General Differential Q} from \citet{Wan2021-re}, thereby allowing us to apply their convergence results for the value function and average-reward estimates of General Differential Q to General RED Q. We then build upon these results, using similar techniques as \citet{Wan2021-re}, to show that the subtask estimates converge as well.

\subsubsection{Proof of Theorem \ref{theorem_convergence_of_red_q_update} (for \emph{Linear} Subtask Functions)}
\label{proof_q_linear}

We first provide the proof for \emph{linear} subtask functions, where the reward-extended TD error can be expressed as a constant, subtask-specific fraction of the regular TD error, such that \(\beta_{i,n}(s, a) = (-1/b_i)\delta_n(s, a)\). We consider the \emph{piecewise linear} case in Section \ref{proof_q_piecewise}.\\

\textbf{Convergence of the value function and average-reward estimates:}

Consider the increment to \(\bar{R}_n\) at each step. Given Assumption \ref{assumption_r_step_size}, we can see from Equation \eqref{async_q_r_bar_update_eqn} that the increment is \(\eta_{_r}\) times the increment to \(Q_n\). As such, as was done in \citet{Wan2021-re}, we can write the cumulative increment as follows:
\begin{align}
    \bar{R}_n - \bar{R}_0 &= \eta_{_r}  \sum_{j = 0}^{n-1} \sum_{s, a} \alpha_{\nu(j, s, a)} \delta_j(s, a) I\{(s, a) \in Y_j\} \nonumber  \\\nonumber \\
    & = \eta_{_r} \left(\sum Q_n - \sum Q_0 \right) \nonumber \\\nonumber \\
    \implies \bar{R}_n &= \eta_{_r} \sum Q_n - \eta_{_r} \sum Q_0 + \bar{R}_0 = \eta_{_r} \sum Q_n - c_r \label{q_r_bar_incremental_1},\\\nonumber \\
    \text{ where } c_r &\doteq \eta_{_r} \sum Q_0 - \bar{R}_0. \label{q_r_bar_incremental_2}
\end{align}

Similarly, consider the increment to \(Z_{i, n}\) (for an arbitrary subtask \(z_i \in \mathcal{Z}\)) at each step. As per Remark \ref{remark_1}, and given Assumption \ref{assumption_r_step_size}, we can write the increment in Equation \eqref{async_q_z_update_eqn} as some constant, subtask-specific fraction of the increment to \(Q_n\). Consequently, we can write the cumulative increment as follows:

\begin{align}
    Z_{i, n} - Z_{i, 0} &= \eta_{z_i}  \sum_{j = 0}^{n-1} \sum_{s, a} \alpha_{\nu(j, s, a)} \beta_{i, j}(s, a) I\{(s, a) \in Y_j\} \nonumber  \\\nonumber \\
    &= \eta_{z_i}  \sum_{j = 0}^{n-1} \sum_{s, a} \alpha_{\nu(j, s, a)} (-1/b_i) \delta_j(s, a)I\{(s, a) \in Y_j\} \nonumber \\\nonumber \\
    & = \eta_{_i} \left(\sum Q_n - \sum Q_0 \right) \nonumber \\\nonumber \\
    \implies Z_{i, n} &= \eta_{_i} \sum Q_n - \eta_{_i} \sum Q_0 + Z_{i, 0} = \eta_{_i} \sum Q_n - c_i \label{q_z_incremental_1},
\end{align}
where,
\begin{align}
    c_i &\doteq \eta_{_i} \sum Q_0 - Z_{i, 0}, \text{ and} \label{q_z_incremental_2}\\\nonumber\\
    \eta_{_i} &\doteq (-1/b_i)\eta_{z_i}. \label{q_defn_eta_n_subtask}
\end{align}

Now consider the subtask function, \(f\). At any given time step, the subtask function can be written as: \(f_n = \tilde{R}_n(s, a) = b_rR_n(s,a) + b_0 +b_1Z_{1, n} + \ldots +b_kZ_{k,n}\), where \(b_r, b_0 \in \mathbb{R}\) and \(b_i \in \mathbb{R}\setminus{\{0\}}\). Given Equation \eqref{q_z_incremental_1}, we can write the subtask function as follows: 
\begin{align}
    f_n &= b_rR_n(s,a) + b_0 +b_1(\eta_{_1} \sum Q_n - c_1) + \ldots +b_k(\eta_{_k} \sum Q_n - c_k)\nonumber\\
    \label{q_z_incremental_f}
    &= b_rR_n(s,a) + \eta_{_f} \sum Q_n - c_f,
\end{align}
where, \(\eta_{_f} = \sum_{j=1}^{k}b_j\eta_{_j}\) and \(c_f = \sum_{j=1}^{k}b_jc_j - b_0\).\\

As such, we can substitute \(\bar{R}_n\) and \(Z_{i,n} \; \forall z_i \in \mathcal{Z}\) in \eqref{async_q_value_update_eqn} with \eqref{q_r_bar_incremental_1} and \eqref{q_z_incremental_f}, respectively, \(\forall s \in \mathcal{S}, a \in \mathcal{A}\), which yields: 
\begin{align}
    \begin{split}
    & Q_{n+1}(s, a) = Q_{n}(s, a) + \,\ldots\\
    & \quad \alpha_{\nu(n, s, a)} \left(b_rR_n(s,a) + \max_{a'} Q_n(S_n'(s, a), a') - Q_n(s, a) - \eta_{_r} \sum Q_n + c_r + \eta_{_f} \sum Q_n - c_f\right) I\{(s, a) \in Y_n\} \nonumber\\
    \\
    & Q_{n+1}(s, a) = Q_{n}(s, a) + \,\ldots\\
    & \quad \alpha_{\nu(n, s, a)} \left(b_rR_n(s,a) + \max_{a'} Q_n(S_n'(s, a), a') - Q_n(s, a) - \eta_{_T} \sum Q_n + c_{_T} \right) I\{(s, a) \in Y_n\} \nonumber\\
    \\
    & Q_{n+1}(s, a) = Q_{n}(s, a) + \,\ldots\\
    & \quad \alpha_{\nu(n, s, a)} \left(\widehat{R}_n(s, a) + \max_{a'} Q_n(S'_n(s, a), a') - Q_n(s, a) - \eta_{_T} \sum Q_n \right) I\{(s, a) \in Y_n\}, \label{q_r_bar_shifted_by_c}
    \end{split}\\
\end{align}
where \(\eta_{_T} = \eta_{_r} - \eta_{_f}\), \(c_{_T} = c_r - c_f\), and \(\widehat{R}_n(s, a) \doteq b_rR_n(s, a) + c_{_T}\). 

\newpage

Equation \eqref{q_r_bar_shifted_by_c} is now in the same form as Equation (B.14) (i.e., General Differential Q) from \citet{Wan2021-re}, who showed that the equation converges a.s. \(Q_n\) to \(q_\infty\) as \( n \to \infty\). Moreover, from this result, \citet{Wan2021-re} showed that \(\bar{R}_n\) converges a.s. to \(\bar{r}_*\) as \( n \to \infty\), and that \(\bar{r}_{\pi_t}\) converges a.s. to \(\bar{r}_*\), where \(\pi_t\) is a greedy policy with respect to \(Q_t\). Given that General RED Q adheres to all the assumptions listed for General Differential Q in \citet{Wan2021-re}, these convergence results apply to General RED Q.\\

\textbf{Convergence of the subtask estimates:}

Consider Equation \eqref{q_r_bar_shifted_by_c}. We can rewrite this equation, \(\forall s \in \mathcal{S}, a \in \mathcal{A}\), as follows:
\begin{align}
\begin{split}
    & Q_{n+1}(s, a) = Q_{n}(s, a) + \,\ldots\\
    & \quad \alpha_{\nu(n, s, a)} \left(b_rR_n(s,a) + \eta_{_f} \sum Q_n + c_{_T} - \eta_{_r} \sum Q_n + \max_{a'} Q_n(S_n'(s, a), a') - Q_n(s, a) \right) I\{(s, a) \in Y_n\} \nonumber\\\\
    & Q_{n+1}(s, a) = Q_{n}(s, a) + \,\ldots\\
    & \quad \alpha_{\nu(n, s, a)} \left(\hat{R}_n(s,a) + c_{_T} - \eta_{_r} \sum Q_n + \max_{a'} Q_n(S'_n(s, a), a') - Q_n(s, a)\right) I\{(s, a) \in Y_n\}, \nonumber \label{q_r_hat}
    \end{split}\\
\end{align}
where,
\begin{align}
\hat{R}_n(s, a) &\doteq b_rR_n(s, a) + \eta_{_f} \sum Q_n \label{eqn_r_hat_q_1}\\
&= b_rR_n(s,a) + b_1(\eta_{_1} \sum Q_n) + \ldots +b_k(\eta_{_k} \sum Q_n) \label{eqn_r_hat_q_2}\\
&\doteq b_rR_n(s,a) + b_1\hat{Z}_{1, n} + \ldots +b_k\hat{Z}_{k, n}. \label{eqn_r_hat_q_3}
\end{align} 

Now consider an MDP, \(\hat{\mathcal{M}}\), which has rewards, \(\mathcal{\hat{R}}\), as defined in Equation \eqref{eqn_r_hat_q_1}, has the same state and action spaces as \(\mathcal{\tilde{M}}\), and has the transition probabilities defined as:
\begin{align}
    \hat{p}(s', \hat{r} \mid s, a) &\doteq p(s', r \mid s, a)\\
    & = \tilde{p}(s', \tilde{r} \mid s, a) \quad \text{(by Definition \ref{definition_1})}, 
    \label{q_z_probs_r_hat}
\end{align}
such that \(\hat{\mathcal{M}} \doteq \langle\mathcal{S}, \mathcal{A}, \mathcal{\hat{R}}, \hat{p}\rangle\). It is easy to check that the communicating assumption holds for the MDP, \(\hat{\mathcal{M}}\). Moreover, given Equation \eqref{q_r_hat} and Assumptions \ref{assumption_subtask_function} and \ref{assumption_subtask_independence}, the optimal average-reward for the MDP, \(\hat{\mathcal{M}}\), \(\hat{\bar{r}}_*\), can be written as follows:
\begin{align}
    \hat{\bar{r}}_* = \bar{r}_* - c_{_T}. \label{q_z_star_modified_vs_z_star}
\end{align}

Now, because
\begin{align}
    q_\infty(s, a)  & = \sum_{s', \tilde{r}} \tilde{p}(s', \tilde{r} \mid s, a) (\tilde{r} - \bar{r}_* + \max_{a'} q_\infty (s', a')) \quad \text{(from \eqref{q_system_of_eqns})} \nonumber\\
    & = \sum_{s', \tilde{r}} \tilde{p}(s', \tilde{r} \mid s, a) (\tilde{r} - (\hat{\bar{r}}_* + c_{_T}) + \max_{a'} q_\infty (s', a')) \quad \text{(from \eqref{q_z_star_modified_vs_z_star})} \nonumber\\
    & = \sum_{s', \tilde{r}} \tilde{p}(s', \tilde{r} \mid s, a) (\tilde{r} - c_{_T} - \hat{\bar{r}}_* + \max_{a'} q_\infty (s', a'))\nonumber\\
    & = \sum_{s', \tilde{r}} \tilde{p}(s', \tilde{r} \mid s, a) (\hat{r} - \hat{\bar{r}}_* + \max_{a'} q_\infty (s', a')) \quad \text{(from \eqref{q_r_hat})} \nonumber\\
    & = \sum_{s', \hat{r}} \hat{p}(s', \hat{r} \mid s, a) (\hat{r} - \hat{\bar{r}}_* + \max_{a'} q_\infty (s', a')) \quad \text{(from \eqref{q_z_probs_r_hat})},\nonumber
\end{align}

we can see that \(q_\infty\) is a solution of not just the state-action value Bellman optimality equation for the MDP, \(\mathcal{\tilde{M}}\), but also the state-action value Bellman optimality equation for the MDP, \(\hat{\mathcal{M}}\).

Next, consider an arbitrary \(i\)th subtask. As per Equations \eqref{eqn_r_hat_q_2} and \eqref{eqn_r_hat_q_3}, we can write the optimal subtask value for the MDP, \(\hat{\mathcal{M}}\), \(\hat{z}_{i_*}\), as follows:
\begin{align}
    \hat{z}_{i_*} = z_{i_*} + c_i \label{q_z_star_shifted_vs_z_star}.
\end{align}

We can then combine Equations \eqref{q_z_sys_eqns}, \eqref{q_z_incremental_1}, and \eqref{q_z_star_shifted_vs_z_star}, which yields:
\begin{align}
    \hat{z}_{i_*} = \eta_{_i} \sum q_\infty \label{q_z_star_shifted_vs_q_infty}.
\end{align}

Next, we can combine Equation \eqref{q_z_incremental_1} with the result from \citet{Wan2021-re} which shows that \(Q_n \to q_\infty\), which yields: 
\begin{align}
    Z_{i, n} \to \eta_{_i} \sum q_\infty - c_i \label{q_z_convergence_final_1}.
\end{align}

Moreover, because \(\eta_{_i} \sum q_\infty = \hat{z}_{i_*}\) (Equation \eqref{q_z_star_shifted_vs_q_infty}), we have:
\begin{align}
    Z_{i, n} \to \hat{z}_{i_*} - c_i \label{q_z_convergence_final_2}.
\end{align}

Finally, because \(\hat{z}_{i_*} = z_{i_*} + c_i\) (Equation \eqref{q_z_star_shifted_vs_z_star}), we have: 
\begin{align} \label{q_z_convergence_final_3}
    Z_{i, n} \to z_{i_*} \text{\ \  a.s. as \ \  } n \to \infty.
\end{align}

We conclude by considering \(z_{i, \pi_t} \; \forall z_i \in \mathcal{Z}\), where \(\pi_t\) is a greedy policy with respect to \(Q_t\). Given that \(Q_t \to q_\infty\) and \(\bar{r}_{\pi_t} \to \bar{r}_*\), almost surely, it directly follows from Definition \ref{definition_1} that \(z_{i, \pi_t} \to z_{i_*} \; \forall z_i \in \mathcal{Z}\), almost surely.

\vspace{12pt}

\subsubsection{Proof of Theorem \ref{theorem_convergence_of_red_q_update} (for \emph{Piecewise Linear} Subtask Functions)}
\label{proof_q_piecewise}

We now provide the proof for \emph{piecewise linear} subtask functions, where the reward-extended TD error can be expressed as follows: 
\begin{equation}
\beta_{i,n}(s, a) = 
\begin{cases} 
(-1/b^{1}_i)\left(\tilde{R}_{1,n}(s, a) - \bar{R}_n - \delta_n(s,a)\right), \,\ r_0 \leq R_{n}(s, a) < r_1 \\
\vdots \\
(-1/b^{m}_i)\left(\tilde{R}_{m,n}(s, a) - \bar{R}_n - \delta_n(s,a)\right), \,\ r_{m-1} \leq R_{n}(s, a) \leq r_m \nonumber
\end{cases},
\end{equation}
where \(r_u \in \mathcal{R} \; \forall \, u = 0, 1, \ldots, m\), and \(r_0 \leq r_1 \leq \ldots \leq r_m\), such that \(r_0, r_m\) represent the lower and upper bounds of the observed per-step reward, \(R_{n}(s, a)\), respectively. Our general strategy in this case is to use a \emph{multiple-timescales} argument, such that we leverage Theorem 2 in Section 6 of \citet{Borkar2009-sr}, along with the results from Theorems B.1 and B.2 of \citet{Wan2021-re}.

To begin, let us consider Assumption \ref{assumption_subtask_stepsize}, which enables the formulation of a multiple-timescales argument. In particular, the \(\alpha_{z_1, n}/ \alpha_n \to 0\), \(\alpha_{z_1, n}/ \alpha_{\bar{r},n} \to 0\), and \(\{\alpha_{z_i, n} / \alpha_{z_{i - 1}, n} \to 0\}_{i=2}^{k}\) conditions imply that the subtask step sizes, \(\{\alpha_{z_i, n}\}_{n=0}^\infty \, \forall z_i \in \mathcal{Z}\), decrease to 0 at faster rates than the value function and average-reward step sizes, \(\{\alpha_n\}_{n=0}^\infty\) and \(\{\alpha_{\bar{r},n}\}_{n=0}^\infty\), respectively. This implies that the subtask updates move on slower timescales compared to the value function and average-reward updates. Hence, as argued in Section 6 of \citet{Borkar2009-sr}, the (faster) value function and average-reward updates, \eqref{async_q_value_update_eqn} and \eqref{async_q_r_bar_update_eqn}, view the (slower) subtask updates, \eqref{async_q_z_update_eqn}, as quasi-static, while the (slower) subtask updates view the (faster) value function and average-reward updates as nearly equilibrated (as we will show below, the results from \citet{Wan2021-re} imply the existence of such an equilibrium point). Similarly, the \(\{\alpha_{z_i, n} / \alpha_{z_{i - 1}, n} \to 0\}_{i=2}^{k}\) condition implies that each subtask update views the other subtask updates as either quasi-static or nearly-equilibrated.

As such, having established the multiple-timescales argument, we now proceed to show the convergence of the value function, average-reward, and subtask estimates:

\newpage

\textbf{Convergence of the value function and average-reward estimates:}

Given the multiple-timescales argument, such that the subtask estimates are viewed as quasi-static (i.e., constant), Equation \eqref{async_q_value_update_eqn} can be viewed as being of the same form as Equation (B.4) (i.e., General Differential Q) from \citet{Wan2021-re}, who showed (via Theorem B.2) that the equation converges a.s. \(Q_n\) to \(q_\infty\) as \( n \to \infty\). Moreover, from this result, \citet{Wan2021-re} showed that \(\bar{R}_n\) converges a.s. to \(\bar{r}_*\) as \( n \to \infty\), and that \(\bar{r}_{\pi_t}\) converges a.s. to \(\bar{r}_*\), where \(\pi_t\) is a greedy policy with respect to \(Q_t\). Given that General RED Q adheres to all the assumptions listed for General Differential Q in \citet{Wan2021-re}, these convergence results apply to General RED Q.\\

\textbf{Convergence of the subtask estimates:}

Let us consider the asynchronous subtask updates \eqref{async_q_z_update_eqn}. Each update in \eqref{async_q_z_update_eqn} is of the same form as Equation 7.1.2 of \citet{Borkar2009-sr}. Accordingly, to show the convergence of the subtask estimates, we can apply the result in Section 7.4 of \citet{Borkar2009-sr}, which shows the convergence of asynchronous updates that are of the same form as Equation 7.1.2. To apply this result, given Assumptions \ref{assumption_async_step_size_1} and \ref{assumption_async_step_size_q_2}, we only need to show the convergence of the \emph{synchronous} version of the subtask updates:
\begin{equation}
\label{sync_z_update_q}
Z_{i, n+1} = Z_{i, n} + \alpha_{z_i, n} \left(g_i(Z_{i, n}) + M^{z_i}_{n+1}\right) \; \forall z_i \in \mathcal{Z},
\end{equation}
where, 
\begin{align*}
& g_i(Z_{i, n})(s,a) \doteq \sum_{s', r} p(s', r \mid s, a)\phi_{i,n}(s,a) - Z_{i, n},\\
& \phi_{i,n}(s,a) \doteq
\begin{cases}
-\frac{1}{b^{1}_i}\left(b^{1}_rR_n(s, a) + \ldots + b^{1}_{i-1}Z_{i-1,n} + b^{1}_{i+1}Z_{i+1,n} + \ldots + b^{1}_kZ_{k,n} - \bar{R}_n - \delta_n(s,a)\right), \ldots\\
\hspace{1.5cm} \ldots, r_0 \leq R_n(s, a) < r_1, \\
\vdots \\
-\frac{1}{b^{m}_i}\left(b^{m}_rR_n(s, a) + \ldots + b^{m}_{i-1}Z_{i-1,n} + b^{m}_{i+1}Z_{i+1,n} + \ldots + b^{m}_kZ_{k,n} - \bar{R}_n - \delta_n(s,a)\right), \ldots\\
\hspace{1.5cm} \ldots, r_{m-1} \leq R_n(s, a) \leq r_m,
\end{cases},\\
& M^{z_i}_{n+1}(s,a) \doteq \left(\phi_{i,n}(s,a) - Z_{i,n}\right) - g_i(Z_{i, n})(s,a).
\end{align*}

To show the convergence of the synchronous update \eqref{sync_z_update_q} under the multiple-timescales argument, we can apply the result of Theorem 2 in Section 6 of \citet{Borkar2009-sr} to show that \(Z_{i,n} \to z_{i_*} \forall z_i \in \mathcal{Z}\) a.s. as \(n \to \infty\). This theorem requires that 3 assumptions be satisfied. As such, we will now show, via Lemmas \ref{lemma_q_1} - \ref{lemma_q_3}, that these 3 assumptions are indeed satisfied:\\

\begin{lemma}
\label{lemma_q_1}
The value function update, \( Q_{n+1} = Q_{n} + \alpha_n (h(Q_n) + M_{n+1})\), where
\begin{align*}
h(Q_n)(s,a) & \doteq \sum_{s', \tilde{r}} \tilde{p}(s', \tilde{r} \mid s, a) (\tilde{r} - \bar{R}_n + \max_{a'} Q_n(s', a') - Q_n(s,a)), \\\nonumber
& = \sum_{s', r} p(s', r \mid s, a) (f(r, Z_{1,n}, Z_{2,n}, \ldots, Z_{k,n}) - \psi(Q_n) + \max_{a'} Q_n(s', a') - Q_n(s,a)), \\\nonumber
M_{n + 1}(s,a) & \doteq  \left(\tilde{R}_n(s, a) - \psi(Q_n) + \max_{a'} Q_n(S_n'(s, a), a') - Q_n(s,a) \right) - h(Q_n)(s,a),\\\nonumber
\psi(Q_n) & = \bar{R}_n \; \text{is a ‘reference function’ as defined in \citet{Wan2021-re}}, \text{ and}\\\nonumber
Z_{1,n}, Z_{2,n}&, ..., Z_{k,n} \; \text{are quasi-static under the multiple-timescales argument},
\end{align*}
has a globally asymptotically stable equilibrium, \(q_\infty(Z_{1,n}, Z_{2,n}, \ldots, Z_{k,n})\), where \(q_\infty\) is a Lipschitz map.
\end{lemma}

\begin{proof}
This was shown in Theorem B.2 of \citet{Wan2021-re}.
\end{proof}

\newpage

\begin{lemma}
\label{lemma_q_2}
The subtask update rules \eqref{sync_z_update_q} each have a globally asymptotically stable equilibrium, \(z_{i_*}\).
\end{lemma}
\begin{proof}
Let us begin by considering the ‘fastest’ subtask, \(Z_{1,n}\), under the multiple-timescales argument, such that all other subtasks, \(Z_{2,n}, \ldots, Z_{k,n}\), operate on slower timescales and can be considered quasi-static. Applying the results of Theorems B.1 and B.2 of \citet{Wan2021-re} under the multiple-timescales argument, we have that: \(h(Q_n) \to 0\) and \(\bar{R}_n \to \bar{r}_{*}\) as \(Q_n \to q_\infty\). Importantly, as argued in Section \ref{red}, \(h(Q_n) \to 0\) implies that \(\delta_n(s,a) \to \lambda_j \in \mathbb{R}\) for each \(j\)th piecewise segment in \(\phi_{i,n}(s,a)\). As such, we can interpret \(Z_{1,n}\) as being the only ‘moving’ variable in \(g_1(Z_{1,n})\), such that all other parameters in the update are static, quasi-static, or nearly-equilibrated. 

Let us now consider the ODE associated with \(g_1(Z_{1,n})\), 
\begin{equation}
\label{eqn_q_z1_ode}
\dot{x}_t = g_1(x_t).
\end{equation} 
To show that the update rule for \(Z_{1,n}\) has a globally asymptotically stable equilibrium, \(z_{1_*}\), it suffices to show that there exists a \emph{Lyapunov function} for the associated ODE \eqref{eqn_q_z1_ode}. 

To this end, we first note, given the discussion in Section \ref{red}, that \(z_{1_*}\) is an equilibrium point for the update rule associated with \(Z_{1,n}\). Moreover, given Assumptions \ref{assumption_subtask_unique}, \ref{assumption_communicating}, and \ref{assumption_action_value_function_uniqueness}, we know that \(z_{1_*}\) is the unique equilibrium point.

We now show the existence of a Lyapunov function with respect to the aforementioned equilibrium point, \(z_{1_*}\). In particular, we consider the function, \(L\), defined by:
\[
L(Z_{1}) = \frac{1}{2}(Z_{1} - z_{1_*})^2.
\]

To establish that \(L\) is a Lyapunov function, we must show that:
\begin{enumerate}
\item \(L\) is continuous,
\item \(L(Z_{1}) = 0\) if \(Z_{1} = z_{1_*}\),
\item \(L(Z_1) > 0\) if \(Z_{1} \neq z_{1_*}\), and
\item For any solution \(\{x_t\}_{t \geq 0}\) of the associated ODE \eqref{eqn_q_z1_ode} and \(0 \leq s < t\), we have \(L(x_t) < L(x_s)\) for all \(x_s \neq z_{1_*}\).\\
\end{enumerate}

It directly follows from the definition of \(L\) that the first three conditions are satisfied.\\ 

We now show that fourth condition is also satisfied:

Let \(\{x_t\}_{t \geq 0}\) be a solution to the associated ODE \eqref{eqn_q_z1_ode}, and let \(0 \leq s < t\). By the chain rule, we have that:
\begin{equation}
\frac{d}{dt} L(x_t) = \frac{\partial L}{\partial x_t} \cdot \frac{dx_t}{dt} = (x_t - z_{1_*}) \cdot g_1(x_t) = (x_t - z_{1_*}) \cdot (\mathbb{E}_\pi[\phi_1] - x_t)= (x_t - z_{1_*}) \cdot (z_{1_*} - x_t).
\end{equation}

We will now analyze the sign of the term \((x_t - z_{1_*}) \cdot (z_{1_*} - x_t)\) when \(x_t \neq z_{1_*}\). There are two cases to consider:

\textbf{Case 1: \(x_t > z_{1_*}\):} We have that \((x_t - z_{1_*}) > 0\) and \((z_{1_*} - x_t) < 0\). Hence, we can conclude that \((x_t - z_{1_*}) \cdot (z_{1_*} - x_t) < 0\).

\textbf{Case 2: \(x_t < z_{1_*}\):} We have that \((x_t - z_{1_*}) < 0\) and \((z_{1_*} - x_t) > 0\). Hence, we can conclude that \((x_t - z_{1_*}) \cdot (z_{1_*} - x_t) < 0\).

As such, in both cases we can conclude that \((x_t - z_{1_*}) \cdot (z_{1_*} - x_t) < 0\), which implies that \(L(x_t) < L(x_s)\) for all \(x_s \neq z_{1_*}\) and \(0 \leq s < t\).\\

As such, we have now verified the four conditions, and can therefore conclude that \(L\) is a valid Lyapunov function. Consequently, we can conclude, under the multiple-timescales argument, that \(Z_{1,n}\) has a globally asymptotically stable equilibrium, \(z_{1_*}\).

Critically, we can leverage the above result to show, using the same techniques as above, that \(Z_{2,n}\) has a globally asymptotically stable equilibrium, \(z_{2_*}\), where \(Q_n, \bar{R}_n, \text{ and } Z_{1,n}\) are considered to be nearly equilibrated, and all remaining subtasks are considered quasi-static. The process can then be repeated for \(Z_{3,n}\) and so forth, thereby showing that the subtask update rules \eqref{sync_z_update_q} each have a globally asymptotically stable equilibrium, \(z_{i_*}\). This completes the proof.
\end{proof}

\vspace{10pt}

\begin{lemma}
\label{lemma_q_3}
\(\sup_n(\vert\vert Q_n \vert\vert + \vert\vert Z_{1, n} \vert\vert) + \vert\vert Z_{2,n} \vert\vert + \ldots + \vert\vert Z_{k,n} \vert\vert) < \infty\) a.s.
\end{lemma}
\begin{proof}
It was shown in Theorem B.2 of \citet{Wan2021-re} that \(\sup_n(\vert\vert Q_n \vert\vert) < \infty\) a.s. Hence, we only need to show that \(\sup_n(\vert\vert Z_{i,n} \vert\vert) < \infty \; \forall z_i \in \mathcal{Z}\) a.s. To this end, we can apply Theorem 7 in Section 3 of \citet{Borkar2009-sr}. This theorem requires 4 assumptions for each \(z_i \in \mathcal{Z}\):
\begin{itemize}
    \item \textbf{(A1)} The function \(g_i\) is Lipschitz. That is, \(\vert \vert g_i(x) - g_i(y)\vert \vert \leq U_i \vert \vert x - y\vert \vert\) for some \(0 < U_i < \infty\).
    \item \textbf{(A2)} The sequence \(\{ \alpha_{z_i, n}\}_{n=0}^\infty\) satisfies \(\alpha_{z_i, n} > 0 \; \forall n \geq 0\), \(\sum \alpha_{z_i, n} = \infty\), and \(\sum \alpha_{z_i,n}^2 < \infty\).
    \item \textbf{(A3)} \(\{M_n^{z_i}\}_{n=0}^\infty\) is a martingale difference sequence that is square-integrable.
    \item \textbf{(A4)} The functions \(g_i(x)_d \doteq g_i(dx)/d\), \(d \geq 1, x \in \mathbb{R}\), satisfy \(g_i(x)_d \to g_i(x)_\infty\) as \(d \to \infty\), uniformly on compacts for some \(g_{i_\infty} \in C(\mathbb{R})\). Furthermore, the ODE \(\dot x_t = g_i(x_t)_\infty\) has the origin as its unique globally asymptotically stable equilibrium.
\end{itemize}

Consider an arbitrary \(i\)th subtask. We note that Assumption \textbf{(A1)} is satisfied given that all operators in \(g_i\) are Lipschitz. Moreover, we note that Assumption \ref{assumption_step_size} satisfies Assumption \textbf{(A2)}. 

We now show that \(\{M_n^{z_i}\}_{n=0}^\infty\) is indeed a martingale difference sequence that is square-integrable, such that \(\mathbb{E}_\pi[M^{z_i}_{n+1} \mid \mathcal{F}_n] = 0 \text{ \ a.s., } n \geq 0\), and 
\(\mathbb{E}_\pi[(M_{n+1}^{z_i})^2 \mid \mathcal{F}_n] < \infty \text{ \ a.s., \ } n \geq 0\). To this end, let \(\mathcal{F}_n \doteq \sigma(Z_{i,u}, M^{z_i}_u, u \leq n), n \geq 0\) denote an increasing family of \(\sigma\)-fields. We have that, for any \(s \in \mathcal{S}\) and \(a \in \mathcal{A}\): \(\mathbb{E}_\pi[M^{z_i}_{n+1}(s,a) \mid \mathcal{F}_n] = \mathbb{E}_\pi \left[\phi_{i,n}(s,a) - Z_{i,n}\right] - g_i(Z_{i, n})(s,a) = 0\). Moreover, since all components of \(\{M_n^{z_i}\}_{n=0}^\infty\) involve bounded quantities, it directly follows that \(\mathbb{E}[\vert \vert M_{n+1}^{z_i} \vert \vert^2 \mid \mathcal{F}_n] \leq K\) for some finite constant \(K > 0\). Hence, Assumption \textbf{(A3)} is verified.

We now verify Assumption \textbf{(A4)}. Under the multiple-timescales argument, we have that:
\begin{align}
    g_i(x)_\infty = \lim_{d \to \infty} g_i(x)_d = \lim_{d \to \infty} \frac{\mathbb{E}_\pi \left[\phi_{i}\right] - dx}{d} = 0 - x = -x.
\end{align}
Clearly, \(g_i(x)_\infty\) is continuous in every \(x \in \mathbb{R}\). As such, we have that \(g_i(x)_\infty \in C(\mathbb{R})\). Now consider the ODE \(\dot{x} = g_i(x)_\infty\). This ODE has the origin as an equilibrium since \(g_i(0)_\infty = 0\). Furthermore, given Lemma \ref{lemma_q_2}, we can conclude that this equilibrium must be the unique globally asymptotically stable equilibrium, thereby satisfying Assumption \textbf{(A4)}.

Assumptions \textbf{(A1)} - \textbf{(A4)} are hence verified, meaning that we can apply the results of Theorem 7 in Section 3 of \citet{Borkar2009-sr} to conclude that \(\sup_n(\vert\vert Z_{i,n} \vert\vert) < \infty \; \forall z_i \in \mathcal{Z}\), almost surely, and hence, that \(\sup_n(\vert\vert Q_n \vert\vert + \vert\vert Z_{1, n} \vert\vert) + \vert\vert Z_{2,n} \vert\vert + \ldots + \vert\vert Z_{k,n} \vert\vert) < \infty\), almost surely.
\end{proof}

As such, we have now verified the 3 assumptions required by Theorem 2 in Section 6 of \citet{Borkar2009-sr}, which means that we can apply the result of the theorem to conclude that \(Z_{i,n} \to z_{i_*} \, \forall z_i \in \mathcal{Z}\) a.s. as \(n \to \infty\).

Finally, as was done in the proof for linear subtask functions, we conclude the proof by considering \(z_{i, \pi_t} \; \forall z_i \in \mathcal{Z}\), where \(\pi_t\) is a greedy policy with respect to \(Q_t\). Given that \(Q_t \to q_\infty\) and \(\bar{r}_{\pi_t} \to \bar{r}_*\), almost surely, it directly follows from Definition \ref{definition_1} that \(z_{i, \pi_t} \to z_{i_*} \; \forall z_i \in \mathcal{Z}\), almost surely.

This completes the proof of Theorem ~\ref{theorem_convergence_of_red_q_update}.

\newpage


\newpage
\section{Leveraging the RED RL Framework for CVaR Optimization}
\label{appendix_RED_CVAR}

This appendix contains details regarding the adaptation of the RED RL framework for CVaR optimization. We first derive an appropriate subtask function, then use it to adapt the RED RL algorithms (see Appendix \ref{appendix_RED_algs}) for CVaR optimization. In doing so, we arrive at the \emph{RED CVaR algorithms}, which are presented in full at the end of this appendix. These RED CVaR algorithms allow us to optimize CVaR (and VaR) without the use of an augmented state-space or an explicit bi-level optimization. We also provide a convergence proof for the tabular RED CVaR Q-learning algorithm, which shows that the VaR and CVaR estimates converge to the optimal long-run VaR and CVaR, respectively.

\subsection{A Subtask-Driven Approach for CVaR Optimization}
In this section, we use the RED RL framework to derive a subtask-driven approach for CVaR optimization that does not require an augmented state-space or an explicit bi-level optimization. To begin, let us consider Equation \eqref{eq_cvar_2}, which is displayed below as Equation \eqref{eq_c_1} for convenience:
\begin{subequations}
\label{eq_c_1}
\begin{align}
\label{eq_c_1_1}
\text{CVaR}_{\tau}(R_t) & = \sup_{y \in \mathcal{R}} \mathbb{E}[y - \frac{1}{\tau}(y - R_t)^{+}]\\
\label{eq_c_1_2}
& = \mathbb{E}[\text{VaR}_{\tau}(R_t) - \frac{1}{\tau}(\text{VaR}_{\tau}(R_t) - R_t)^{+}], 
\end{align}
\end{subequations}
where \(\tau \in (0, 1)\) denotes the CVaR parameter, and \(R_t\) denotes the observed per-step reward.

We can see from Equation \eqref{eq_c_1} that CVaR can be interpreted as an expectation (or average) of sorts, which suggests that it may be possible to leverage the average-reward MDP to optimize this expectation, by treating the reward CVaR as the average-reward, \(\bar{r}_{\pi}\), that we want to optimize. However, this requires that we know the optimal value of the scalar, \(y\), because the expectation in Equation \eqref{eq_c_1_2} only holds for this optimal value (which corresponds to the per-step reward VaR). Unfortunately, this optimal value is typically not known beforehand, so in order to optimize CVaR, we also need to optimize \(y\).

Importantly, we can utilize RED RL framework to turn the optimization of \(y\) into a subtask, such that CVaR is the primary control objective (i.e., the \(\bar{r}_{\pi}\) that we want to optimize), and VaR (\(y\) in Equation \eqref{eq_c_1}), is the (single) subtask. This is in contrast to existing MDP-based methods, which typically leverage Equation \eqref{eq_c_1} when optimizing for CVaR by augmenting the state-space with a state that corresponds (either directly or indirectly) to an estimate of \(\text{VaR}_{\tau}(R_t)\) (in this case, \(y\)), and solving the bi-level optimization shown in Equation \eqref{eq_cvar_3}, thereby increasing computational costs. 

To utilize the RED RL framework, we first need to derive a valid subtask function for CVaR that satisfies the requirements of Definition \ref{definition_1}. To this end, let us consider Equation \eqref{eq_c_1}. In particular, we can see that if we treat the expression inside the expectation in Equation \eqref{eq_c_1} as our subtask function, \(f\) (see Definition \ref{definition_1}), then we have a piecewise linear subtask function that is invertible with respect to each input given all other inputs, where the subtask, VaR, is independent of the observed per-step reward. Hence, we can adapt Equation \eqref{eq_c_1} as our subtask function (given that is satisfies Definition \ref{definition_1}), as follows:
\begin{equation}
\label{eq_c_2}
\tilde{R}_t = \text{VaR} - \frac{1}{\tau}(\text{VaR} - R_t)^{+},
\end{equation}
where \(R_t\) is the observed per-step reward, \(\tilde{R}_t\) is the extended per-step reward, VaR is the value-at-risk of the observed per-step reward, and \(\tau\) is the CVaR parameter. Importantly, this is a valid subtask function with the following properties: the average (or expected value) of the \emph{extended} reward corresponds to the CVaR of the \emph{observed} reward, and the optimal average of the \emph{extended} reward corresponds to the optimal CVaR of the \emph{observed} reward. This is formalized as Corollaries \ref{corollary_c_1} - \ref{corollary_c_4} below:

\clearpage

\begin{corollary} 
\label{corollary_c_1}
The function presented in Equation \eqref{eq_c_2} is a valid subtask function.
\end{corollary}
\begin{proof}
The function presented in Equation \eqref{eq_c_2} is clearly a piecewise linear function that is invertible with respect to each input given all other inputs. Moreover, the subtask, VaR, is independent of the observed per-step reward. Hence, this function satisfies Definition \ref{definition_1} for the subtask, VaR.
\end{proof}

\begin{corollary} 
\label{corollary_c_2}
If the subtask, VaR (from Equation \eqref{eq_c_2}) is estimated, and such an estimate is equal to the long-run VaR of the observed per-step reward, then the average (or expected value) of the extended reward, \(\tilde{R}_t\), from Equation \eqref{eq_c_2} is equal to the long-run CVaR of the observed per-step reward.
\end{corollary}
\begin{proof}
This follows directly from Equation \eqref{eq_c_1_2}.
\end{proof}

\begin{corollary} 
\label{corollary_c_3}
If the subtask, VaR (from Equation \eqref{eq_c_2}) is estimated, and the resulting average of the extended reward from Equation \eqref{eq_c_2} is equal to the long-run CVaR of the observed per-step reward, then the VaR estimate is equal to the long-run VaR of the observed per-step reward.
\end{corollary}
\begin{proof}
This follows directly from Equation \eqref{eq_c_1_2}.
\end{proof}

\begin{corollary} 
\label{corollary_c_4}
A policy that yields an optimal long-run average of the extended reward, \(\tilde{R}_t\), from Equation \eqref{eq_c_2} is a CVaR-optimal policy. In other words, the optimal long-run average of the extended reward corresponds to the optimal long-run CVaR of the observed reward.
\end{corollary}
\begin{proof}
For a given policy, we know from Equation \eqref{eq_c_1_1} that, across a range of VaR estimates, the best possible long-run average of the extended reward for that policy corresponds to the long-run CVaR of the observed reward for that same policy. Hence, the best possible long-run average of the extended reward that can be achieved across various policies and VaR estimates, corresponds to the optimal long-run CVaR of the observed reward.
\end{proof}

As such, we now have a subtask function with a single subtask, VaR, and an extended reward whose average, when optimized, corresponds to the optimal CVaR of the observed reward. We are now ready to apply the RED RL framework. First, we can derive the reward-extended TD error update for our subtask, VaR, using the methodology outlined in Section \ref{reward_extended_td}, where, in this case, we have a piecewise linear subtask function with two segments. To this end, if we assume that the CDF of the per-step reward distribution is continuous at \(\text{VaR}_{\tau}(R_{t+1})\), the resulting subtask update is as follows:
\begin{equation}
\label{eq_c_3}
\text{VaR}_{t+1} =
\begin{cases} 
      \text{VaR}_t + \alpha_{_{\text{VaR}},t} \left(\delta_t + \text{CVaR}_t - \text{VaR}_t \right), & R_{t+1} \geq \text{VaR}_t \\
      \text{VaR}_t + \alpha_{_{\text{VaR}},t}\left(\left(\frac{\tau}{ \tau - 1}\right)\delta_t + \text{CVaR}_t - \text{VaR}_t \right), & R_{t+1} < \text{VaR}_t
   \end{cases} \,,
\end{equation}
where \(\delta_t\) is the regular TD error, and \(\alpha_{_{\text{VaR}},t}\) is the step size.

Note that to derive the update \eqref{eq_c_3}, we used the fact that when \(R_{t+1} < \text{VaR}_{\tau}(R_{t+1})\) and the CDF of the per-step reward distribution is continuous at \(\text{VaR}_{\tau}(R_{t+1})\), we have that
\begin{subequations}
\label{eq_c_4}
\begin{align}
& \mathbb{E}[\text{VaR}_{\tau}(R_{t+1}) - \frac{1}{\tau}(\text{VaR}_{\tau}(R_{t+1}) - R_{t+1}) \mid R_{t+1} < \text{VaR}_{\tau}(R_{t+1})]\\
&\quad \quad = \text{VaR}_{\tau}(R_{t+1}) - \frac{1}{\tau}\left(\text{VaR}_{\tau}(R_{t+1}) - \mathbb{E}[R_{t+1} \mid R_{t+1} < \text{VaR}_{\tau}(R_{t+1})]\right)\\
&\quad \quad = \text{VaR}_{\tau}(R_{t+1}) - \frac{1}{\tau}(\text{VaR}_{\tau}(R_{t+1}) - \text{CVaR}_{\tau}(R_{t+1})).
\end{align}
\end{subequations}

With this update, we now have all the components needed to utilize the RED algorithms in Appendix \ref{appendix_RED_algs} to optimize CVaR (where CVaR corresponds to the \(\bar{r}_{\pi}\) that we want to optimize). We call these CVaR-specific algorithms, the \emph{RED CVaR algorithms}. The full algorithms are included at the end of this appendix. 

We now present the tabular \emph{RED CVaR Q-learning} algorithm, along with a corresponding theorem and convergence proof which shows that the VaR and CVaR estimates converge to the optimal long-run VaR and CVaR of the observed reward, respectively:

\textbf{RED CVaR Q-learning algorithm (tabular):}  We update a table of estimates, \(Q_t: \mathcal{S} \, \times\, \mathcal{A} \rightarrow \mathbb{R}\) as follows:
\begin{subequations}
\label{eq_c_9}
\begin{align}
\label{eq_c_9_1}
& \tilde{R}_{t+1} = \text{VaR}_t - \frac{1}{\tau}(\text{VaR}_t - R_{t+1})^{+}\\
\label{eq_c_9_2}
& \delta_{t} = \tilde{R}_{t+1} - \text{CVaR}_t +\max_a Q_{t}(S_{t+1}, a) - Q_{t}(S_t, A_t)\\
\label{eq_c_9_3}
& Q_{t+1}(S_t, A_t) = Q_{t}(S_t, A_t) + \alpha_{t}\delta_{t}\\
\label{eq_c_9_4}
& Q_{t+1}(s, a) = Q_{t}(s, a), \quad \forall s, a \neq S_t, A_t\\
\label{eq_c_9_5}
& \text{CVaR}_{t+1} = \text{CVaR}_t + \alpha_{_{\text{CVaR}},t}\delta_{t}\\
\label{eq_c_9_6}
& \text{VaR}_{t+1} =
\begin{cases} 
      \text{VaR}_t + \alpha_{_{\text{VaR}},t} \left(\delta_t + \text{CVaR}_t - \text{VaR}_t \right), & R_{t+1} \geq \text{VaR}_t \\
      \text{VaR}_t + \alpha_{_{\text{VaR}},t}\left(\left(\frac{\tau}{ \tau - 1}\right)\delta_t + \text{CVaR}_t - \text{VaR}_t \right), & R_{t+1} < \text{VaR}_t
   \end{cases} \,,
\end{align}
\end{subequations}

where \(R_{t}\) denotes the observed reward, \(\text{VaR}_t\) denotes the VaR estimate, \(\text{CVaR}_t\) denotes the CVaR estimate, \(\delta_t\) denotes the TD error, and \(\alpha_t\), \(\alpha_{_{\text{CVaR}},t}\), and \(\alpha_{_{\text{VaR}},t}\) denote the step sizes.\\ 

\begin{theorem}
\label{theorem_c_1}
The RED CVaR Q-learning algorithm \eqref{eq_c_9} converges, almost surely, \(\text{CVaR}_t\) to \(\text{CVaR}^*\), \(\text{VaR}_t\) to \(\text{VaR}^*\), \(\text{CVaR}_{\pi_t}\) to \(\text{CVaR}^*\), \(\text{VaR}_{\pi_t}\) to \(\text{VaR}^*\), and \(Q_t\) to a solution of \(q\) in the Bellman Equation \eqref{eq_avg_reward_5}, up to an additive constant, \(c\), where \(\pi_t\) is any greedy policy with respect to \(Q_t\), if the following assumptions hold: 1) the MDP is communicating, 2) the solution of \(q\) in \eqref{eq_avg_reward_5} is unique up to a constant, 3) the step sizes are in accordance with Assumptions \ref{assumption_step_size}, \ref{assumption_async_step_size_1}, \ref{assumption_r_step_size}, \ref{assumption_subtask_stepsize}, and \ref{assumption_async_step_size_q_2}, 4) \(Q_{t}\) is updated an infinite number of times for all state-action pairs, such that the ratio of the update frequency of the most-updated state–action pair to the least-updated state–action pair is finite, 5) the subtask function outlined in Equation \eqref{eq_c_2} is in accordance with Definition \ref{definition_1}, 6) there is a unique solution to the CVaR optimization \eqref{eq_c_1}, and 7) the CDF of the limiting per-step reward distribution is continuous at \(\text{VaR}_{\tau}(R_{t+1})\). 
\end{theorem}

\begin{proof}
By definition, the RED CVaR Q-learning algorithm \eqref{eq_c_9} is of the form of the generic RED Q-learning algorithm \eqref{eq_alg_2}, where \(\text{CVaR}_t\) corresponds to \(\bar{R}_t\) and \(\text{VaR}_t\) corresponds to \(Z_{i, t}\) for a single subtask. We also know from Corollary \ref{corollary_c_1} that the subtask function used is valid. Hence, Theorem \ref{theorem_4_3} applies, such that:

\emph{i)} \(\text{CVaR}_t\) and \(\text{CVaR}_{\pi_t}\) converge a.s. to the optimal long-run average, \(\bar{r}*\), of the extended reward from the subtask function (i.e., the optimal long-run average of \(\tilde{R}_t\)),

\emph{ii)} \(\text{VaR}_t\) and \(\text{VaR}_{\pi_t}\) converge a.s. to the corresponding optimal subtask value, \(z*\), and

\emph{iii)} \(Q_t\) converges to a solution of \(q\) in the Bellman Equation \eqref{eq_avg_reward_5},

all up to an additive constant, \(c\).

Hence, to complete the proof, we need to show that \(\bar{r}* = \text{CVaR}^*\) and \(z* = \text{VaR}^*\):

From Corollary \ref{corollary_c_4} we know that the optimal long-run average of the extended reward corresponds to the optimal long-run CVaR of the observed reward, hence we can conclude that \(\bar{r}* = \text{CVaR}^*\). Finally, from Corollary \ref{corollary_c_3} we can deduce that since \(\text{CVaR}_t\) converges a.s. to \(\text{CVaR}^*\), then \(z*\) must correspond to \(\text{VaR}^*\). This completes the proof. 
\end{proof}

As such, with the RED CVaR Q-learning algorithm, we now have a way to optimize the long-run CVaR (and VaR) of the observed reward without the use of an augmented state-space, or an explicit bi-level optimization. See Section \ref{risk_red} and Appendix \ref{appendix_experiments} for empirical results obtained when using the RED CVaR algorithms.

\newpage

\subsection{Additional Commentary}

We now provide additional commentary on the subtask-driven approach for CVaR optimization:\\

\begin{remark}{In the above analysis, we made the assumption that the CDF of the limiting per-step reward distribution is continuous at \(\text{VaR}_{\tau}(R_{t+1})\), thereby yielding a simpler update (i.e., Equation \eqref{eq_c_3}) for the subtask, VaR. We note that this assumption was made for simplicity and convenience, and that an equivalent analysis could be conducted for a more generic VaR update (that does not make the aforementioned assumption) to yield the same theoretical conclusion as Theorem \ref{theorem_c_1}. Moreover, as shown empirically in Section \ref{risk_red} and Appendix \ref{appendix_experiments}, we found that using the simplified version of the VaR update \eqref{eq_c_3} yielded strong performance.\\}
\end{remark}

\begin{remark}{A natural question to ask could be whether we can extend the above convergence results to the prediction case. In other words, can we show that a tabular RED CVaR TD-learning algorithm will converge to the long-run VaR and CVaR of the observed reward induced by following a given policy? It turns out that, because we are not optimizing the expectation in Equation \eqref{eq_c_1_1} when doing prediction (we are only learning it), we cannot guarantee that we will eventually find the optimal VaR estimate, which implies that we may not recover the CVaR value (since Equation \eqref{eq_c_1_2} only holds to the optimal VaR value). However, this is not to say that a RED CVaR TD-learning algorithm has no use; in fact, we do use such an algorithm as part of an actor-critic architecture for optimizing CVaR in the inverted pendulum experiment (see Appendix \ref{appendix_experiments}). Empirically, as discussed in Section \ref{risk_red}, we find that this actor-critic approach is able to find the optimal CVaR policy.\\}
\end{remark}

\begin{remark}{It should be noted that in the risk measure literature, risk measures are typically classified into two categories: \emph{static} or \emph{dynamic}. Static risk measures are most commonly characterized as being interpretable, but lacking in \emph{time consistency} \citep{Boda2006-tw}. Conversely, dynamic risk measures are typically characterized as being difficult to interpret, but capable of time consistency. Curiously, in our case, the CVaR risk measure that we aim to optimize can be characterized as being both interpretable and time-consistent as \(t \to \infty\) (while noting that there is some time inconsistency before \(t \to \infty\)). We note that this does not affect the significance of our results, but rather suggests that the current categorization of risk measures may be insufficient to capture such nuances that occur in the average-reward setting.}
\end{remark}

\newpage

\subsection{RED CVaR Algorithms}
Below is the pseudocode for the RED CVaR algorithms.

\begin{algorithm}
   \caption{RED CVaR Q-Learning (Tabular)}
   \label{alg_cvar_1}
\begin{algorithmic}
    \STATE {\bfseries Input:} the policy \(\pi\) to be used (e.g., \(\varepsilon\)-greedy)
    \STATE {\bfseries Algorithm parameters:} step size parameters \(\alpha\), \(\alpha_{_{\text{CVaR}}}\), \(\alpha_{_{\text{VaR}}}\); CVaR parameter \(\tau\)
    \STATE Initialize \(Q(s, a) \: \forall s, a\) (e.g. to zero)
    \STATE Initialize CVaR arbitrarily (e.g. to zero)
    \STATE Initialize VaR arbitrarily (e.g. to zero)
    \STATE Obtain initial \(S\)
    \WHILE{still time to train}
        \STATE \(A \leftarrow\) action given by \(\pi\) for \(S\)
        \STATE Take action \(A\), observe \(R, S'\)
        \STATE \(\tilde{R} = \text{VaR} - \frac{1}{\tau} \max \{\text{VaR} - R, 0\}\)
        \STATE \(\delta = \tilde{R} - \text{CVaR} + \max_a Q(S', a) - Q(S, A)\)
        \STATE \(Q(S, A) = Q(S, A) + \alpha\delta\)
        \STATE \(\text{CVaR} = \text{CVaR} + \alpha_{_{\text{CVaR}}}\delta\)
        \IF{\(R \geq \text{VaR}\)}
            \STATE \(\text{VaR} = \text{VaR} + \alpha_{_{\text{VaR}}}(\delta + \text{CVaR} - \text{VaR})\)
        \ELSE
            \STATE \(\text{VaR} = \text{VaR} + \alpha_{_{\text{VaR}}}\left(\left(\frac{\tau}{ \tau - 1}\right)\delta + \text{CVaR} - \text{VaR}\right)\)
        \ENDIF
        \STATE \(S = S'\)
    \ENDWHILE
    \STATE return \(Q\)
\end{algorithmic}
\end{algorithm}

\begin{algorithm}
   \caption{RED CVaR Actor-Critic}
   \label{alg_cvar_2}
\begin{algorithmic}
    \STATE {\bfseries Input:} a differentiable state-value function parameterization \(\hat{v}(s, \boldsymbol{w})\); a differentiable policy parameterization \(\pi(a \mid s, \boldsymbol{\theta})\)
    \STATE {\bfseries Algorithm parameters:} step size parameters \(\alpha\), \(\alpha_{\pi}\), \(\alpha_{_{\text{CVaR}}}\), \(\alpha_{_{\text{VaR}}}\); CVaR parameter \(\tau\)
    \STATE Initialize state-value weights \(\boldsymbol{w} \in \mathbb{R}^{d}\) and policy weights \(\boldsymbol{\theta} \in \mathbb{R}^{d'}\) (e.g. to \(\boldsymbol{0}\))
    \STATE Initialize CVaR arbitrarily (e.g. to zero)
    \STATE Initialize VaR arbitrarily (e.g. to zero)
    \STATE Obtain initial \(S\)
    \WHILE{still time to train}
        \STATE \(A \sim \pi(\cdot \mid S, \boldsymbol{\theta})\)
        \STATE Take action \(A\), observe \(R, S'\)
        \STATE \(\tilde{R} = \text{VaR} - \frac{1}{\tau} \max \{\text{VaR} - R, 0\}\)
        \STATE \(\delta = \tilde{R} - \text{CVaR} + \hat{v}(S', \boldsymbol{w}) - \hat{v}(S, \boldsymbol{w})\)
        \STATE \(\boldsymbol{w} = \boldsymbol{w} + \alpha\delta\nabla\hat{v}(S, \boldsymbol{w})\)
        \STATE \(\boldsymbol{\theta} = \boldsymbol{\theta} + \alpha_{\pi}\delta\nabla \text{ln} \pi(A \mid S, \boldsymbol{\theta})\)
        \STATE \(\text{CVaR} = \text{CVaR} + \alpha_{_{\text{CVaR}}}\delta\)
        \IF{\(R \geq \text{VaR}\)}
            \STATE \(\text{VaR} = \text{VaR} + \alpha_{_{\text{VaR}}}(\delta + \text{CVaR} - \text{VaR})\)
        \ELSE
            \STATE \(\text{VaR} = \text{VaR} + \alpha_{_{\text{VaR}}}\left(\left(\frac{\tau}{ \tau - 1}\right)\delta + \text{CVaR} - \text{VaR}\right)\)
        \ENDIF
        \STATE \(S = S'\)
    \ENDWHILE
    \STATE return \(\boldsymbol{w}\), \(\boldsymbol{\theta}\)
\end{algorithmic}
\end{algorithm}


\clearpage

\section{Numerical Experiments}
\label{appendix_experiments}

This appendix contains details regarding the numerical experiments performed as part of this work. We discuss the experiments performed in the \emph{red-pill blue-pill} environment (see Appendix \ref{appendix_RPBP} for more details on the red-pill blue-pill environment), as well as the experiments performed in the well-known \emph{inverted pendulum} environment. 

\begin{figure}[htbp]
\centerline{\includegraphics[scale=0.6]{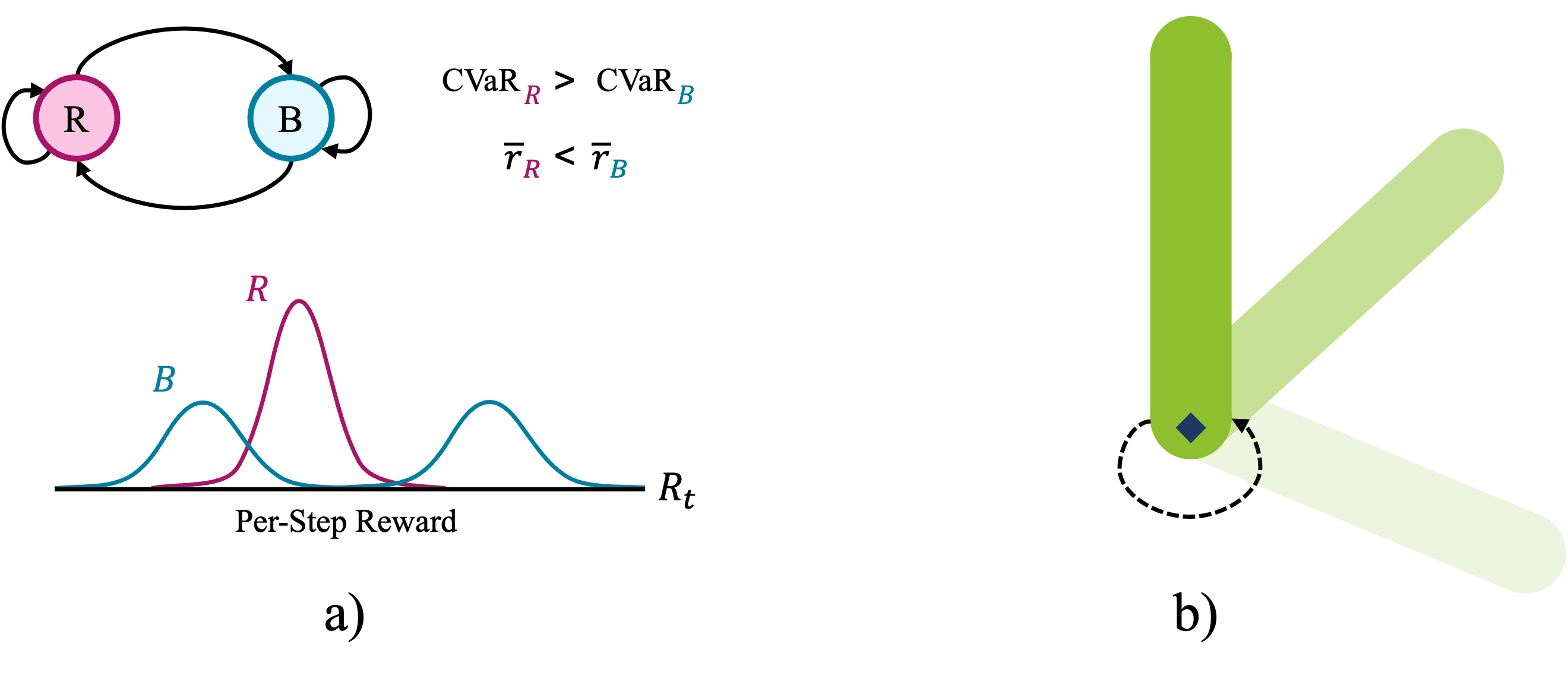}}
\caption{An illustration of the \textbf{a)} red-pill blue-pill, and \textbf{b)} inverted pendulum environments.}
\label{fig_experiments}
\end{figure}

The aim of the experiments was to contrast and compare the RED CVaR algorithms (see Appendix \ref{appendix_RED_CVAR}) with the Differential algorithms from \citet{Wan2021-re} in the context of risk-aware decision-making. In particular, we aimed to show how the RED CVaR algorithms could be utilized to optimize for CVaR (without the use of an augmented state-space or an explicit bi-level optimization scheme), and contrast the results to those of the Differential algorithms, which served as a sort of ‘baseline’ to illustrate how our \emph{risk-aware} approach contrasts a \emph{risk-neutral} approach. In other words, we aimed to show whether our algorithms could successfully enable a learning agent to act in a risk-aware manner instead of the usual risk-neutral manner.

In terms of the algorithms used, we used the RED CVaR Q-learning algorithm (Algorithm \ref{alg_cvar_1}) in the red-pill blue-pill experiment, and the RED CVaR Actor-Critic algorithm (Algorithm \ref{alg_cvar_2}) in the inverted pendulum experiment. In terms of the (risk-neutral) Differential algorithms used for comparison (see Appendix \ref{risk_neutral_algorithms} for the full algorithms), we used the Differential Q-learning algorithm (Algorithm \ref{alg_reg_1}) in the red-pill blue-pill experiment, and the Differential Actor-Critic algorithm (Algorithm \ref{alg_reg_2}) in the inverted pendulum experiment.

\subsection{Red-Pill Blue-Pill Experiment}
\label{exp_red_pill_blue_pill}

In the first experiment, we consider a two-state environment that we created for the purposes of testing our algorithms. It is called the \emph{red-pill blue-pill} environment (see Appendix \ref{appendix_experiments}), where at every time step an agent can take either a ‘red pill’, which takes them to the ‘red world’ state, or a ‘blue pill’, which takes them to the ‘blue world’ state. Each state has its own characteristic per-step reward distribution, and in this case, for a sufficiently low CVaR parameter, \(\tau\), the red world state has a per-step reward distribution with a lower (worse) mean but a higher (better) CVaR compared to the blue world state. As such, this task allows us to answer the following question: \emph{can a RED CVaR algorithm successfully enable the agent to learn a policy that prioritizes optimizing the reward CVaR over the average-reward?} In particular, we would expect that the RED CVaR Q-learning algorithm enables the agent to learn a policy that prefers to stay in the red world, and that the (risk-neutral) Differential Q-learning algorithm (from \citet{Wan2021-re}) enables the agent to learn a policy that prefers to stay in the blue world. This task is illustrated in Figure \ref{fig_experiments}{\textcolor{mylightblue}{a)}.

For this experiment, we ran both algorithms using various combinations of step sizes for each algorithm. We used an \(\varepsilon\)-greedy policy with a fixed epsilon of 0.1, and a CVaR parameter, \(\tau\), of 0.25. We set all initial guesses to zero. We ran the algorithms for 100k time steps.

For the Differential Q-learning algorithm, we tested every combination of the value function step size, \(\alpha\in\{1/n, \text{2e-4, 2e-3, 2e-2, 2e-1}\}\) (where \(1/n\) refers to a step size sequence that decreases the step size according to the time step, \(n\)), with the average-reward step size, \(\eta\alpha\), where \(\eta\in\{\text{1e-4, 1e-3, 1e-2, 1e-1, 1.0, 2.0}\}\), for a total of 30 unique combinations. Each combination was run 50 times using different random seeds, and the results were averaged across the runs. The resulting (averaged) average-reward over the last 1,000 time steps is displayed in Figure \ref{fig_tuning}. As shown in the figure, a value function step size of 2e-4 and an average-reward \(\eta\) of 1.0 resulted in the highest average-reward in the final 1,000 time steps in the red-pill blue-pill task. These are the parameters used to generate the results displayed in Figure \ref{fig_results_1}{\textcolor{mylightblue}{a)}. 

\begin{figure}[htbp]
\centerline{\includegraphics[scale=0.36]{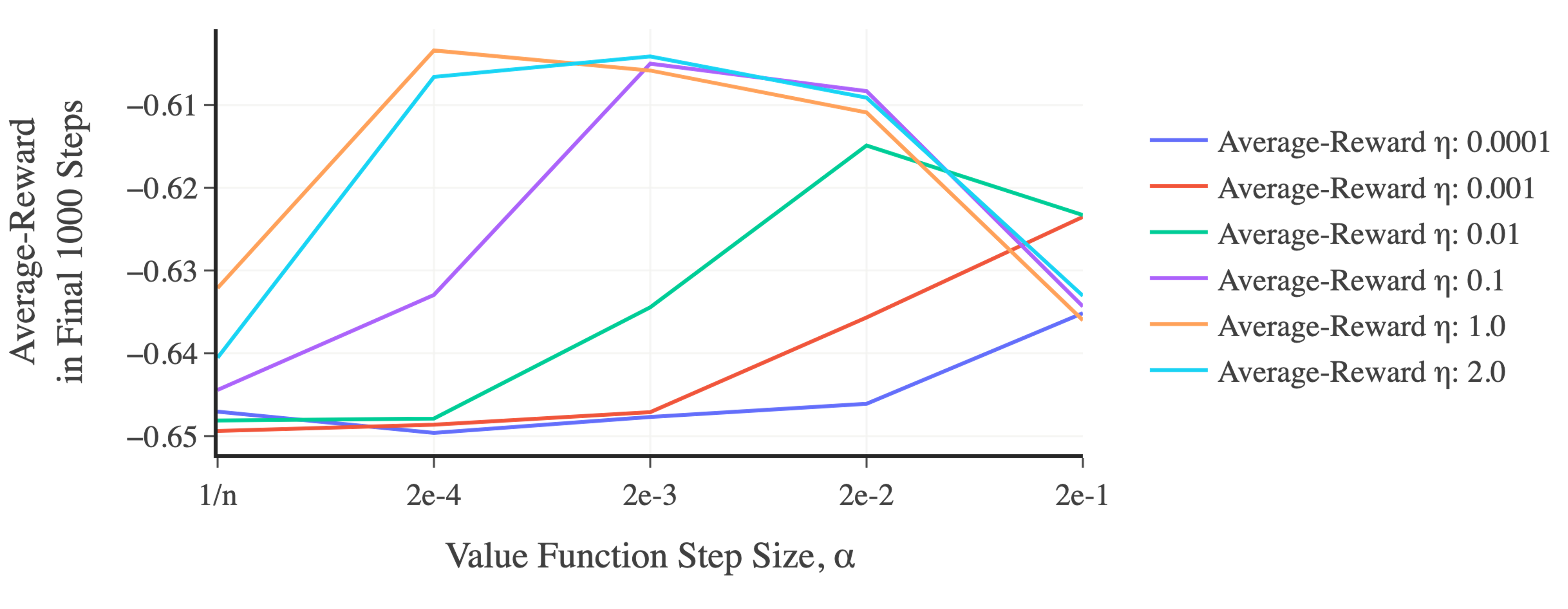}}
\caption{Step size tuning results for the red-pill blue-pill task when using the Differential Q-learning algorithm. The average-reward in the final 1,000 steps is displayed for various combinations of value function and average-reward step sizes.}
\label{fig_tuning}
\end{figure}

For the RED CVaR Q-learning algorithm, we tested every combination of the value function step size, \(\alpha\in\{1/n, \text{2e-4, 2e-3, 2e-2, 2e-1}\}\), with the average-reward (in this case CVaR) step size, \(\alpha_{_{\text{CVaR}}} \doteq \eta_{_\text{CVaR}}\alpha\), where \(\eta_{_\text{CVaR}}\in\{\text{1e-4, 1e-3, 1e-2, 1e-1, 1.0, 2.0}\}\), and the VaR step size, \(\alpha_{_{\text{VaR}}} \doteq \eta_{_\text{VaR}}\alpha\), where \(\eta_{_\text{VaR}}\in\{\text{1e-4, 1e-3, 1e-2, 1e-1, 1.0, 2.0}\}\), for a total of 180 unique combinations. Each combination was run 50 times using different random seeds, and the results were averaged across the runs. A value function step size of 2e-2, an average-reward (CVaR) \(\eta\) of 1e-1, and a VaR \(\eta\) of 1e-1 yielded the best results and were used to generate the results displayed in Figures \ref{fig_results_1}{\textcolor{mylightblue}{a)} and \ref{fig_results_2}{\textcolor{mylightblue}{a)}.\\

\textbf{Follow-up Experiment: Varying the CVaR Parameter}

Given the results shown in Figure \ref{fig_results_1}{\textcolor{mylightblue}{a)}, we can see that, with proper hyperparameter tuning, the tabular RED CVaR Q-learning algorithm is able to reliably find the optimal CVaR policy for a CVaR parameter, \(\tau\), of 0.25. In the context of the red-pill blue-pill environment, this means that the agent learns to stay in the red world state because the state has a characteristic reward distribution with a better (higher) CVaR compared to the blue world state. By contrast, the risk-neutral Differential Q-learning algorithm yields an average-reward optimal policy that keeps the agent in the blue world state because the state has a better (higher) average-reward compared to the red world state.

Now consider what would happen if we used the RED CVaR Q-learning algorithm with a \(\tau\) of 0.99. By definition, a CVaR corresponding to a \(\tau \approx \text{1.0}\) is equivalent to the average-reward. Hence, with a \(\tau\) of 0.99, we would expect that the optimal CVaR policy corresponds to staying in the blue world state (since it has the better average-reward). This means that for some \(\tau\) between 0.25 and 0.99, there is a critical point where the CVaR-optimal policy changes from staying in the red world (let us call this the \emph{red policy}) to staying in the blue world state (let us call this the \emph{blue policy}).

Importantly, we can estimate this critical point using simple Monte Carlo (MC). We are able to use MC in this case because both policies effectively stay in a single state (the red or blue world state), such that the CVaR of the policies can be estimated by sampling the characteristic reward distribution of each state, while accounting for the exploration \(\varepsilon\). Figure \ref{fig_tau_1} shows the MC estimate of the CVaR of the red and blue policies for a range of CVaR parameters, assuming an exploration \(\varepsilon\) of 0.1. Note that we used the same distribution parameters listed in Appendix \ref{appendix_RPBP} for the red-pill blue-pill environment. As shown in Figure \ref{fig_tau_1}, this critical point occurs somewhere around \(\tau \approx \text{0.8}\).

\begin{figure}[htbp]
\centerline{\includegraphics[scale=0.32]{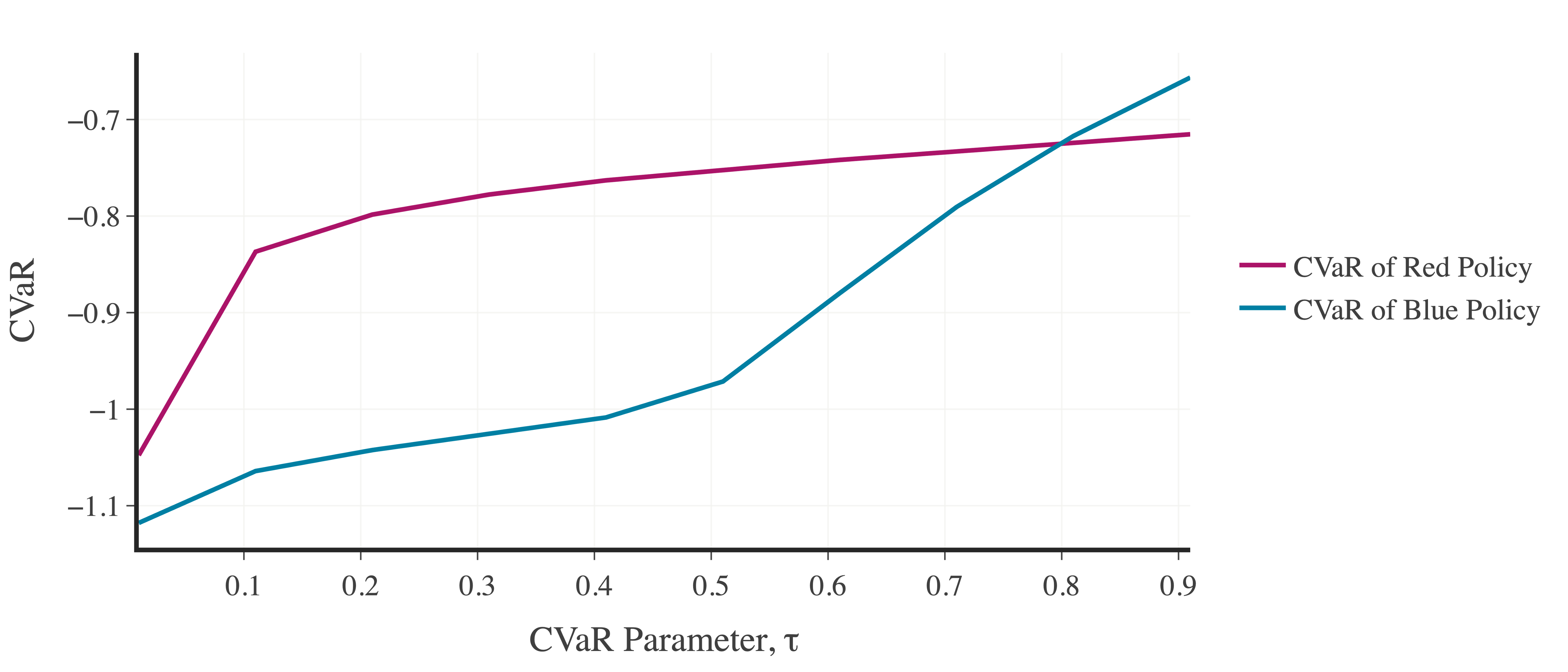}}
\caption{Monte Carlo estimates of the CVaR of the red and blue policies for a range of CVaR parameters in the red-pill blue-pill environment.}
\label{fig_tau_1}
\end{figure}

As such, one way that we can further validate the tabular RED CVaR Q-learning algorithm is by re-running the red-pill blue-pill experiment for different CVaR parameters, and seeing if the optimal CVaR policy indeed changes at a \(\tau \approx \text{0.8}\). Importantly, this allows us to empirically validate whether the algorithm actually optimizes at the desired risk level. When running this experiment, we used the same hyperparameters used to generate the results in Figure \ref{fig_results_1}{\textcolor{mylightblue}{a)}. We ran the experiment for \(\tau\in\{\text{0.1, 0.25, 0.5, 0.75, 0.85, 0.9}\}\). For each \(\tau\), we performed 50 runs using different random seeds, and the results were averaged across the runs.

\begin{figure}[htbp]
\centerline{\includegraphics[scale=0.31]{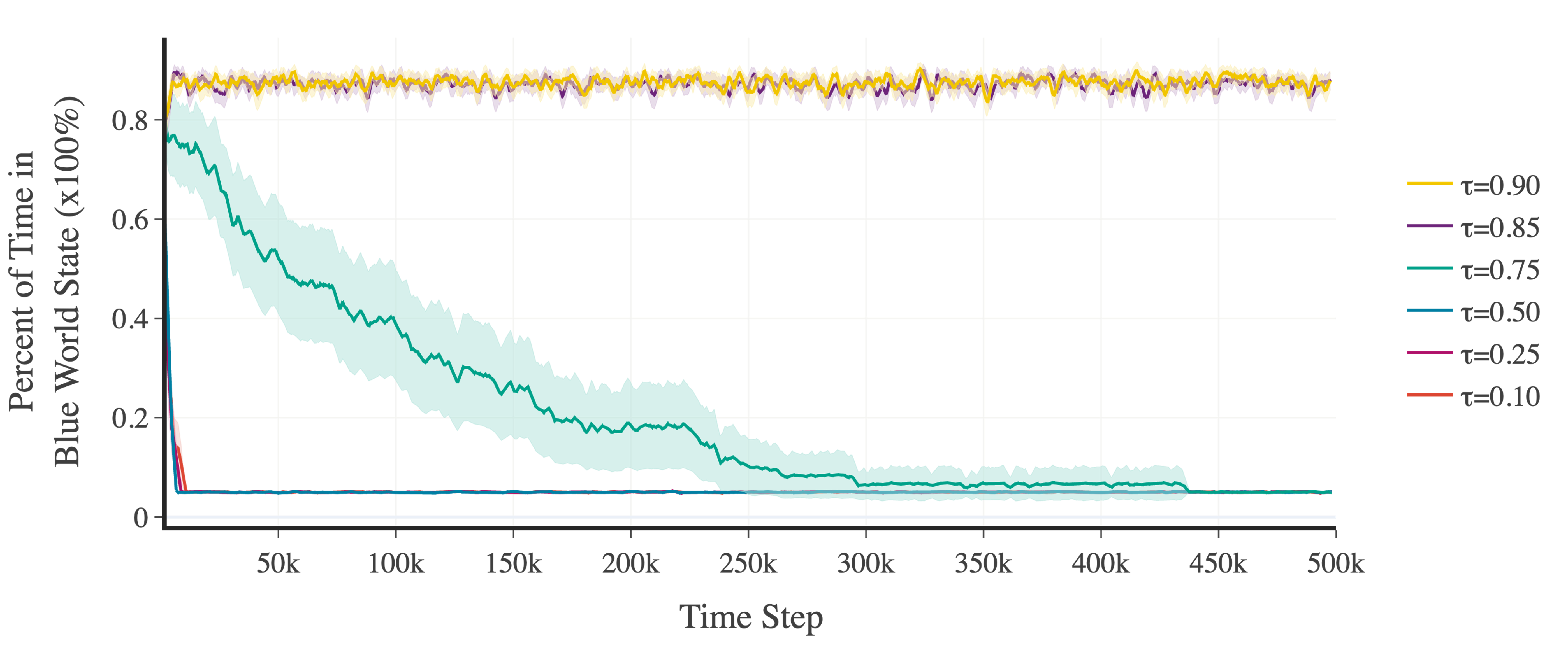}}
\caption{Rolling percent of time that the agent stays in the blue world state as learning progresses when using the RED CVaR Q-learning algorithm in the red-pill blue-pill environment for a range of CVaR parameters. A solid line denotes the mean percent of time spent in the blue world state, and the corresponding shaded region denotes a 95\% confidence interval over 50 runs.}
\label{fig_tau_2}
\end{figure}

\newpage

Figure \ref{fig_tau_2} shows the results of this experiment. In particular, the figure shows a rolling percent of time that the agent stays in the blue world state as learning progresses (note that we used an exploration \(\varepsilon\) of 0.1). From the figure, we can see that for \(\tau\in\{\text{0.1, 0.25, 0.5, 0.75}\}\), the agent learns to stay in the red world state, and for \(\tau\in\{\text{0.85, 0.9}\}\), the agent learns to stay in the blue world state. This is consistent with what we would expect, given that the critical point is \(\tau \approx \text{0.8}\). Hence, these results further validate that our algorithm is able to optimize at the desired risk level.

\subsection{Inverted Pendulum Experiment}
\label{exp_inverted_pendulum}

In the second experiment, we consider the well-known \emph{inverted pendulum} environment, where an agent is tasked with learning how to optimally balance an inverted pendulum. We chose this task because it provides us with the opportunity to test our algorithms in an environment where: 1) we must use function approximation (given the high-dimensional state-space), and 2) where the optimal CVaR policy and the optimal average-reward policy are the same policy (i.e., the policy that best balances the pendulum will yield a limiting reward distribution with both the optimal average-reward and reward CVaR). This hence allows us to directly compare the performance of our RED CVaR algorithms to that of the regular Differential algorithms, as well as to gauge how function approximation affects the performance of our algorithms. For this task, we utilized a simple actor-critic architecture \citep{Barto1983-qr, Sutton2018-eh} as this allowed us to compare the performance of a (non-tabular) RED CVaR TD-learning algorithm with a (non-tabular) Differential TD-learning algorithm. This task is illustrated in Figure \ref{fig_experiments}{\textcolor{mylightblue}{b)}.

For this experiment, we ran both algorithms using various combinations of step sizes for each algorithm. We used a fixed CVaR parameter, \(\tau\), of 0.1. We set all initial guesses to zero. We ran the algorithms for 100k time steps. For simplicity, we used tile coding \citep{Sutton2018-eh} for both the value function and policy parameterizations, where we parameterized a softmax policy. For each parameterization, we used 32 tilings, each with 8 X 8 tiles. 

By using a linear function approximator (i.e., tile coding), the gradients for the value function and policy parameterizations can be simplified as follows:

\begin{equation}
\label{eq_d_1}
\nabla \hat{v}(s,\boldsymbol{w}) = \boldsymbol{x}(s),
\end{equation}

\begin{equation}
\label{eq_d_2}
\nabla \text{ln} \pi(a \mid s,\boldsymbol{\theta}) = \boldsymbol{x}_h(s,a) - \sum_{u \in \mathcal{A}} \pi(u \mid s, \boldsymbol{\theta})\boldsymbol{x}_h(s, u),
\end{equation}

where \(s \in \mathcal{S}\), \(a \in \mathcal{A}\), \(\boldsymbol{x}(s)\) is the state feature vector, and \(\boldsymbol{x}_h(s,a)\) is the softmax preference vector.

For the Differential Actor-Critic algorithm, we tested every combination of the value function step size, \(\alpha\in\{1/n, \text{2e-4, 2e-3, 2e-2}\}\) (where \(1/n\) refers to a step size sequence that decreases the step size according to the time step, \(n\)), with \(\eta\)'s for the average-reward and policy step sizes, \(\eta\alpha\), where \(\eta\in\{\text{1e-3, 1e-2, 1e-1, 1.0, 2.0}\}\), for a total of 100 unique combinations. Each combination was run 10 times using different random seeds, and the results were averaged across the runs. A value function step size of 2e-3, a policy \(\eta\) of 2.0, and an average-reward \(\eta\) of 1e-2 yielded the best results and were used to generate the results displayed in Figure \ref{fig_results_1}{\textcolor{mylightblue}{b)}.

For the RED CVaR Actor-Critic algorithm, we tested every combination of the value function step size, \(\alpha\in\{1/n, \text{2e-4, 2e-3, 2e-2}\}\), with \(\eta\)'s for the average-reward (CVaR), VaR, and policy step sizes, \(\eta\alpha\), where \(\eta\in\{\text{1e-3, 1e-2, 1e-1, 1.0, 2.0}\}\), for a total of 500 unique combinations. Each combination was run 10 times using different random seeds, and the results were averaged across the runs. A value function step size of 2e-3, a policy \(\eta\) of 1.0, an average-reward (CVaR) \(\eta\) of 1e-2, and a VaR \(\eta\) of 1e-3 were used to generate the results displayed in Figures \ref{fig_results_1}{\textcolor{mylightblue}{b)} and \ref{fig_results_2}{\textcolor{mylightblue}{b)}.

\newpage
\subsection{Risk-Neutral Differential Algorithms}
This section contains the pseudocode for the risk-neutral Differential algorithms used for comparison in our experiments.
\label{risk_neutral_algorithms}

\begin{algorithm}
   \caption{Differential Q-Learning (Tabular)}
   \label{alg_reg_1}
\begin{algorithmic}
    \STATE {\bfseries Input:} the policy \(\pi\) to be used (e.g., \(\varepsilon\)-greedy)
    \STATE {\bfseries Algorithm parameters:} step size parameters \(\alpha\), \(\eta\)
    \STATE Initialize \(Q(s, a) \: \forall s, a\) (e.g. to zero)
    \STATE Initialize \(\bar{R}\) arbitrarily (e.g. to zero)
    \STATE Obtain initial \(S\)
    \WHILE{still time to train}
        \STATE \(A \leftarrow\) action given by \(\pi\) for \(S\)
        \STATE Take action \(A\), observe \(R, S'\)
        \STATE \(\delta = R - \bar{R} + \max_a Q(S', a) - Q(S, A)\)
        \STATE \(Q(S, A) = Q(S, A) + \alpha\delta\)
        \STATE \(\bar{R} = \bar{R} + \eta \alpha \delta\)
        \STATE \(S = S'\)
    \ENDWHILE
    \STATE return \(Q\)
\end{algorithmic}
\end{algorithm}

\begin{algorithm}
   \caption{Differential Actor-Critic}
   \label{alg_reg_2}
\begin{algorithmic}
    \STATE {\bfseries Input:} a differentiable state-value function parameterization \(\hat{v}(s, \boldsymbol{w})\); a differentiable policy parameterization \(\pi(a \mid s, \boldsymbol{\theta})\)
    \STATE {\bfseries Algorithm parameters:} step size parameters \(\alpha\), \(\eta_{\pi}\), \(\eta_{_{\bar{R}}}\)
    \STATE Initialize state-value weights \(\boldsymbol{w} \in \mathbb{R}^{d}\) and policy weights \(\boldsymbol{\theta} \in \mathbb{R}^{d'}\) (e.g. to \(\boldsymbol{0}\))
    \STATE Initialize \(\bar{R}\) arbitrarily (e.g. to zero)
    \STATE Obtain initial \(S\)
    \WHILE{still time to train}
        \STATE \(A \sim \pi(\cdot \mid S, \boldsymbol{\theta})\)
        \STATE Take action \(A\), observe \(R, S'\)
        \STATE \(\delta = R - \bar{R} + \hat{v}(S', \boldsymbol{w}) - \hat{v}(S, \boldsymbol{w})\)
        \STATE \(\boldsymbol{w} = \boldsymbol{w} + \alpha\delta\nabla\hat{v}(S, \boldsymbol{w})\)
        \STATE \(\boldsymbol{\theta} = \boldsymbol{\theta} + \eta_{\pi}\alpha\delta\nabla \text{ln} \pi(A \mid S, \boldsymbol{\theta})\)
        \STATE \(\bar{R} = \bar{R} + \eta_{_{\bar{R}}}\alpha\delta\)
        \STATE \(S = S'\)
    \ENDWHILE
    \STATE return \(\boldsymbol{w}\), \(\boldsymbol{\theta}\)
\end{algorithmic}
\end{algorithm}



\clearpage
\section{Red-Pill Blue-Pill Environment}
\label{appendix_RPBP}

This appendix contains a Python implementation of the \emph{red-pill blue-pill} environment introduced in this work. The environment consists of a two-state MDP, where at every time step an agent can take either a ‘red pill’, which takes them to the ‘red world’ state, or a ‘blue pill’, which takes them to the ‘blue world’ state. Each state has its own characteristic per-step reward distribution, and in this case, for a sufficiently low CVaR parameter, \(\tau\), the red world state has a per-step reward distribution with a lower (worse) mean but higher (better) CVaR compared to the blue world state. More specifically, the red world state reward distribution is characterized as a gaussian distribution with a mean of -0.7 and a standard deviation of 0.05. Conversely, the blue world state is characterized by a mixture of two gaussian distributions with means of -1.0 and -0.2, and standard deviations of 0.05. We assume all rewards are non-positive. The Python implementation of the environment is provided below:

\begin{lstlisting}[language=Python]
import pandas as pd
import numpy as np

class EnvironmentRedPillBluePill:
  def __init__(self, dist_2_mix_coefficient=0.5):
    # set distribution parameters
    self.dist_1 = {'mean': -0.7, 'stdev': 0.05}
    self.dist_2a = {'mean': -1.0, 'stdev': 0.05}
    self.dist_2b = {'mean': -0.2, 'stdev': 0.05}
    self.dist_2_mix_coefficient = dist_2_mix_coefficient

    # start state
    self.start_state = np.random.choice(['redworld', 'blueworld'])

  def env_start(self, start_state=None):
    # return initial state
    if pd.isnull(start_state):
      return self.start_state
    else:
      return start_state

  def env_step(self, state, action, terminal=False):
    if action == 'red_pill':
      next_state = 'redworld'
    elif action == 'blue_pill':
      next_state = 'blueworld'

    if state == 'redworld':
      reward = np.random.normal(loc=self.dist_1['mean'], 
                                scale=self.dist_1['stdev'])
    elif state == 'blueworld':
      dist = np.random.choice(['dist2a', 'dist2b'], 
                              p=[self.dist_2_mix_coefficient, 
                                1 - self.dist_2_mix_coefficient])
      if dist == 'dist2a':
        reward = np.random.normal(loc=self.dist_2a['mean'], 
                                  scale=self.dist_2a['stdev'])
      elif dist == 'dist2b':
        reward = np.random.normal(loc=self.dist_2b['mean'], 
                                  scale=self.dist_2b['stdev'])

    return min(0, reward), next_state, terminal
\end{lstlisting}


\end{document}